\documentclass[sn-mathphys,Numbered]{sn-jnl}

\usepackage{graphicx}%
\usepackage{multirow}%
\usepackage{amsmath,bm,bbm,amssymb,amsfonts}%
\usepackage{amsthm}%
\usepackage{mathrsfs}%
\usepackage[title]{appendix}%
\usepackage{xcolor}%
\usepackage{textcomp}%
\usepackage{manyfoot}%
\usepackage{booktabs}%
\usepackage{algorithm}%
\usepackage{algorithmicx}%
\usepackage{algpseudocode}%
\usepackage{listings}%

\DeclareMathOperator*{\argmin}{arg\,min}

\usepackage{mathtools}

\newtheorem{theorem}{Theorem}
\newtheorem{proposition}[theorem]{Proposition}%

\begin{document}

\title[Predictive Low Rank Matrix Learning under Partial Observations: Mixed-Projection ADMM]{Predictive Low Rank Matrix Learning under Partial Observations: Mixed-Projection ADMM}


\author[1]{\fnm{Dimitris} \sur{Bertsimas}}\email{dbertsim@mit.edu}
\equalcont{These authors contributed equally to this work.}

\author*[2]{\fnm{Nicholas} \sur{Johnson}}\email{nagj@mit.edu}
\equalcont{These authors contributed equally to this work.}

\affil[1]{\orgdiv{Sloan School of Management}, \orgname{Massachusetts Institute of Technology}, \orgaddress{\street{100 Main Street}, \city{Cambridge}, \postcode{02142}, \state{MA}, \country{USA}}}

\affil[2]{\orgdiv{Operations Research Center}, \orgname{Massachusetts Institute of Technology}, \orgaddress{\street{1 Amherst Street}, \city{Cambridge}, \postcode{02142}, \state{MA}, \country{USA}}}


\abstract{We study the problem of learning a partially observed matrix under the low rank assumption in the presence of fully observed side information that depends linearly on the true underlying matrix. This problem consists of an important generalization of the Matrix Completion problem, a central problem in Statistics, Operations Research and Machine Learning, that arises in applications such as recommendation systems, signal processing, system identification and image denoising. We formalize this problem as an optimization problem with an objective that balances the strength of the fit of the reconstruction to the observed entries with the ability of the reconstruction to be predictive of the side information. We derive a mixed-projection reformulation of the resulting optimization problem and present a strong semidefinite cone relaxation. We design an efficient, scalable alternating direction method of multipliers algorithm that produces high quality feasible solutions to the problem of interest. Our numerical results demonstrate that in the small rank regime ({\color{black}$k \leq 10$}), our algorithm outputs solutions that achieve on average {\color{black}$2.3\%$} lower objective value and {\color{black}$41\%$} lower $\ell_2$ reconstruction error than the solutions returned by the best performing benchmark method on synthetic data. The runtime of our algorithm is competitive with and often superior to that of the benchmark methods. Our algorithm is able to solve problems with $n = 10000$ rows and $m = 10000$ columns in less than a minute. On large scale real world data, our algorithm produces solutions that achieve $67\%$ lower out of sample error than benchmark methods in $97\%$ less execution time.}




\maketitle

\section{Introduction} \label{sec_lrml4:intro}

In many real world applications, we are faced with the problem of recovering a (often large) matrix from a (often small) subset of its entries. This problem, known as Matrix Completion (MC), has gained significant attention due to its broad range of applications in areas such as signal processing \citep{candes2010matrix}, system identification \citep{liu2019power} and image denoising \citep{ji2010robust}. The fundamental task in MC is to accurately reconstruct the missing entries of a matrix given a limited number of observed entries. The challenge is particularly pronounced when the number of observed entries is small relatively to the dimension of the matrix, yet this is the common scenario in practice.

One of the most prominent uses of MC is in recommendation systems, where the goal is to predict user preferences for items (e.g., movies, products) based on a partially observed user-item rating matrix \citep{ramlatchan2018survey}. The Netflix Prize competition highlighted the potential of MC techniques, where the objective was to predict missing ratings in a user-movie matrix to improve recommendation accuracy \citep{bell2007lessons}. The success of such systems hinges on the assumption that the underlying rating matrix is low rank, meaning that the preferences of users can be well-approximated by a small number of factors. Indeed, it has been well studied that many real world datasets are low rank \citep{udell2019big}.

In many practical applications, in addition to a collection of observed matrix entries we additionally have access to auxiliary side information that can be leveraged when performing the reconstruction. For example, in a recommendation system, side information might consist of social network data or item attributes. The vast majority of existing approaches to MC in the presence of side information incorporate the side information by making additional structural restrictions on the reconstructed matrix beyond the usual low rank assumption (see, for example, \cite{xu2013speedup, chiang2015matrix, bertsimas2020fast}). In this work, we take an alternate approach by assuming that the side information can be well {\color{black}modeled} as a linear function of the underlying full matrix. In this setting, the side information can be thought of as labels for a regression problem where the unobserved matrix consists of the regression features. This assumption is in keeping with ideas from the predictive low rank kernel learning literature \citep{bach2005predictive} (note however that low rank kernel learning assumes a fully observed input matrix). {\color{black}One motivation for this approach is that it incorporates the side information in a more flexible manner in contrast to existing approaches that make fixed structural assumptions through an explicit parametrization of the partially observed matrix as a function of the side information. Indeed, the model we present is most natural for the setting of MC with noisy side information in contrast to the setting of perfectly noiseless side information \citep{farhat2013genomic, natarajan2014inductive} or graph based side information \citep{elmahdy2020matrix, suh2021use}.}

Formally, let $\Omega \subseteq [n] \times [m]$ denote a collection of revealed entries of a partially observed matrix $\bm{A} \in \mathbb{R}^{n \times m}$, let $\bm{Y} \in \mathbb{R}^{n \times d}$ denote a matrix of side information and let $k$ denote a specified target rank. We consider the problem given by
\begin{equation}
    \begin{aligned}
        \min_{\bm{X} \in \mathbb{R}^{n \times m}, \bm{\alpha} \in \mathbb{R}^{m \times d}} \sum_{(i, j) \in \Omega}(X_{ij}-A_{ij})^2 + \lambda \Vert \bm{Y} - \bm{X}\bm{\alpha} \Vert_F^2 + \gamma \Vert \bm{X} \Vert_{\star} \quad \text{s.t.} \quad \text{rank}(\bm{X}) \leq k,
    \end{aligned} \label{opt_lrml4:MC_primal}
\end{equation} where $\lambda, \gamma > 0$ are hyperparameters that in practice can either take a default value or can be cross-validated by minimizing a validation metric \citep{validation} to obtain strong out-of-sample performance \citep{bousquet2002stability}. We assume that the ground truth matrix $\bm{A}$ has low rank and that the side information can be well approximated as $\bm{Y} = \bm{A} \bm{\alpha} + \bm{N}$ for some weighting matrix $\bm{\alpha}$ and noise matrix $\bm{N}$. The first term in the objective function of \eqref{opt_lrml4:MC_primal} measures how well the observed entries of the unknown matrix are fit by the estimated matrix $\bm{X}$, the second term of the objective function measures how well the side information $\bm{Y}$ can be represented as a linear function of the estimated matrix $\bm{X}$ and the final term of the objective is a regularization term. To the best of our knowledge, Problem \eqref{opt_lrml4:MC_primal} has not previously been directly studied despite its very natural motivation. 

\subsection{Contribution and Structure}

In this paper, we tackle \eqref{opt_lrml4:MC_primal} by developing novel mixed-projection optimization techniques \citep{bertsimas2020mixed}. We show that solving \eqref{opt_lrml4:MC_primal} is equivalent to solving an appropriately defined robust optimization problem. We develop an exact reformulation of \eqref{opt_lrml4:MC_primal} by combining a parametrization of the $\bm{X}$ decision variable as the product of two low rank factors with the introduction of a projection matrix to model the column space of $\bm{X}$. We derive a semidefinite cone convex relaxation for our mixed-projection reformulation and we present an efficient, scalable alternating direction method of multipliers (ADMM) algorithm that produces high quality feasible solutions to \eqref{opt_lrml4:MC_primal}. Our numerical results show that across all synthetic data experiments in the small rank regime ({\color{black}$k \leq 10$}), our algorithm outputs solutions that achieve on average {\color{black}$2.3\%$} lower objective value in \eqref{opt_lrml4:MC_primal} and {\color{black}$41\%$} lower $\ell_2$ reconstruction error than the solutions returned by the best performing benchmark method on a per experiment basis. For the $5$ synthetic data experiments with $k > 15$, the only benchmark that returns a solution with superior quality than that returned by our algorithm takes on average $3$ times as long to execute. The runtime of our algorithm is competitive with and often superior to that of the benchmark methods. Our algorithm is able to solve problems with $n = 10000$ rows and $m = 10000$ columns in less than a minute. On large scale real world data, our algorithm produces solutions that achieve $67\%$ lower out of sample error than benchmark methods in $97\%$ less execution time.

The rest of the paper is laid out as follows. In Section \ref{sec_lrml4:lit_review}, we review previous work that is closely related to \eqref{opt_lrml4:MC_primal}. In Section \ref{sec_lrml4:form_properties}, we study \eqref{opt_lrml4:MC_primal} under a robust optimization lens and investigate formulating \eqref{opt_lrml4:MC_primal} as a two stage optimization problem where the inner stage is a regression problem that can be solved in closed form. We formulate \eqref{opt_lrml4:MC_primal} as a mixed-projection optimization problem in Section \ref{sec_lrml4:mixed_proj} and present a natural convex relaxation. In Section \ref{sec_lrml4:admm}, we present and rigorously study our ADMM algorithm. Finally, in Section \ref{sec_lrml4:experiments} we investigate the performance of our algorithm against benchmark methods on synthetic and real world data.

\paragraph{Notation:} We let nonbold face characters such as $b$ denote scalars, lowercase bold faced characters such as $\bm{x}$ denote vectors, uppercase bold faced characters such as $\bm{X}$ denote matrices, and calligraphic uppercase characters such as $\mathcal{Z}$ denote sets. We let $[n]$ denote the set of running indices $\{1, ..., n\}$. We let $\bm{0}_n$ denote an $n$-dimensional vector of all $0$'s, $\bm{0}_{n\times m}$ denote an $n \times m$-dimensional matrix of all $0$'s, and $\bm{I}_n$ denote the $n \times n$ identity matrix. We let $\mathcal{S}^n$ denote the cone of $n \times n$ symmetric matrices and $\mathcal{S}^n_+$ denote the cone of $n \times n$ positive semidefinite matrices.

\section{Literature Review} \label{sec_lrml4:lit_review}

In this section, we review a handful of notable approaches from the literature that have been employed to solve MC and to solve general low rank optimization problems. As an exhaustive literature review of MC methods is outside of the scope of this paper, we focus our review on a handful of well studied approaches which we will employ as benchmark methods in this work. We additionally give an overview of the ADMM algorithmic framework which is of central relevance to this work. For a more detailed review of the MC literature, we refer the reader to \cite{ramlatchan2018survey} and \cite{nguyen2019low}.

\subsection{Matrix Completion Methods} \label{ssec_lrml4:mc_methods}

\subsubsection{Iterative-SVD}

Iterative-SVD is an expectation maximization style algorithm \citep{dempster1977maximum} that generates a solution to the MC problem by iteratively computing a singular value decomposition (SVD) of the current iterate and estimating the missing values by performing a regression against the low rank factors returned by SVD \citep{troyanskaya2001missing}. This is one of a handful of methods in the literature that leverage the SVD as their primary algorithmic workhorse \citep{billsus1998learning, sarwar2000application}. Concretely, given a partially observed matrix $\{X_{ij}\}_{(i, j) \in \Omega}$ where $\Omega \subseteq [n] \times [m]$ and a target rank $k \in \mathbf{N}_+$, Iterative-SVD proceeds as follows:

\begin{enumerate}
    \item Initialize the iteration count $t \leftarrow 0$ and initialize missing entries of $X_{ij}, (i, j) \notin \Omega$ with the row average $X_{ij} = \frac{\sum_{l:(i, l)\in \Omega} X_{il}}{\vert\{l:(i, l)\in \Omega\}\vert}$.
    \item Compute a rank $k$ SVD $\bm{X}_t = \bm{U}_t\bm{\Sigma}_t\bm{V}_t^T$ of the current iterate where $\bm{U}_t \in \mathbb{R}^{n \times k}, \bm{\Sigma}_t \in \mathbb{R}^{k \times k}, \bm{V}_t \in \mathbb{R}^{m \times k}$.
    \item For each $(i, j) \notin \Omega$, estimate the missing value $(\bm{X}_{t+1})_{ij}$ by regressing all other entries in row $i$ against all except the $j^{th}$ row of $\bm{V}_t$. Concretely, letting $\bm{\Tilde{x}} = (\bm{X}_t)_{i, \star \setminus j} \in \mathbb{R}^{m-1}$ denote the column vector consisting of the $i^{th}$ row of $\bm{X}_t$ excluding the $j^{th}$ entry, letting $\bm{\Tilde{V}}=(\bm{V}_t)_{\star \setminus j, \star} \in \mathbb{R}^{(m-1) \times k}$ denote the matrix formed by eliminating the $j^{th}$ row from $\bm{V}_t$ and letting $\bm{\hat{v}} = (\bm{V}_t)_{j, \star} \in \mathbb{R}^k$ denote the column vector consisting of the $j^{th}$ row of $\bm{V}_t$, we set $(\bm{X}_{t+1})_{ij} = \bm{\hat{v}}^T(\bm{\Tilde{V}}^T\bm{\Tilde{V}})^{-1}\bm{\Tilde{V}}^T\bm{\Tilde{x}}.$
    \item Terminate if the total change between $\bm{X}_t$ and $\bm{X}_{t+1}$ is less than $0.01$. Otherwise, increment $t$ and return to Step 2.
\end{enumerate}

\subsubsection{Soft-Impute}

Soft-Impute is a convex relaxation inspired algorithm that leverages the nuclear norm as a low rank inducing regularizer \citep{JMLR:v11:mazumder10a}. This approach is one of a broad class of methods that tackle MC from a nuclear norm minimization lens \citep{fazel2002matrix, candes2010power, candes2012exact}. Seeking a reconstruction with minimum nuclear norm is typically motivated by the observation that the nuclear norm ball given by $\mathcal{B}=\{\bm{X} \in \mathbb{R}^{n \times n}: \Vert \bm{X} \Vert_\star \leq k\}$ is the convex hull of the nonconvex set $\mathcal{X} = \{\bm{X}\in \mathbb{R}^{n \times n}: \text{rank}(\bm{X})\leq k, \Vert \bm{X} \Vert_\sigma \leq 1\}$, where $\Vert \cdot \Vert_\sigma$ denotes the spectral norm. Moreover, several conditions have been established under which nuclear norm minimization methods are guaranteed to return the ground truth matrix \citep{candes2010power, candes2012exact} though such conditions tend to be strong and hard to verify in practice. Soft-Impute proceeds by iteratively replacing the missing elements of the matrix with those obtained from a soft thresholded low rank singular value decomposition. Accordingly, similarly to Iterative-SVD, Soft-Impute relies on the computation of a low rank SVD as the primary algorithmic workhorse. The approach relies on the result that for an arbitrary matrix $\bm{X}$, the solution of the problem $\min_{\bm{Z}} \frac{1}{2} \Vert \bm{X} - \bm{Z} \Vert_F^2 + \lambda \Vert \bm{Z} \Vert_\star$ is given by $\bm{\Hat{Z}} = S_\lambda(\bm{X})$ where $S_\lambda(\cdot)$ denotes the soft-thresholding operation \citep{donoho1995wavelet}. Explicitly, Soft-Impute proceeds as follows for a given regularization parameter $\lambda > 0$ and termination threshold $\epsilon > 0$:

\begin{enumerate}
    \item Initialize the iteration count $t \leftarrow 0$ and initialize $\bm{Z}_t = \bm{0}_{n \times m}$.
    \item Compute $\bm{Z}_{t+1} = S_\lambda(P_\Omega(\bm{X}) + P_\Omega^\perp(\bm{Z}_t))$ where $P_\Omega(\cdot)$ denotes the operation that projects onto the revealed entries of $\bm{X}$ while $P_\Omega^\perp(\cdot)$ denotes the operation that projects onto the missing entries of $\bm{X}$.
    \item Terminate if $\frac{\Vert\bm{Z}_t-\bm{Z}_{t+1}\Vert_F^2}{\Vert\bm{Z}_t\Vert_F^2}$. Otherwise, increment $t$ and return to Step 2.
\end{enumerate}

\subsubsection{Fast-Impute}

Fast-Impute is a projected gradient descent approach to MC that has desirable global convergence properties \citep{bertsimas2020fast}. Fast-Impute belongs to the broad class of methods that solve MC by factorizing the target matrix as $\bm{X} = \bm{U} \bm{V}^T$ where $\bm{U} \in \mathbb{R}^{n \times k}, \bm{V} \in \mathbb{R}^{m \times k}$ and performing some variant of gradient descent (or alternating minimization) on the matrices $\bm{U}$ and $\bm{V}$ \citep{koren2009matrix, jain2015fast, zheng2016convergence, jin2016provable}. We note that we leverage this common factorization in the approach to \eqref{opt_lrml4:MC_primal} presented in this work. Gradient descent based methods have shown great success. Despite the non-convexity of the factorization, it has been shown that in many cases gradient descent and its variants will nevertheless converge to a globally optimal solution \citep{chen2015fast, ge2015escaping, sun2016guaranteed, ma2018implicit, bertsimas2020fast}. Fast-Impute takes the approach of expressing $\bm{U}$ as a closed form function of $\bm{V}$ after performing the facorization and directly performs projected gradient descent updates on $\bm{V}$ with classic Nesterov acceleration \citep{nesterov1983method}. Moreover, to enhance scalability of their method, \cite{bertsimas2020fast} design a stochastic gradient extension of Fast-Impute that estimates the gradient at each update step by only considering a sub sample of the rows and columns of the target matrix. {\color{black}Finally, given side information $\bm{Y} \in \mathbb{R}^{m \times p}, $\cite{bertsimas2020fast} present a version of their algorithm that incorporates the side information by factorizing the target matrix as $\bm{X} = \bm{U}\bm{S}^T\bm{Y}^T$ for $\bm{U} \in \mathbb{R}^{n \times k}, \bm{S} \in \mathbb{R}^{p \times k}$.} 

\subsection{Low Rank Optimization Methods}

\subsubsection{ScaledGD} \label{sssec_lrml4:scaledGD}

ScaledGD is a highly performant method to obtain strong solutions to low rank matrix estimation problems that take the following form: \[\min_{\bm{X}\in \mathbb{R}^{n \times m}} f(\bm{X}) = \frac{1}{2}\Vert \mathcal{A}(\bm{X}) - \bm{y}\Vert_2^2 \quad \text{s.t.  rank}(\bm{X}) \leq k, \] where $\mathcal{A}(\cdot): \mathbb{R}^{n \times m} \rightarrow \mathbb{R}^l$ models some measurement process and we have $\bm{y} \in \mathbb{R}^l$ \citep{tong2021accelerating}. ScaledGD proceeds by factorizing the target matrix as $\bm{X} = \bm{U}\bm{V}^T$ and iteratively performing gradient updates on the low rank factors $\bm{U}, \bm{V}$ after preconditioning the gradients with an adaptive matrix that is efficient to compute. Doing so yields a linear convergence rate that is notably independent of the condition number of the low rank matrix. In so doing, ScaledGD combines the desirable convergence rate of alternating minimization with the desirable low per-iteration cost of gradient descent. Explicitly, letting $\mathcal{L}(\bm{U}, \bm{V}) = f(\bm{U}\bm{V}^T)$, ScaledGD updates the low rank factors as:
\[\bm{U}_{t+1} \leftarrow \bm{U}_t - \eta\nabla_{\bm{U}}\mathcal{L}(\bm{U}_t, \bm{V}_t)(\bm{V}^T\bm{V})^{-1},\]
\[\bm{V}_{t+1} \leftarrow \bm{V}_t - \eta\nabla_{\bm{V}}\mathcal{L}(\bm{U}_t, \bm{V}_t)(\bm{U}^T\bm{U})^{-1},\] where $\eta > 0$ denotes the step size.

\subsubsection{Mixed-Projection Conic Optimization}

Mixed-projection conic optimization is a recently proposed modelling and algorithmic framework designed to tackle a broad class of matrix optimization problems \citep{bertsimas2020mixed, bertsimas2021perspective}. Specifically, this approach considers problems that have the following form:

\begin{equation}
    \begin{aligned}
        \min_{\bm{X}\in \mathbb{R}^{n \times m}} \langle \bm{C}, \bm{X} \rangle + \lambda \cdot \text{rank}(\bm{X}) + \Omega(\bm{X}) \quad \text{s.t.} \quad \bm{A} \bm{X} = \bm{B}, \,\, \text{rank}(\bm{X}) \leq k, \,\, \bm{X} \in \mathcal{K},
    \end{aligned} \label{opt_lrml4:MPCO}
\end{equation} where $\bm{C} \in \mathbb{R}^{n \times m}$ is a cost matrix, $\lambda > 0$, $k \in \mathbb{N}_+$, $\bm{A} \in \mathbb{R}^{l \times n}, \bm{B} \in \mathbb{R}^{l \times m}, \mathcal{K}$ denotes a proper cone in the sense of \cite{boyd2004convex} and $\Omega(\cdot)$ is a Frobenius norm regularization function or a spectral norm regularization function of the input matrix. The main workhorse of mixed-projection conic optimization is the use of a projection matrix to cleverly model the rank terms in \eqref{opt_lrml4:MPCO}. This can be viewed as the matrix generalization of using binary variables to model the sparsity of a vector in mixed-integer optimization. \cite{bertsimas2020mixed} show that for an arbitrary matrix $\bm{X} \in \mathbb{R}^{n \times m}$, we have
\[\text{rank}(\bm{X}) \leq k \iff \exists \bm{P} \in \mathbb{S}^n: \bm{P}^2=\bm{P}, \text{Tr}(\bm{P}) \leq k, \bm{X}=\bm{P}\bm{X}.\] Introducing projection matrices allows the rank functions to be eliminated from \eqref{opt_lrml4:MPCO} at the expense of introducing non convex quadratic equality constraints. From here, most existing works that leverage mixed-projection conic optimization have either focused on obtaining strong semidefinite based convex relaxations \citep{bertsimas2020mixed, bertsimas2021perspective} or have focused on obtaining certifiably optimal solutions for small and moderately sized problem instances \citep{bertsimas2023slr, bertsimas2023optimal}. In this work, we leverage the mixed-projection framework to scalably obtain high quality solutions to large problem instances.

\subsection{Alternating Direction Method of Multipliers}

Alternating direction method of multipliers (ADMM) is an algorithm that was originally designed to solve linearly constrained convex optimization problems of the form
\begin{equation}
    \begin{aligned}
        &\min_{\bm{x}\in \mathbb{R}^n, \bm{z} \in \mathbb{R}^m} && f(\bm{x})+g(\bm{z}) \quad \text{s.t. 
 }\bm{A}\bm{x}+\bm{B}\bm{z}=\bm{c},
    \end{aligned} \label{opt_lrml4:ADMM_review}
\end{equation} where we have $\bm{A}\in \mathbb{R}^{l \times n}, \bm{B} \in \mathbb{R}^{l \times m}, \bm{c} \in \mathbb{R}^l$ and the functions $f$ and $g$ are assumed to be convex \citep{boyd2011distributed}. The main benefit of ADMM is that it can combine the decomposition benefits of dual ascent with the desirable convergence properties of the method of multipliers. Letting $\bm{y} \in \mathbb{R}^l$ denote the dual variable, the augmented Lagrangian of \eqref{opt_lrml4:ADMM_review} is given by \[\mathcal{L}^A(\bm{x}, \bm{z}, \bm{y}) = f(\bm{x})+g(\bm{z}) + \bm{y}^T(\bm{A}\bm{x}+\bm{B}\bm{z}-\bm{c})+\frac{\rho}{2}\Vert \bm{A}\bm{x}+\bm{B}\bm{z}-\bm{c} \Vert_2^2,\] where $\rho > 0$ is the augmented Lagrangian parameter. ADMM then proceeds by iteratively updating the primal variable $\bm{x}$, updating the primal variable $\bm{z}$ and taking a gradient ascent step on the dual variable $\bm{y}$. Explicitly, ADMM consists of the following updates:
\begin{enumerate}
    \item $\bm{x}_{t+1} \leftarrow \argmin_{\bm{x}} \mathcal{L}^A(\bm{x}, \bm{z}_t, \bm{y}_t)$,
    \item $\bm{z}_{t+1} \leftarrow \argmin_{\bm{z}} \mathcal{L}^A(\bm{x}_{t+1}, \bm{z}, \bm{y}_t)$,
    \item $\bm{y}_{t+1} \leftarrow \bm{y}_t + \rho (\bm{A}\bm{x}_{t+1}+\bm{B}\bm{z}_{t+1}-\bm{c})$.
\end{enumerate}
Under very mild regularity conditions on $f, g$ and $\mathcal{L}^A$, it is well known that ADMM is guaranteed to produce a sequence of primal iterates that converges to the optimal value of \eqref{opt_lrml4:ADMM_review} and a sequence of dual iterates that converge to the optimal dual variable (note that there is no guarantee of primal variable convergence) \citep{boyd2011distributed}. Importantly, although ADMM was originally designed for linearly constrained convex optimization, it has often been applied to non convex optimization problems and yielded empirically strong results \citep{xu2016empirical}. This observation has motivated work to explore the theoretical convergence behavior of ADMM and its variants on specific classes of non convex optimization problems \citep{guo2017convergence, wang2019global, wang2021nonconvex}.

\section{Formulation Properties} \label{sec_lrml4:form_properties}

In this section, we rigorously investigate certain key features of \eqref{opt_lrml4:MC_primal}. Specifically, we establish an equivalence between \eqref{opt_lrml4:MC_primal} and an appropriately defined robust optimization problem. Moreover, we illustrate that \eqref{opt_lrml4:MC_primal} can be reduced to an optimization problem over only $\bm{X}$ and establish that the resulting objective function is not convex, not concave and non-smooth. Finally, we study how efficient evaluations of the reduced problem objective function can be performed.

\subsection{Equivalence Between Regularization and Robustness}

Real-world datasets frequently contain inaccuracies and missing values, which hinder the ability of machine learning models to generalize effectively to new data when these inconsistencies are not appropriately {\color{black}modeled}. Consequently, robustness is a crucial quality for machine learning models, both in theory and application \citep{xu2009robustness, bertsimas2020robust}. In this section, we show that our regularized problem \eqref{opt_lrml4:MC_primal} can be viewed as a robust optimization (RO) problem. This finding justifies the inclusion of the nuclear norm regularization term in \eqref{opt_lrml4:MC_primal} and is in a similar flavor as known results from the
robust optimization literature in the case of vector \citep{bertsimas2018characterization} and matrix \citep{bertsimas2023slr} problems. {\color{black}Note that the equivalence presented in Proposition \ref{prop_lrml4:ro} follows from the well studied duality of the nuclear norm and the spectral norm.}

\begin{proposition} \label{prop_lrml4:ro}
    Problem \eqref{opt_lrml4:MC_primal} is equivalent to the following robust optimization problem:
    \begin{equation}
    \begin{aligned}
        \min_{\substack{\bm{X} \in \mathbb{R}^{n \times m} \\ \bm{\alpha} \in \mathbb{R}^{m \times d}}} \max_{\bm{\Delta} \in \mathcal{U}} \sum_{(i, j) \in \Omega}(X_{ij}-A_{ij})^2 + \lambda \Vert \bm{Y} - \bm{X}\bm{\alpha} \Vert_F^2 + \langle \bm{X}, \bm{\Delta} \rangle \quad \text{s.t.} \,\, \text{rank}(\bm{X}) \leq k_0,
    \end{aligned} \label{opt_lrml4:MC_primal_ro}
    \end{equation} where $\mathcal{U}=\{\bm{\Delta} \in \mathbb{R}^{n \times m}: \Vert \bm{\Delta} \Vert_\sigma \leq \gamma\}$.
\end{proposition}

\begin{proof}
    To establish this result, it suffices to argue that $\max_{\bm{\Delta} \in \mathcal{U}}\langle \bm{X}, \bm{\Delta}\rangle = \gamma \Vert \bm{X}\Vert_\star$. This equivalence follows immediately from the fact that the nuclear norm is dual to the spectral norm. So as to keep this manuscript self contained, we present a proof of this equivalence below.

    Consider any matrix $\bm{\Bar{\Delta}} \in \mathbb{R}^{n \times m}$ such that $\Vert \bm{\Bar{\Delta}} \Vert_\sigma \leq \gamma$. Let $\bm{X} = \bm{U}\bm{\Sigma}\bm{V}^T$ be a singular value decomposition of $\bm{X}$ where we let $r = \text{rank}(\bm{X})$ and we have $\bm{U} \in \mathbb{R}^{n \times r}, \bm{\Sigma} \in \mathbb{R}^{r \times r}, \bm{V} \in \mathbb{R}^{m \times r}$. We have
    \begin{equation*}
    \begin{aligned}
        \langle \bm{X}, \bm{\Bar{\Delta}} \rangle &= \text{Tr}(\bm{\Bar{\Delta}}^T\bm{U}\bm{\Sigma}\bm{V}^T) = \text{Tr}(\bm{V}^T\bm{\Bar{\Delta}}^T\bm{U}\bm{\Sigma}) = \langle \bm{U}^T\bm{\Bar{\Delta}}\bm{V}, \bm{\Sigma}\rangle = \sum_{i=1}^r \Sigma_{ii} (\bm{U}^T\bm{\Bar{\Delta}}\bm{V})_{ii} \\
        &= \sum_{i=1}^r \Sigma_{ii} U_i^T\bm{\Bar{\Delta}}V_i \leq \sum_{i=1}^r \Sigma_{ii} \sigma_1(\bm{\Bar{\Delta}}) \leq \gamma \sum_{i=1}^r \Sigma_{ii} = \gamma \Vert \bm{X} \Vert_{\star},
    \end{aligned}
    \end{equation*} where we have used the fact that $\bm{\Sigma}$ is a diagonal matrix and the columns of $\bm{U}$ and $\bm{V}$ have unit length. Thus, we have shown that $\gamma \Vert \bm{X}\Vert_\star$ is an upper bound for $\max_{\bm{\Delta} \in \mathcal{U}}\langle \bm{X}, \bm{\Delta}\rangle$. To show that the upper bound is always achieved, consider the matrix $\bm{\Tilde{\Delta}} = \gamma \bm{U} \bm{V}^T \in \mathbb{R}^{n \times m}$ where $\bm{U}$ and $\bm{V}$ are taken from a singular value decomposition of $\bm{X}$. Observe that \[\Vert \bm{\Tilde{\Delta}} \Vert_\sigma = \gamma \Vert \bm{U} \bm{V}^T \Vert_\sigma \leq \gamma \implies \bm{\Tilde{\Delta}} \in \mathcal{U}.\] We conclude by noting that \[\langle \bm{X}, \bm{\Tilde{\Delta}} \rangle = Tr(\bm{V}\bm{\Sigma}\bm{U}^T\gamma \bm{U}\bm{V}^T) = \gamma Tr(\bm{V}^T\bm{V}\bm{\Sigma}\bm{U}^T\bm{U}) = \gamma \text{Tr}(\bm{I}\bm{\Sigma}\bm{I}) = \gamma \Vert \bm{X}\Vert_\star.\]
    
\end{proof}
Proposition \ref{prop_lrml4:ro} implies that solving the nuclear norm regularized \eqref{opt_lrml4:MC_primal} is equivalent to solving an unregularized robust optimization problem that protects against adversarial perturbations that are bounded in spectral norm. This result is not surprising given the duality of norms, yet is nevertheless insightful.

\subsection{A Partial Minimization} \label{ssec_lrml4:partial_min}

Let $g(\bm{X}, \bm{\alpha})$ denote the objective function of \eqref{opt_lrml4:MC_primal}. Note that $g(\bm{X}, \bm{\alpha})$ is bi-convex in $(\bm{X}, \bm{\alpha})$ but is not jointly convex due to the product $\bm{X}\bm{\alpha}$. Observe that we can simplify \eqref{opt_lrml4:MC_primal} by performing a partial minimization in $\bm{\alpha}$. For any $\bm{X}$, the problem in $\bm{\alpha}$ requires finding the unconstrained minimum of a convex quadratic function. The gradient of $g$ with respect to $\bm{\alpha}$ is given by $\nabla_{\bm{\alpha}}g(\bm{X}, \bm{\alpha}) = 2 \lambda \bm{X}^T(\bm{X}\bm{\alpha}-\bm{Y})$. Setting $\nabla_{\bm{\alpha}}g(\bm{X}, \bm{\alpha})$ to $0$ yields $\bm{\alpha}^\star = (\bm{X}^T\bm{X})^\dagger \bm{X}^T\bm{Y}$ as a minimizer of $g$ over $\bm{\alpha}$. Note that $\bm{M}^\dagger$ denotes the pseudo-inverse of a (possibly rank deficient) square matrix $\bm{M} \in \mathbb{R}^{l \times l}$. Specifically, letting $r = \text{rank}(\bm{M})$ and $\bm{M} = \bm{U}\bm{\Sigma}\bm{V}^T$ be a singular value decomposition of $\bm{M}$ with $\bm{U}, \bm{V} \in \mathbb{R}^{l \times r}$ and $\bm{\Sigma} \in \mathbb{R}^{l \times l}$, we have $\bm{M}^\dagger = \bm{U}\bm{\Sigma}^{-1}\bm{V}^T$. Letting $f(\bm{X})$ correspond to the partially minimized objective function of  \eqref{opt_lrml4:MC_primal}, we have 
\begin{equation*}
    \begin{aligned}
        f(\bm{X}) &= \min_{\bm{\alpha}}g(\bm{X}, \bm{\alpha}) = \sum_{(i, j) \in \Omega}(X_{ij}-A_{ij})^2 + \lambda \Vert (\bm{I}_n - \bm{X}(\bm{X}^T\bm{X})^\dagger \bm{X}^T)\bm{Y} \Vert_F^2 + \gamma \Vert \bm{X} \Vert_{\star} \\
        &= \sum_{(i, j) \in \Omega}(X_{ij}-A_{ij})^2 + \lambda \text{Tr}\big{(}\bm{Y}^T(\bm{I}_n - \bm{X}(\bm{X}^T\bm{X})^\dagger \bm{X}^T)\bm{Y}\big{)} + \gamma \Vert \bm{X} \Vert_{\star}.
    \end{aligned}
\end{equation*} We note that $\bm{\alpha}^\star$ corresponds to the well studied ordinary least squares solution. When $\bm{X}^T\bm{X}$ has full rank, $\bm{\alpha}$ is the unique minimizer of $g$. If $\bm{X}^T\bm{X}$ is rank deficient, $\bm{\alpha}^\star$ corresponds to the minimizer with minimum norm.

Though we have simplified the objective function of \eqref{opt_lrml4:MC_primal}, $f(\bm{X})$ is not a particularly well behaved function. We formalize this statement in Proposition \eqref{prop_lrml4:f_properties}.

\begin{proposition} \label{prop_lrml4:f_properties}
    The function $f(\bm{X})$ is in general neither convex nor concave and is non-smooth.
\end{proposition}

\begin{proof}
    To illustrate that $f(\bm{X})$ is in general neither convex nor concave, suppose that $\Omega = \emptyset$, $n = 2$ and $m = d = \lambda = \gamma = 1$. In this setting, we have $\bm{x}, \bm{y} \in \mathbb{R}^{2 \times 1}$. Assuming that $\bm{x} \neq \bm{0}_2$, we can write the objective function as
    \begin{equation*}
    \begin{aligned}
        f(\bm{x}) &= \text{Tr}(\bm{y}^T(\bm{I}_2 - \bm{x}(\bm{x}^T\bm{x})^{-1}\bm{x}^T)\bm{y}) + \Vert \bm{x}\Vert_\star \\
        &= \bm{y}^T\bm{y} - \frac{\bm{y}^T\bm{x}\bm{x}^T\bm{y}}{\bm{x}^T\bm{x}} + \Vert \bm{x}\Vert_2 \\
        &= \bm{y}^T\bm{y} - \frac{(\bm{y}^T\bm{x})^2}{\bm{x}^T\bm{x}} + \sqrt{\bm{x}^T\bm{x}}.
    \end{aligned}
    \end{equation*} For $\bm{x} = \bm{0}_2$, the objective value $f(\bm{0}_2)$ is equal to $\bm{y}^T\bm{y}$. Let $\bm{y} = \bm{1}_2$ and consider the line in $\bm{R}^2$ defined by $\mathcal{X} = \{\bm{x} \in \mathbb{R}^2: x_2 = x_1 + 1\}$. The restriction of $f(\bm{x})$ to the line defined by $\mathcal{X}$ is a univariate function given by \[
    f_\mathcal{X}(t) = 2 - \frac{(2t+1)^2}{2t^2+2t+1} + \sqrt{2t^2+2t+1},
    \] where $t \in \mathbb{R}$ is a dummy variable. Observe that we have $f_\mathcal{X}(-1) = f_\mathcal{X}(0) = 2, f_\mathcal{X}(-0.5) = 2 + \frac{\sqrt{2}}{2}$ and $f_\mathcal{X}(-4)=f_\mathcal{X}(3)=5.04$. Thus, the point $(-0.5, f_\mathcal{X}(0.5))$ lies above the chord connecting $(-1, f_\mathcal{X}(-1))$ and $(0, f_\mathcal{X}(0))$, so $f_\mathcal{X}(t)$ is not a convex function. Moreover, the point $(-1, f_\mathcal{X}(-1))$ lies below the chord connecting $(-4, f_\mathcal{X}(-4))$ and $(-0.5, f_\mathcal{X}(-0.5))$, so $f_\mathcal{X}(t)$ is not a concave function. Since a function is convex (respectively concave) if and only if its restriction to every line is convex (respectively concave), we have established that $f(\bm{X})$ is neither convex nor concave since $f_\mathcal{X}(t)$ is neither convex nor concave. To conclude the proof of Proposition \ref{prop_lrml4:f_properties}, note that the non-smooth property of $f(\bm{X})$ follows immediately from the non-smooth property of the nuclear norm function.
\end{proof}

Although the above closed form partial minimization in $\bm{\alpha}$ eliminates $m \times d$ variables form \eqref{opt_lrml4:MC_primal}, this comes at the expense of introducing a $m \times m$ matrix pseudo-inverse term into the objective function which can be computationally expensive to evaluate. Efficient evaluation of an objective function is crucial in many optimization problems to quickly measure solution quality. A plethora of modern optimization techniques require iterative objective function evaluations. As a result, the computational cost of evaluating an objective function can quickly become the bottleneck of an algorithm's complexity. Directly evaluating $f(\bm{X})$ naively requires $O(\vert \Omega \vert)$ operations for the first term, $O(m^2n+m^3+n^2d)$ for the second term (forming the matrix $\bm{X}^T\bm{X}$ is $O(m^2n)$, taking the pseudo-inverse is $O(m^3)$, computing the products involving $\bm{Y}$ is $O(n^2d)$) and requires $O(mn\min(m, n))$ for the third term (the nuclear norm can be evaluated by computing a singular value decomposition of $\bm{X}$). We observe that computing the second term of $f(\bm{X})$ involving the pseudo-inverse dominates the complexity calculation. Indeed, the overall complexity of evaluating $f(\bm{X})$ naively is $O(m^2n+m^3+n^2d)$.

Fortunately, it is possible to make evaluations of $f(\bm{X})$ without explicitly forming the product $\bm{X}^T\bm{X}$ or taking a pseudo-inverse. Proposition \ref{prop_lrml4:pseudo} illustrates that it suffices (in terms of computational complexity) to take a singular value decomposition of $\bm{X}$. Moreover, a large class of optimization algorithms require only function evaluations for feasible solutions. If we consider only those values of $\bm{X}$ that are feasible to \eqref{opt_lrml4:MC_primal}, it is sufficient (in terms of computational complexity) to take a rank $k$ truncated singular value decomposition of $\bm{X}$ to make functions evaluations of $f(\bm{X})$.
\begin{proposition} \label{prop_lrml4:pseudo}
    The function $f(\bm{X})$ can equivalently be written as \[ f(\bm{X})= \sum_{(i, j) \in \Omega}(X_{ij}-A_{ij})^2 + \lambda \text{Tr}\big{(}\bm{Y}^T(\bm{I}_n - \bm{U}\bm{U}^{T})\bm{Y}\big{)} + \gamma \Vert \bm{X} \Vert_{\star},\] where $\bm{X} = \bm{U}\bm{\Sigma}\bm{V}^T$ is a singular value decomposition of $\bm{X}$ where we let $r = \text{rank}(\bm{X})$ and we have $\bm{U} \in \mathbb{R}^{n \times r}, \bm{\Sigma} \in \mathbb{R}^{r \times r}, \bm{V} \in \mathbb{R}^{m \times r}$.
\end{proposition}

\begin{proof}
    To establish the result, it suffices to show that \[\text{Tr}\big{(}\bm{Y}^T\bm{X}(\bm{X}^T\bm{X})^\dagger \bm{X}^T\bm{Y}\big{)} = \text{Tr}\big{(}\bm{Y}^T \bm{U}\bm{U}^{T}\bm{Y}\big{)}.\] Let $\bm{X} = \bm{U}\bm{\Sigma}\bm{V}^T$ be a singular value decomposition of $\bm{X}$ where $r = \text{rank}(\bm{X})$ and $\bm{U} \in \mathbb{R}^{n \times r}, \bm{\Sigma} \in \mathbb{R}^{r \times r}, \bm{V} \in \mathbb{R}^{m \times r}$. Observe that
    \begin{equation*}
        \begin{aligned}
            \text{Tr}\big{(}\bm{Y}^T\bm{X}(\bm{X}^T\bm{X})^\dagger \bm{X}^T\bm{Y}\big{)} &= \text{Tr}\big{(}\bm{Y}^T\bm{U}\bm{\Sigma}\bm{V}^T(\bm{V}\bm{\Sigma}\bm{U}^T \bm{U}\bm{\Sigma}\bm{V}^T)^\dagger \bm{V}\bm{\Sigma}\bm{U}^T\bm{Y}\big{)} \\
            &= \text{Tr}\big{(}\bm{Y}^T\bm{U}\bm{\Sigma}\bm{V}^T(\bm{V}\bm{\Sigma}^2\bm{V}^T)^\dagger \bm{V}\bm{\Sigma}\bm{U}^T\bm{Y}\big{)} \\
            &= \text{Tr}\big{(}\bm{Y}^T\bm{U}\bm{\Sigma}\bm{V}^T\bm{V}\bm{\Sigma}^{-2}\bm{V}^T \bm{V}\bm{\Sigma}\bm{U}^T\bm{Y}\big{)} \\
            &= \text{Tr}\big{(}\bm{Y}^T\bm{U}\bm{\Sigma}\bm{\Sigma}^{-2}\bm{\Sigma}\bm{U}^T\bm{Y}\big{)} \\
            &= \text{Tr}\big{(}\bm{Y}^T\bm{U}\bm{U}^T\bm{Y}\big{)},
        \end{aligned}
    \end{equation*} where we have repeatedly invoked the property that $\bm{U}^T\bm{U} = \bm{V}^T\bm{V} = \bm{I}_r$.
\end{proof}
In light of Proposition \ref{prop_lrml4:pseudo}, evaluating $f(\bm{X})$ for feasible solutions still requires $O(\vert \Omega \vert)$ operations for the first term, but the second term can be evaluated using $O(kn(m+d))$ operations (performing a truncated singular value decomposition is $O(knm)$ and computing the products involving $\bm{Y}$ is $O(knd)$) and the third term can be evaluated using $O(knm)$ operations (by performing a truncated singular value decomposition) for an overall complexity of $O(kn(m+d))$. This is significantly less expensive than the $O(m^2n+m^3+n^2d)$ complexity of naive direct evaluation of $f(\bm{X})$ introduced previously.

\section{An Exact Mixed-Projection Formulation} \label{sec_lrml4:mixed_proj}

In this section, we reformulate \eqref{opt_lrml4:MC_primal} as a mixed-projection optimization problem and further reduce the dimension of the resulting problem in a commonly studied manner by parameterizing $\bm{X}$ as the matrix product of two low dimensional matrices. Thereafter, we illustrate how to employ the matrix generalization of the perspective relaxation \citep{gunluk2012perspective, bertsimas2020mixed, bertsimas2021perspective, bertsimas2023slr} to construct a convex relaxation of \eqref{opt_lrml4:MC_primal}.

We first note that given the result of Section \ref{ssec_lrml4:partial_min}, we can rewrite \eqref{opt_lrml4:MC_primal} as an optimization problem only over $\bm{X}$ as follows:

\begin{equation}
    \begin{aligned}
        &\min_{\bm{X} \in \mathbb{R}^{n \times m}} & & \sum_{(i, j) \in \Omega}(X_{ij}-A_{ij})^2 + \lambda \text{Tr}\big{(}\bm{Y}^T(\bm{I}_n - \bm{X}(\bm{X}^T\bm{X})^\dagger \bm{X}^T)\bm{Y}\big{)} + \gamma \Vert \bm{X} \Vert_{\star} \\
        &\text{s.t.} & & \text{rank}(\bm{X}) \leq k.
    \end{aligned} \label{opt_lrml4:MC_primal_X}
\end{equation} 
Observe that the matrix $\bm{X}(\bm{X}^T\bm{X})^\dagger \bm{X}^T$ is the linear transformation that projects vectors onto the subspace spanned by the columns of the matrix $\bm{X}$. Drawing on ideas presented in \cite{bertsimas2020mixed, bertsimas2021perspective, bertsimas2023slr}, we introduce an orthogonal projection matrix $\bm{P} \in \mathcal{P}_k$ to model the column space of $\bm{X}$ where $\mathcal{P}_\eta = \{\bm{P} \in \mathcal{S}^n: \bm{P}^2=\bm{P}, \text{tr}(\bm{P}) \leq \eta\}$ for $\eta \geq 0$. We can express the desired relationship between $\bm{P}$ and $\bm{X}$ as $\bm{X} = \bm{P}\bm{X}$ since projecting a matrix onto its own column space leaves the matrix unchanged. This gives the following reformulation of \eqref{opt_lrml4:MC_primal}:

\begin{equation}
    \begin{aligned}
        &\min_{\bm{X} \in \mathbb{R}^{n \times m}, \bm{P} \in \mathbb{R}^{n \times n}} & & \sum_{(i, j) \in \Omega}(X_{ij}-A_{ij})^2 + \lambda \text{Tr}\big{(}\bm{Y}^T(\bm{I}_n - \bm{P})\bm{Y}\big{)} + \gamma \Vert \bm{X} \Vert_{\star} \\
        &\text{s.t.} & & (\bm{I}_n - \bm{P})\bm{X} = \bm{0}_{n \times m}, \, \bm{P} \in \mathcal{P}_{\min(k, \text{rank}(\bm{X}))}.
    \end{aligned} \label{opt_lrml4:MC_primal_XP}
\end{equation} 
Observe that the matrix pseudo-inverse term has been eliminated from the objective function, however we have introduced the bilinear constraint $\bm{X} = \bm{P}\bm{X}$ which is non convex in the optimization variables as well as the non convex constraint $\bm{P} \in \mathcal{P}_{\min(k, \text{rank}(\bm{X}))}$. We now have the following result:

\begin{proposition} \label{prop_lrml4:XP_reform}
Problem \eqref{opt_lrml4:MC_primal_XP} is a valid reformulation of \eqref{opt_lrml4:MC_primal_X}.
\end{proposition}

\begin{proof}
    We show that given a feasible solution to \eqref{opt_lrml4:MC_primal_XP}, we can construct a feasible solution to \eqref{opt_lrml4:MC_primal_X} that achieves the same objective value and vice versa.

    Consider an arbitrary feasible solution $(\bm{\Bar{X}}, \bm{\Bar{P}})$ to \eqref{opt_lrml4:MC_primal_XP}. Since $\bm{\Bar{P}}\bm{\Bar{X}} = \bm{\Bar{X}}$ and $\bm{\Bar{P}} \in \mathcal{P}_{\min(k, \text{rank}(\bm{\Bar{X}}))}$, we have $\text{rank}(\bm{\Bar{X}}) \leq k$. We claim that $\bm{\Bar{X}}$ achieves the same objective value in \eqref{opt_lrml4:MC_primal_X} as $(\bm{\Bar{X}}, \bm{\Bar{P}})$ achieves in \eqref{opt_lrml4:MC_primal_XP}. To show this, it suffices to illustrate that for all $(\bm{\Bar{X}}, \bm{\Bar{P}})$ feasible to \eqref{opt_lrml4:MC_primal_XP} we have $H(\bm{\Bar{X}}) \coloneqq \bm{\Bar{X}} (\bm{\Bar{X}}^T \bm{\Bar{X}})^\dagger \bm{\Bar{X}}^T = \bm{\Bar{P}}$. The matrix $\bm{\Bar{P}}$ is an orthogonal projection matrix since it is symmetric and satisfies $\bm{\Bar{P}}^2=\bm{\Bar{P}}$. Moreover, since $\text{rank}(\bm{\Bar{P}})=\text{rank}(\bm{\Bar{X}})$ and $\bm{\Bar{P}}\bm{\Bar{X}} = \bm{\Bar{X}}$ we know that $\bm{\Bar{P}}$ is an orthogonal projection onto the subspace spanned by the columns of $\bm{\Bar{X}}$. Similarly, it can easily be verified that $H(\bm{\Bar{X}})$ is symmetric and satisfies $H(\bm{\Bar{X}})^2 = H(\bm{\Bar{X}})$, $\text{rank}(H(\bm{\Bar{X}})) = \text{rank}(\bm{\Bar{X}})$ and $H(\bm{\Bar{X}})\bm{\Bar{X}} = \bm{\Bar{X}}$. Thus, $H(\bm{\Bar{X}})$ is also an orthogonal projection matrix onto the subspace spanned by the columns of $\bm{\Bar{X}}$. To conclude, we invoke the property that given a subspace $\mathcal{V} \subset \mathbb{R}^n$ the orthogonal projection onto $\mathcal{V}$ is uniquely defined. To see this, suppose $\bm{P}_1$ and $\bm{P}_2$ are two orthogonal projections onto $\mathcal{V}$. Let $l = \text{dim}(\mathcal{V})$. Let $\{\bm{e}_i\}_{i=1}^l$ be an orthogonal basis for $\mathcal{V}$ and let $\{\bm{e}_i\}_{l+1}^n$ be an orthogonal basis for $\mathcal{V}^\bot$. Since $\bm{P}_1$ is an orthogonal projection onto $\mathcal{V}$, we have $\bm{P}_1\bm{e}_i=\bm{e}_i$ for all $1 \leq i \leq l$ and $\bm{P}_1\bm{e}_i=\bm{0}_n$ for all $l+1 \leq i \leq n$. However, the same must hold for $\bm{P}_2$ which implies that $\bm{P}_1 = \bm{P}_2$.
    
    Consider an arbitrary feasible solution $\bm{\Bar{X}}$ to \eqref{opt_lrml4:MC_primal_X}. Let $r = \text{rank}(\bm{\Bar{X}})$ and $\bm{\Bar{X}} = \bm{\Bar{U}}\bm{\Bar{\Sigma}}\bm{\Bar{V}}^T$ be a singular value decomposition of $\bm{\Bar{X}}$ where we have $\bm{\Bar{U}} \in \mathbb{R}^{n \times r}, \bm{\Bar{\Sigma}} \in \mathbb{R}^{r \times r}, \bm{\Bar{V}} \in \mathbb{R}^{m \times r}$. Define $\bm{\Bar{P}} = \bm{\Bar{U}} \bm{\Bar{U}}^T$. By construction, we have $\bm{\Bar{P}} \in \mathcal{P}_{\min(k, \text{rank}(\bm{\Bar{X}}))}$ since $r \leq k$. Moreover, it is easy to verify that \[\bm{\Bar{P}}\bm{\Bar{X}} = \bm{\Bar{U}} \bm{\Bar{U}}^T \bm{\Bar{U}}\bm{\Bar{\Sigma}}\bm{\Bar{V}}^T = \bm{\Bar{U}}\bm{\Bar{\Sigma}}\bm{\Bar{V}}^T = \bm{\Bar{X}},\] where we have used the property $\bm{\Bar{U}}^T \bm{\Bar{U}} = \bm{I}_r$.
    Finally, Proposition \ref{prop_lrml4:pseudo} immediately implies that $(\bm{\Bar{X}}, \bm{\Bar{P}})$ achieves the same objective in \eqref{opt_lrml4:MC_primal_XP} as $\bm{\Bar{X}}$ achieves in \eqref{opt_lrml4:MC_primal_X}. This completes the proof.
\end{proof}
Optimizing explicitly over the space of $n \times m$ matrices can rapidly become prohibitively costly in terms of runtime and memory requirements. Accordingly, we adopt the common approach of factorizing $\bm{X} \in \mathbb{R}^{n \times m}$ as $\bm{U}\bm{V}^T$ for $\bm{U} \in \mathbb{R}^{n \times k}, \bm{V} \in \mathbb{R}^{m \times k}$. This leads to the following formulation:

\begin{equation}
    \begin{aligned}
        &\min_{\substack{\bm{U} \in \mathbb{R}^{n \times k} \\ \bm{V} \in \mathbb{R}^{m \times k}, \\ \bm{P} \in \mathbb{R}^{n \times n}}} & & \sum_{(i, j) \in \Omega}((\bm{U}\bm{V}^T)_{ij}-A_{ij})^2 + \lambda \text{Tr}\big{(}\bm{Y}^T(\bm{I}_n - \bm{P})\bm{Y}\big{)} + \frac{\gamma}{2}( \Vert \bm{U} \Vert_F^2 + \Vert \bm{V} \Vert_F^2) \\
        &\text{s.t.} & & (\bm{I}_n-\bm{P})\bm{U} = \bm{0}_{n \times k}, \, \bm{P} \in \mathcal{P}_{\min(k, \text{rank}(\bm{U}\bm{V}^T))}.
    \end{aligned} \label{opt_lrml4:MC_primal_UVP}
\end{equation}
Notice that we have replaced $n \times m$ optimization variables with $k \times (n + m)$ optimization variables, an often significant dimension reduction in practice. Attentive readers may object that though this is true, we have introduced $n^2$ decision variables through the introduction of the projection matrix variable $\bm{P}$ which nullifies any savings introduced through the factorization of $\bm{X}$. Note, however, that it is possible to factor any feasible projection matrix as $\bm{P} = \bm{M}\bm{M}^T$ for some $\bm{M} \in \mathbb{R}^{n \times k}$. In Section \ref{sec_lrml4:admm}, we leverage this fact so that the presence of the projection matrix incurs a cost of $n \times k$ additional variables rather than $n^2$ variables. We have the following result:

\begin{proposition} \label{prop_lrml4:UVP_reform}
Problem \eqref{opt_lrml4:MC_primal_UVP} is a valid reformulation of \eqref{opt_lrml4:MC_primal_XP}.
\end{proposition}

\begin{proof}
We show that given a feasible solution to \eqref{opt_lrml4:MC_primal_UVP}, we can construct a feasible solution to \eqref{opt_lrml4:MC_primal_XP} that achieves the same or lesser objective value and vice versa.

Consider an arbitrary feasible solution $(\bm{\Bar{U}}, \bm{\Bar{V}}, \bm{\Bar{P}})$ to \eqref{opt_lrml4:MC_primal_UVP}. Let $\bm{\Bar{X}} = \bm{\Bar{U}}\bm{\Bar{V}}^T$. We will show that $(\bm{\Bar{X}}, \bm{\Bar{P}})$ is feasible to \eqref{opt_lrml4:MC_primal_XP} and achieves the same or lesser objective as $(\bm{\Bar{U}}, \bm{\Bar{V}}, \bm{\Bar{P}})$ does in \eqref{opt_lrml4:MC_primal_UVP}. Feasibility of \eqref{opt_lrml4:MC_primal_UVP} implies that $\bm{\Bar{P}} \in \mathcal{P}_{\min(k, \text{rank}(\bm{\Bar{U}}\bm{\Bar{V}}^T))} = \mathcal{P}_{\min(k, \text{rank}(\bm{\Bar{X}}))}$ and also that \[(\bm{I}_n-\bm{\Bar{P}})\bm{\Bar{X}} = (\bm{I}_n-\bm{\Bar{P}})\bm{\Bar{U}}\bm{\Bar{V}}^T = \bm{0}_{n \times k}\bm{\Bar{V}}^T = \bm{0}_{n \times m},\] thus the solution $(\bm{\Bar{X}}, \bm{\Bar{P}})$ is certainly feasible for \eqref{opt_lrml4:MC_primal_XP}. To see that $(\bm{\Bar{X}}, \bm{\Bar{P}})$ achieves the same or lesser objective value, it suffices to argue that $\Vert \bm{\Bar{X}} \Vert_\star \leq \frac{1}{2}(\Vert \bm{\Bar{U}} \Vert_F^2 + \Vert \bm{\Bar{V}} \Vert_F^2)$. This follows immediately from the following {\color{black}well-known} proposition established by \cite{JMLR:v11:mazumder10a} (see Appendix A.5 in their paper for a proof):
\begin{proposition} \label{lemma:mazumder}
    For any matrix $\bm{Z}$, the following holds: \[\Vert \bm{Z}\Vert_\star = \min_{\bm{U}, \bm{V}: \bm{Z}=\bm{U}\bm{V}^T} \frac{1}{2}(\Vert \bm{U} \Vert_F^2 + \Vert \bm{V} \Vert_F^2).\] If $\text{rank}(\bm{Z})=k\leq \min(m, n)$, then the minimum above is attained at a factor decomposition $\bm{U}_{n \times k}\bm{V}_{m \times k}^T$. Letting $\bm{Z}_{n \times m} = \bm{L}_{n \times k}\bm{\Sigma}_{k \times k}\bm{R}_{m \times k}^T$ denote a singular value decomposition of $\bm{Z}$, the minimum above is attained at $\bm{U}_{n \times k} = \bm{L}_{n \times k}\bm{\Sigma}_{k \times k}^\frac{1}{2}, \bm{V}_{m \times k} = \bm{R}_{m \times k}\bm{\Sigma}_{k \times k}^\frac{1}{2}$.
\end{proposition}

Consider now an arbitrary feasible solution $(\bm{\Bar{X}}, \bm{\Bar{P}})$ to \eqref{opt_lrml4:MC_primal_XP}. Let $\bm{\Bar{X}} = \bm{L}\bm{\Sigma}\bm{R}^T$ be a singular value decomposition of $\bm{\Bar{X}}$ where $\bm{L} \in \mathbb{R}^{n \times k}, \bm{\Sigma} \in \mathbb{R}^{k \times k}, \bm{R} \in \mathbb{R}^{m \times k}$ and define $\bm{\Bar{U}}=\bm{L}\bm{\Sigma}^\frac{1}{2}, \bm{\Bar{V}}=\bm{R}\bm{\Sigma}^\frac{1}{2}$. Feasibility of $(\bm{\Bar{X}}, \bm{\Bar{P}})$ in \eqref{opt_lrml4:MC_primal_XP} implies that $\bm{\Bar{P}} \in \mathcal{P}_{\min(k, \text{rank}(\bm{\Bar{X}}))} = \mathcal{P}_{\min(k, \text{rank}(\bm{\Bar{U}}\bm{\Bar{V}}^T))}$. Moreover, since the columns of $\bm{L}$ form an orthogonal basis for the columns space of $\bm{\Bar{X}}$, the condition $(\bm{I}_n-\bm{\Bar{P}})\bm{\Bar{X}} = \bm{0}_{n \times m}$ implies that \[(\bm{I}_n-\bm{\Bar{P}})\bm{\Bar{U}} = (\bm{I}_n-\bm{\Bar{P}})\bm{L}\bm{\Sigma}^\frac{1}{2} = \bm{0}_{n \times k}\bm{\Sigma}^\frac{1}{2} = \bm{0}_{n \times k}.\] Thus, the solution $(\bm{\Bar{U}}, \bm{\Bar{V}}, \bm{\Bar{P}})$ is feasible to \eqref{opt_lrml4:MC_primal_UVP}. Moreover, by Proposition \ref{lemma:mazumder} we have $\frac{1}{2}(\Vert \bm{\Bar{U}} \Vert_F^2 + \Vert \bm{\Bar{V}} \Vert_F^2) = \Vert \bm{\Bar{X}} \Vert_\star$ so $(\bm{\Bar{U}}, \bm{\Bar{V}}, \bm{\Bar{P}})$ achieves the same objective in \eqref{opt_lrml4:MC_primal_UVP} as $(\bm{\Bar{X}}, \bm{\Bar{P}})$ does in \eqref{opt_lrml4:MC_primal_XP}. This completes the proof.

\end{proof}

In the remainder of the paper, we will relax the constraint $\bm{P} \in \mathcal{P}_{\min(k, \text{rank}(\bm{U}\bm{V}^T))}$ to $\bm{P} \in \mathcal{P}_k$ and develop a scalable algorithm to obtain high quality feasible solutions. Explicitly, we consider the problem given by:

\begin{equation}
    \begin{aligned}
        &\min_{\substack{\bm{U} \in \mathbb{R}^{n \times k} \\ \bm{V} \in \mathbb{R}^{m \times k}, \\ \bm{P} \in \mathbb{R}^{n \times n}}} & & \sum_{(i, j) \in \Omega}((\bm{U}\bm{V}^T)_{ij}-A_{ij})^2 + \lambda \text{Tr}\big{(}\bm{Y}^T(\bm{I}_n - \bm{P})\bm{Y}\big{)} + \frac{\gamma}{2}( \Vert \bm{U} \Vert_F^2 + \Vert \bm{V} \Vert_F^2) \\
        &\text{s.t.} & & (\bm{I}_n-\bm{P})\bm{U} = \bm{0}_{n \times k}, \, \bm{P} \in \mathcal{P}_k.
    \end{aligned} \label{opt_lrml4:MC_primal_UVPk}
\end{equation}
It is straightforward to see that the optimal value of \eqref{opt_lrml4:MC_primal_UVPk} is no greater than the optimal value of \eqref{opt_lrml4:MC_primal_UVP}. Unfortunately, the converse does not necessarily hold. To see why the optimal value of \eqref{opt_lrml4:MC_primal_UVPk} can be strictly less than that of \eqref{opt_lrml4:MC_primal_UVP} in certain pathological cases, suppose we had $k = n = m$, $\Omega = \emptyset$. In this setting, letting $\bm{\Bar{P}} = \bm{I}_n$, $\bm{U}=\bm{0}_{n \times k}$ and $\bm{\Bar{V}} = \bm{0}_{m \times k}$, the solution $(\bm{\Bar{U}}, \bm{\Bar{V}}, \bm{\Bar{P}})$ would be feasible to  \eqref{opt_lrml4:MC_primal_UVPk} and achieve an objective value of $0$. However the optimal value of \eqref{opt_lrml4:MC_primal_UVP} would be strictly greater than $0$ in this setting as long as $\bm{Y} \neq 0$. Although \eqref{opt_lrml4:MC_primal_UVPk} is a relaxation of \eqref{opt_lrml4:MC_primal}, we will see in Section \ref{sec_lrml4:experiments} that the solutions we will obtain to \eqref{opt_lrml4:MC_primal_UVPk} will be high quality solutions for \eqref{opt_lrml4:MC_primal}, the main problem of interest.

\subsection{A Positive Semidefinite Cone Relaxation}

Convex relaxations are useful in non convex optimization primarily for two reasons. Firstly, given the objective value achieved by an arbitrary feasible solution, strong convex relaxations can be used to upperbound the worst case suboptimality of said solution. Secondly, convex relaxations can often be used as building blocks for global optimization procedures. In this section, we present a natural convex relaxation of \eqref{opt_lrml4:MC_primal_UVPk} that leverages the matrix generalization of the perspective relaxation \citep{gunluk2012perspective, bertsimas2020mixed, bertsimas2021perspective, bertsimas2023slr}.

Rather than working directly with \eqref{opt_lrml4:MC_primal_UVPk}, consider the equivalent formulation \eqref{opt_lrml4:MC_primal_XP} with $\mathcal{P}_{\min(k, \text{rank}(\bm{X}))}$ replaced by $\mathcal{P}_k$. Before proceeding, we will assume knowledge of an upper bound $M \in \mathbb{R}_+$ on the spectral norm of an optimal $\bm{X}$ to \eqref{opt_lrml4:MC_primal_XP}. Tighter bounds $M$ are desirable as they will lead to stronger convex relaxations of \eqref{opt_lrml4:MC_primal_XP}. We note that it is always possible to specify such an upper bound $M$ without prior knowledge of an optimal solution to \eqref{opt_lrml4:MC_primal_XP}. To see this, note that setting $\bm{X} = \bm{0}_{n \times m}$ in \eqref{opt_lrml4:MC_primal_XP} produces an objective value of $\sum_{(i, j) \in \Omega}A_{ij}^2 + \lambda \Vert \bm{Y}\Vert_F^2$. Thus, any $\bm{X}$ such that $\gamma \Vert \bm{X} \Vert_\star > \sum_{(i, j) \in \Omega}A_{ij}^2 + \lambda \Vert \bm{Y}\Vert_F^2$ cannot possibly be optimal to \eqref{opt_lrml4:MC_primal_XP}. Finally, since the nuclear norm is an upper bound on the spectral norm of a matrix, we must have \[\Vert \bm{X} \Vert_\sigma \leq \frac{\sum_{(i, j) \in \Omega}A_{ij}^2 + \lambda \Vert \bm{Y}\Vert_F^2}{\gamma},\] for any matrix $\bm{X}$ that is optimal to \eqref{opt_lrml4:MC_primal_XP}. We can therefore take $M = \frac{\sum_{(i, j) \in \Omega}A_{ij}^2 + \lambda \Vert \bm{Y}\Vert_F^2}{\gamma}$.

Notice that the non convexity in \eqref{opt_lrml4:MC_primal_XP} is captured entirely by the bilinear constraint $(\bm{I}_n-\bm{P})\bm{X} = \bm{0}_{n \times m}$ and the quadratic constraint $\bm{P}^2=\bm{P}$. In keeping with the approach presented in \cite{bertsimas2021perspective, bertsimas2023slr}, we leverage the matrix perspective to convexify the bilinear term and solve over the convex hull of the set $\mathcal{P}_k$. Recalling that the nuclear norm is semidefinite representable, we have the following formulation:

\begin{equation}
    \begin{aligned}
        &\min_{\substack{\bm{P}, \bm{W}_1 \in \mathbb{R}^{n \times n}, \\ \bm{X} \in \mathbb{R}^{n \times m} \\ \bm{W}_2 \in \mathbb{R}^{m \times m}}} & & \sum_{(i, j) \in \Omega}(X_{ij}-A_{ij})^2 + \lambda \text{Tr}\big{(}\bm{Y}^T(\bm{I}_n - \bm{P})\bm{Y}\big{)} + \frac{\gamma}{2}(\text{Tr}(\bm{W}_1)+\text{Tr}(\bm{W}_2)) \\
        &\text{s.t.} & & \bm{I}_n \succeq \bm{P} \succeq 0, \, \text{Tr}(\bm{P}) \leq k, \\
        & & & \begin{pmatrix}M \bm{P} & \bm{X}\\ \bm{X}^T & M \bm{I}_m\end{pmatrix} \succeq 0, \, \begin{pmatrix}\bm{W}_1 & \bm{X}\\ \bm{X}^T & \bm{W}_2\end{pmatrix} \succeq 0.
    \end{aligned} \label{opt_lrml4:MC_SDP_relax}
\end{equation} We now have the following result:

\begin{proposition}
    Problem \eqref{opt_lrml4:MC_SDP_relax} is a valid convex relaxation of \eqref{opt_lrml4:MC_primal_UVPk}.
\end{proposition}

\begin{proof}
    Problem \eqref{opt_lrml4:MC_SDP_relax} is clearly a convex optimization problem. We will show that the optimal value of \eqref{opt_lrml4:MC_SDP_relax} is a lower bound on the optimal value of \eqref{opt_lrml4:MC_primal_UVPk} by showing that given any optimal solution to  \eqref{opt_lrml4:MC_primal_UVPk}, we can construct a feasible solution to \eqref{opt_lrml4:MC_SDP_relax} that achieves the same objective value.

    Consider any optimal solution $(\bm{\Bar{U}}, \bm{\Bar{V}}, \bm{\Bar{P}})$ to \eqref{opt_lrml4:MC_primal_UVPk}. From the proof of Proposition \ref{prop_lrml4:UVP_reform}, we know that the solution $(\bm{\Bar{X}}, \bm{\Bar{P}})$ where $\bm{\Bar{X}} = \bm{\Bar{U}} \bm{\Bar{V}}^T$ is feasible to \eqref{opt_lrml4:MC_primal_XP} (where we replace the constraint $\bm{P} \in \mathcal{P}_{\min(k, \text{rank}(\bm{U}\bm{V}^T))}$ with $\bm{P} \in \mathcal{P}_k$) and must also be optimal. Let $\bm{\Bar{X}} = \bm{L} \bm{\Sigma} \bm{R}^T$ be a singular value decomposition of $\bm{\Bar{X}}$ with $\bm{L} \in \mathbb{R}^{n \times k}, \bm{\Sigma} \in \mathbb{R}^{k \times k}$ and $\bm{R} \in \mathbb{R}^{m \times k}$. Let $\bm{\Bar{W}}_1 = \bm{L}\bm{\Sigma}\bm{L}^T$ and $\bm{\Bar{W}}_2 = \bm{R}\bm{\Sigma}\bm{R}^T$. We claim that $(\bm{\Bar{X}}, \bm{\Bar{P}}, \bm{\Bar{W}}_1, \bm{\Bar{W}}_2)$ is feasible to \eqref{opt_lrml4:MC_SDP_relax} and achieves the same objective value as $(\bm{\Bar{X}}, \bm{\Bar{P}})$ does in \eqref{opt_lrml4:MC_primal_XP}.

    From the feasibility of $\bm{\Bar{P}}$ in \eqref{opt_lrml4:MC_primal_UVPk}, we know that $\bm{\Bar{P}} \in \mathcal{P}_k$ which implies $\bm{I}_n \succeq \bm{\Bar{P}} \succeq 0$ and $\text{Tr}(\bm{\Bar{P}}) \leq k$. By the generalized Schur complement lemma (see \cite{boyd1994linear}, Equation 2.41), we know that \[\begin{pmatrix}M \bm{\Bar{P}} & \bm{\Bar{X}}\\ \bm{\Bar{X}}^T & M \bm{I}_m\end{pmatrix} \succeq 0 \iff M \bm{I}_m \succeq 0, \text{ and } M \bm{I}_m - \bm{\Bar{X}}^T(M\bm{\Bar{P}})^\dagger \bm{\Bar{X}} \succeq 0.\] We trivially have $M \bm{I}_m \succeq 0$. To see that the second condition holds, note that since $\bm{\Bar{P}}$ is a projection matrix and $\bm{\Bar{P}}\bm{\Bar{X}} = \bm{\Bar{X}}$, we have $\bm{\Bar{X}}^T(M\bm{\Bar{P}})^\dagger \bm{\Bar{X}} = \frac{1}{M}\bm{\Bar{X}}^T\bm{\Bar{P}}\bm{\Bar{X}} = \frac{1}{M}\bm{\Bar{X}}^T\bm{\Bar{X}}$. Furthermore, since $\bm{\Bar{X}}$ is optimal to \eqref{opt_lrml4:MC_primal_XP}, we have $\Vert \bm{\Bar{X}} \Vert_\sigma \leq M$ by assumption. Thus, we have \[\Vert \bm{\Bar{X}} \Vert_\sigma \leq M \implies \Vert \bm{\Bar{X}}^T\bm{\Bar{X}} \Vert_\sigma \leq M^2 \implies M^2\bm{I}_m \succeq \bm{\Bar{X}}^T\bm{\Bar{X}} \implies M \bm{I}_m \succeq \frac{1}{M}\bm{\Bar{X}}^T\bm{\Bar{X}}.\] Finally, observe that \[\begin{pmatrix}\bm{\Bar{W}}_1 & \bm{\Bar{X}}\\ \bm{\Bar{X}}^T & \bm{\Bar{W}}_2\end{pmatrix} = \begin{pmatrix}\bm{L}\bm{\Sigma}\bm{L}^T & \bm{L} \bm{\Sigma} \bm{R}^T\\ \bm{R} \bm{\Sigma} \bm{L}^T & \bm{R}\bm{\Sigma}\bm{R}^T\end{pmatrix} = \begin{pmatrix}\bm{L} \\ \bm{R}\end{pmatrix} \bm{\Sigma} \begin{pmatrix}\bm{L} \\ \bm{R}\end{pmatrix}^T.\] Since $\bm{\Sigma}$ is a diagonal matrix with non negative entries, the matrix $\begin{pmatrix}\bm{\Bar{W}}_1 & \bm{\Bar{X}}\\ \bm{\Bar{X}}^T & \bm{\Bar{W}}_2\end{pmatrix}$ is certainly positive semidefinite. Thus we have shown that $(\bm{\Bar{X}}, \bm{\Bar{P}}, \bm{\Bar{W}}_1, \bm{\Bar{W}}_2)$ is indeed feasible to \eqref{opt_lrml4:MC_SDP_relax}. To conclude the proof, we note that
    \begin{equation*}
        \begin{aligned}
            \frac{\gamma}{2}(\text{Tr}(\bm{\Bar{W}}_1)+\text{Tr}(\bm{\Bar{W}}_2)) &= \frac{\gamma}{2}(\text{Tr}(\bm{L}\bm{\Sigma}\bm{L}^T)+\text{Tr}(\bm{R}\bm{\Sigma}\bm{R}^T)) \\
            &= \frac{\gamma}{2}(\text{Tr}(\bm{L}^T\bm{L}\bm{\Sigma})+\text{Tr}(\bm{R}^T\bm{R}\bm{\Sigma})) \\
            &= \frac{\gamma}{2}(\text{Tr}(\bm{\Sigma})+\text{Tr}(\bm{\Sigma})) = \gamma \Vert \bm{\Bar{X}} \Vert_\star,
        \end{aligned}
    \end{equation*} thus $(\bm{\Bar{X}}, \bm{\Bar{P}}, \bm{\Bar{W}}_1, \bm{\Bar{W}}_2)$ achieves the same objective value in \eqref{opt_lrml4:MC_SDP_relax} as $(\bm{\Bar{X}}, \bm{\Bar{P}})$ achieves in \eqref{opt_lrml4:MC_primal_XP}.
\end{proof} In general, an optimal solution to \eqref{opt_lrml4:MC_SDP_relax} will have $\bm{P} \notin \mathcal{P}_k$. We briefly note that to obtain a stronger convex relaxation, one could leverage eigenvector disjunctions \citep{bertsimas2023optimal, saxena2010convex} to iteratively cut off solutions to \eqref{opt_lrml4:MC_SDP_relax} with $\bm{P} \notin \mathcal{P}_k$ and form increasingly tighter disjunctive approximations to the set $\mathcal{P}_k$.

\section{Mixed-Projection ADMM} \label{sec_lrml4:admm}

 In this section, we present an alternating direction method of multipliers (ADMM) algorithm that is scalable and obtains high quality solutions for \eqref{opt_lrml4:MC_primal_UVPk} and we investigate its convergence properties. Rather than forming the augmented Lagrangian directly for \eqref{opt_lrml4:MC_primal_UVPk}, we first modify our problem formulation by introducing a dummy variable $\bm{Z} \in \mathbb{R}^{n \times k}$ that is an identical copy of $\bm{U}$. Additionally, rather than directly enforcing the constraint $\bm{P} \in \mathcal{P}_k$, we introduce an indicator function penalty $\mathbb{I}_{\mathcal{P}_k}(\bm{P})$ into the objective function where $\mathbb{I}_{\mathcal{X}}(\bm{x}) = 0$ if $\bm{x} \in \mathcal{X}$, otherwise $\mathbb{I}_{\mathcal{X}}(\bm{x}) = \infty$. Explicitly, we consider the following problem:

\begin{equation}
    \begin{aligned}
        &\min_{\bm{U}, \bm{Z} \in \mathbb{R}^{n \times k}, \bm{V} \in \mathbb{R}^{m \times k}, \bm{P} \in \mathbb{R}^{n \times n}} & & \sum_{(i, j) \in \Omega}((\bm{U}\bm{V}^T)_{ij}-A_{ij})^2 + \lambda \text{Tr}\big{(}\bm{Y}^T(\bm{I}_n - \bm{P})\bm{Y}\big{)} \\
        & & &+ \frac{\gamma}{2}( \Vert \bm{U} \Vert_F^2 + \Vert \bm{V} \Vert_F^2) + \mathbb{I}_{\mathcal{P}_k}(\bm{P}) \\
        &\text{s.t.} & & (\bm{I}_n-\bm{P})\bm{Z} = \bm{0}_{n \times k}, \, \bm{U}-\bm{Z} = \bm{0}_{n \times k}.
    \end{aligned} \label{opt_lrml4:MC_primal_dummy}
\end{equation} It is trivial to see that \eqref{opt_lrml4:MC_primal_dummy} is equivalent to \eqref{opt_lrml4:MC_primal_UVPk}. We will see in this section that working with formulation \eqref{opt_lrml4:MC_primal_dummy} leads to an ADMM algorithm with favorable decomposition properties. Introducing dual variables $\bm{\Phi}, \bm{\Psi} \in \mathbb{R}^{n \times k}$ for the constraints $(\bm{I}_n-\bm{P})\bm{U} = \bm{0}_{n \times k}$ and $\bm{U}-\bm{Z} = \bm{0}_{n \times k}$ respectively, the augmented Lagrangian $\mathcal{L}^A$ for \eqref{opt_lrml4:MC_primal_dummy} is given by:

\begin{equation}
    \begin{aligned}
        \mathcal{L}^A(\bm{U}, \bm{V}, \bm{P}, \bm{Z}, \bm{\Phi}, \bm{\Psi}) &= \sum_{(i, j) \in \Omega}((\bm{U}\bm{V}^T)_{ij}-A_{ij})^2 + \lambda \text{Tr}\big{(}\bm{Y}^T(\bm{I}_n - \bm{P})\bm{Y}\big{)} \\ &+ \frac{\gamma}{2}( \Vert \bm{U} \Vert_F^2 + \Vert \bm{V} \Vert_F^2)
        + \mathbb{I}_{\mathcal{P}_k}(\bm{P}) + \text{Tr}(\bm{\Phi}^T(\bm{I}_n-\bm{P})\bm{Z}) \\ &+ \text{Tr}(\bm{\Psi}^T(\bm{Z}-\bm{U}))
        +\frac{\rho_1}{2}\Vert (\bm{I}_n-\bm{P})\bm{Z} \Vert_F^2 +\frac{\rho_2}{2}\Vert \bm{Z} - \bm{U} \Vert_F^2,
    \end{aligned} \label{eq:lagrangian}
\end{equation} where $\rho_1, \rho_2 > 0$ are non-negative penalty parameters. In what follows, we show that performing a partial minimization of the augmented Lagrangian \eqref{eq:lagrangian} over each of the primal variables $\bm{U}, \bm{V}, \bm{P}, \bm{Z}$ yields a subproblem that can be solved efficiently. We present each subproblem and investigate the complexity of computing the corresponding subproblem solutions.

\subsection{Subproblem in \texorpdfstring{$\bm{U}$}{U}}

First, suppose we fix variables $\bm{V}, \bm{P}, \bm{Z}, \bm{\Phi}, \bm{\Psi}$ and seek to minimize $\mathcal{L}^A(\bm{U}, \bm{V}, \bm{P}, \bm{Z}, \bm{\Phi}, \bm{\Psi})$ over $\bm{U}$. Eliminating terms that do not depend on $\bm{U}$, the resulting subproblem is given by

\begin{equation}
    \begin{aligned}
        \min_{\bm{U} \in \mathbb{R}^{n \times k}} \sum_{(i, j) \in \Omega}((\bm{U}\bm{V}^T)_{ij}-A_{ij})^2 + \frac{\gamma}{2}\Vert \bm{U} \Vert_F^2 - \text{Tr}(\bm{\Psi}^T\bm{U}) +\frac{\rho_2}{2}\Vert \bm{Z} - \bm{U} \Vert_F^2.
    \end{aligned} \label{opt_lrml4:U_prob}
\end{equation}
We now have the following result:

\begin{proposition} \label{prop_lrml4:U_sol}
    The optimal solution $\bm{\Bar{U}}$ for \eqref{opt_lrml4:U_prob} is given by
    \begin{equation}\label{eq:opt_U}
        \Bar{U}_{i, \star} = [2\bm{V}^T\bm{W}_i\bm{V}+(\gamma + \rho_2)\bm{I}_k]^{-1}[2\bm{V}^T\bm{W}_iA_{i, \star}+\Psi_{i, \star}+\rho_2Z_{i, \star}],
    \end{equation} for each $i \in \{1, \ldots, n\}$ where each $\bm{W}_i \in \mathbb{R}^{m \times m}$ is a diagonal matrix satisfying $(\bm{W}_i)_{jj} = 1$ if $(i, j) \in \Omega$, otherwise $(\bm{W}_i)_{jj} = 0$. Here, the column vectors $\Bar{U}_{i, \star} \in \mathbb{R}^k, A_{i, \star} \in \mathbb{R}^m, \Psi_{i, \star} \in \mathbb{R}^k, Z_{i, \star} \in \mathbb{R}^k$ denote the $i^{th}$ row of the matrices $\bm{\Bar{U}}, \bm{A}, \bm{\Psi}, \bm{Z}$ respectively, where the unknown entries of $\bm{A}$ are taken to be $0$.
\end{proposition}

\begin{proof}
    Let $f(\bm{U})$ denote the objective function of \eqref{opt_lrml4:U_prob}. With $\{\bm{W}_i\}_{i=1}^n$ defined as in Proposition \ref{prop_lrml4:U_sol}, observe that we can write $f(\bm{U})$ as
    \begin{equation*}
        \begin{aligned}
            f(\bm{U}) &= \sum_{i=1}^n \bigg{[}\Vert \bm{W}_i(\bm{V}U_{i, \star} - A_{i, \star}) \Vert_2^2 + \frac{\gamma}{2} \Vert U_{i, \star} \Vert_2^2 -  \Psi_{i,\star}^TU_{i, \star}+ \frac{\rho_2}{2} \Vert Z_{i, \star} - U_{i, \star} \Vert_2^2\bigg{]} \\
            &= \sum_{i=1}^ng_i(\bm{U}),
        \end{aligned}
    \end{equation*} where we define $g_i(\bm{U}) = \Vert \bm{W}_i(\bm{V}U_{i, \star} - A_{i, \star}) \Vert_2^2 + \frac{\gamma}{2} \Vert U_{i, \star} \Vert_2^2 -  \Psi_{i,\star}^TU_{i, \star}+ \frac{\rho_2}{2} \Vert Z_{i, \star} - U_{i, \star} \Vert_2^2$. Thus, we have shown that $f(\bm{U})$ is separable over the rows of the matrix $\bm{U}$. Each function $g_i(\bm{U})$ is a (strongly) convex quadratic. Thus, we can minimize $g_i(\bm{U})$ by setting its gradient to $0$. For any fixed row $i \in \{1, \ldots, n\}$, we can differentiate and obtain \[ \nabla_{\bm{U}_{i, \star}}g_i(\bm{U}) = 2\bm{V}^T\bm{W}_i(\bm{V}U_{i, \star} - A_{i \star}) + \gamma U_{i, \star}-\Psi_{i, \star}-\rho_2(Z_{i, \star}-U_{i, \star}).\] By equating the gradient $\nabla_{\bm{U}_{i, \star}}g_i(\bm{U})$ to $0$ and rearranging, we obtain that the optimal vector $\Bar{U}_{i, \star}$ is given by \eqref{eq:opt_U}. This completes the proof.
\end{proof}

Observe that since the matrix $\bm{V}^T\bm{W}_i\bm{V}$ is positive semidefinite and $\gamma + \rho_2 > 0$, the matrix inverse $[2\bm{V}^T\bm{W}_i\bm{V}+(\gamma + \rho_2)\bm{I}_k]^{-1}$ is well defined for all $i \in \{1, \ldots, n\}$. Computing the optimal solution to \eqref{opt_lrml4:U_prob} requires computing $n$ different $k \times k$ matrix inverses (where in general $k \ll \min \{m,n\}$). Computing a single $k \times k$ matrix inverse requires $O(k^3)$ time and forming the matrix product $\bm{V}^T\bm{W}_i\bm{V}$ requires $O(k^2m)$ time for a given $i$. Thus, the complexity of computing the optimal solution for a single column is $O(k^3+k^2m)$. Notice that each column of $\bm{\Bar{U}}$ can be computed independently of the other columns. We leverage this observation by developing a multi-threaded implementation of the algorithm presented in this section. Letting $w$ denote the number of compute threads available, computing the optimal solution $\bm{\Bar{U}}$ of \eqref{opt_lrml4:U_prob} requires $O\Big{(}\frac{k^3n+k^2mn}{\min\{w,n\}}\Big{)}$ time (the term $\min\{w,n\}$ in the denominator reflects that fact that increasing the number of available compute threads beyond the number of columns of $\bm{\Bar{U}}$ does not yield additional reduction in compute complexity).

\subsection{Subproblem in \texorpdfstring{$\bm{V}$}{V}}

Now, suppose we fix variables $\bm{U}, \bm{P}, \bm{Z}, \bm{\Phi}, \bm{\Psi}$ and seek to minimize $\mathcal{L}^A(\bm{U}, \bm{V}, \bm{P}, \bm{Z}, \bm{\Phi}, \bm{\Psi})$ over $\bm{V}$. Eliminating terms that do not depend on $\bm{V}$, the resulting subproblem is given by

\begin{equation}
    \begin{aligned}
        \min_{\bm{V} \in \mathbb{R}^{m \times k}} \sum_{(i, j) \in \Omega}((\bm{U}\bm{V}^T)_{ij}-A_{ij})^2 + \frac{\gamma}{2}\Vert \bm{V} \Vert_F^2.
    \end{aligned} \label{opt_lrml4:V_prob}
\end{equation}
We now have the following result:

\begin{proposition} \label{prop_lrml4:V_sol}
    The optimal solution $\bm{\Bar{V}}$ for \eqref{opt_lrml4:V_prob} is given by
    \begin{equation} \label{eq:opt_V}
        \Bar{V}_{j, \star} = [2\bm{U}^T\bm{W}_j\bm{U}+\gamma\bm{I}_k]^{-1}[2\bm{U}^T\bm{W}_jA_{\star, j}],
    \end{equation}
    for each $j \in \{1, \ldots, m\}$ where each $\bm{W}_j \in \mathbb{R}^{n \times n}$ is a diagonal matrix satisfying $(\bm{W}_j)_{ii} = 1$ if $(i, j) \in \Omega$, otherwise $(\bm{W}_j)_{ii} = 0$. Here, the column vector $\Bar{V}_{j, \star} \in \mathbb{R}^k$ denotes the $j^{th}$ row of $\bm{\Bar{V}}$ while the column vector $A_{\star, j} \in \mathbb{R}^n$ denotes the $j^{th}$ column of $\bm{A}$ where the unknown entries of $\bm{A}$ are taken to be $0$.
\end{proposition}

\begin{proof}
    This proof follows the proof of Proposition \ref{prop_lrml4:U_sol}. Let $f(\bm{V})$ denote the objective function of \eqref{opt_lrml4:V_prob}. With $\{\bm{W}_j\}_{j=1}^m$ defined as in Proposition \ref{prop_lrml4:V_sol}, observe that we can write $f(\bm{V})$ as
    \begin{equation*}
        \begin{aligned}
            f(\bm{V}) &= \sum_{j=1}^m \bigg{[}\Vert \bm{W}_j(\bm{U}V_{j, \star} - A_{\star, j}) \Vert_2^2 + \frac{\gamma}{2}\Vert V_{j, \star} \Vert_2^2\bigg{]} = \sum_{j=1}^mg_j(\bm{V}),
        \end{aligned}
    \end{equation*} where we define $g_j(\bm{V}) = \Vert \bm{W}_j(\bm{U}V_{j, \star} - A_{\star, j}) \Vert_2^2 + \frac{\gamma}{2}\Vert V_{j, \star} \Vert_2^2$. Thus, we have shown that $f(\bm{V})$ is separable over the rows of the matrix $\bm{V}$. Each function $g_j(\bm{V})$ is a (strongly) convex quadratic. Thus, we can minimize $g_j(\bm{V})$ by setting its gradient to $0$. For any fixed row $j \in \{1, \ldots, m\}$, we can differentiate and obtain \[ \nabla_{\bm{V}_{j, \star}}g_j(\bm{V}) = 2\bm{U}^T\bm{W}_j(\bm{U}V_{j, \star} - A_{\star, j}) + \gamma V_{j, \star}.\] By equating the gradient $\nabla_{\bm{V}_{j, \star}}g_j(\bm{V})$ to $0$ and rearranging, we obtain that the optimal vector $\Bar{U}_{i, \star}$ is given by \eqref{eq:opt_V}. This completes the proof.
\end{proof}

Observe that since the matrix $\bm{U}^T\bm{W}_j\bm{U}$ is positive semidefinite and $\gamma > 0$, the matrix inverse $[2\bm{U}^T\bm{W}_j\bm{U}+\gamma\bm{I}_k]^{-1}$ is well defined for all $j \in \{1, \ldots, m\}$. Computing the optimal solution to \eqref{opt_lrml4:V_prob} requires computing $m$ different $k \times k$ matrix inverses. Forming the matrix product $\bm{U}^T\bm{W}_j\bm{U}$ requires $O(k^2n)$ time for a given $j$. Thus, the complexity of computing the optimal solution for a single column is $O(k^3+k^2n)$. Notice that, similarly to the solution of \eqref{opt_lrml4:U_prob}, each column of $\bm{\Bar{V}}$ can be computed independently of the other columns. As before, we leverage this observation in our multi-threaded implementation of the algorithm presented in this section. Letting $w$ denote the number of compute threads available, computing the optimal solution $\bm{\Bar{V}}$ of \eqref{opt_lrml4:V_prob} requires $O\Big{(}\frac{k^3m+k^2mn}{\min\{w,m\}}\Big{)}$ time. 

The optimal solution $\bm{\Bar{V}}$ to \eqref{opt_lrml4:V_prob} reveals that the Frobenius norm regularization term on $\bm{V}$ in \eqref{opt_lrml4:MC_primal_UVPk} (which emerges from the nuclear norm regularization term on $\bm{X}$ in \eqref{opt_lrml4:MC_primal}) has computational benefits. Indeed, if we had $\gamma = 0$, it is possible that the matrix $\bm{U}^T\bm{W}_j\bm{U}$ be singular at certain iterates of our ADMM algorithm, in which case the corresponding matrix inverse would be undefined. This observation is in keeping with several recent works in the statistics, machine learning and operations research literatures where the presence of a regularization penalty in the objective function yields improved out of sample performance as well as benefits in computational tractability (see for example \cite{bertsimas2020rejoinder, bertsimas2021unified, bertsimas2021perspective, bertsimas2023slr,  bertsimas2023compressed}).

\subsection{Subproblem in \texorpdfstring{$\bm{P}$}{P}}

Now, suppose we fix variables $\bm{U}, \bm{V}, \bm{Z}, \bm{\Phi}, \bm{\Psi}$ and seek to minimize $\mathcal{L}^A(\bm{U}, \bm{V}, \bm{P}, \bm{Z}, \bm{\Phi}, \bm{\Psi})$ over $\bm{P}$. Eliminating terms that do not depend on $\bm{P}$, the resulting subproblem is given by

\begin{equation}
    \begin{aligned}
        \min_{\bm{P} \in \mathbb{S}^n_+} -\lambda \text{Tr}(\bm{Y}^T\bm{P}\bm{Y}) - \text{Tr}(\bm{\Phi}^T\bm{P}\bm{Z})
        +\frac{\rho_1}{2}\Vert (\bm{I}_n-\bm{P})\bm{Z} \Vert_F^2 \quad \text{s.t.} \quad \bm{P}^2=\bm{P}, \, \text{Tr}(\bm{P}) \leq k.
    \end{aligned} \label{opt_lrml4:P_prob}
\end{equation}
We now have the following result:

\begin{proposition} \label{prop_lrml4:P_sol}
    Let $\bm{M}\bm{\Sigma}\bm{M}^T$ be a rank $k$ truncated singular value decomposition for the matrix given by:
    \[\Big{(} \lambda \bm{Y}\bm{Y}^T+\frac{\rho_1}{2}\bm{Z}\bm{Z}^T+\frac{1}{2}(\bm{\Phi}\bm{Z}^T+\bm{Z}\bm{\Phi}^T) \Big{)},\] where $\bm{\Sigma} \in \mathbb{R}^{k \times k}, \bm{M} \in \mathbb{R}^{n \times k}, \bm{M}^T\bm{M} = \bm{I}_k$.The optimal solution $\bm{\Bar{P}}$ for \eqref{opt_lrml4:P_prob} is given by $\bm{\Bar{P}} = \bm{M}\bm{M}^T$.
\end{proposition}

\begin{proof}
    Let $f(\bm{P})$ denote the objective function of \eqref{opt_lrml4:P_prob}. Observe that for any $\bm{P}$ that is feasible to \eqref{opt_lrml4:P_prob}, we can write $f(\bm{P})$ as: 
    \begin{equation*}
        \begin{aligned}
            f(\bm{P}) &= -\lambda \text{Tr}(\bm{Y}^T\bm{P}\bm{Y}) - \text{Tr}(\bm{\Phi}^T\bm{P}\bm{Z})
        +\frac{\rho_1}{2}\Vert (\bm{I}_n-\bm{P})\bm{Z} \Vert_F^2 \\
        &= -\lambda \text{Tr}(\bm{Y}\bm{Y}^T\bm{P}) - \text{Tr}(\bm{Z}\bm{\Phi}^T\bm{P})
        +\frac{\rho_1}{2}\text{Tr}(\bm{Z}\bm{Z}^T(\bm{I}_n-\bm{P})) \\
        &= \frac{\rho_1}{2}\text{Tr}(\bm{Z}\bm{Z}^T)- \langle \lambda\bm{Y}\bm{Y}^T + \bm{Z}\bm{\Phi}^T+\frac{\rho_1}{2}\bm{Z}\bm{Z}^T, \bm{P} \rangle.
        \end{aligned}
    \end{equation*} Thus, it is immediately clear that a solution will be optimal to \eqref{opt_lrml4:P_prob} if and only if it is optimal to the problem given by:
    \begin{equation}
        \begin{aligned}
            \max_{\bm{P} \in \mathbb{S}^n_+} \langle \bm{C}, \bm{P} \rangle \quad \text{s.t.} \quad \bm{P}^2=\bm{P}, \, \text{Tr}(\bm{P}) \leq k,
        \end{aligned} \label{opt_lrml4:P_prob_reduced}
    \end{equation} where we define the matrix $\bm{C} \in \mathbb{R}^{n \times n}$ as $\bm{C} = \lambda\bm{Y}\bm{Y}^T + \bm{Z}\bm{\Phi}^T+\frac{\rho_1}{2}\bm{Z}\bm{Z}^T$. Let $\bm{\Bar{C}} = \frac{1}{2}(\bm{C}+\bm{C}^T)$ denote the symmetric part of $\bm{C}$. Observe that for any symmetric matrix $\bm{P}$, we can consider $\langle \bm{\Bar{C}}, \bm{P} \rangle$ in place of $\langle \bm{C}, \bm{P} \rangle$ since we have
    \begin{equation*}
        \begin{aligned}
            \langle \bm{C}, \bm{P} \rangle &= \sum_{i=1}^n\sum_{j=1}^nP_{ij}C_{ij}=\sum_{i=1}^nP_{ii}C_{ii}+\sum_{i=1}^{n-1}\sum_{j=i+1}^nP_{ij}(C_{ij}+C_{ji}) \\
            &= \sum_{i=1}^nP_{ii}\Bar{C}_{ii}+2\sum_{i=1}^{n-1}\sum_{j=i+1}^nP_{ij}\Bar{C}_{ij}=\sum_{i=1}^n\sum_{j=1}^nP_{ij}\Bar{C}_{ij}=\langle \bm{\Bar{C}}, \bm{P} \rangle.
        \end{aligned}
    \end{equation*} Let $\bm{\Bar{C}} = \bm{M}\bm{\Sigma}\bm{M}^T$ be a full singular value decomposition of $\bm{\Bar{C}}$ with $\bm{M}, \bm{\Sigma} \in \mathbb{R}^{n \times n}, \bm{M}^T\bm{M} = \bm{M}\bm{M}^T=\bm{I}_n$. The matrix $\bm{\Sigma}$ is the diagonal matrix of (ordered) singular values of $\bm{\Bar{C}}$ and we let $\sigma_i$ denote the $i^{th}$ singular value. Any feasible matrix $\bm{P}$ to \eqref{opt_lrml4:P_prob_reduced} can be written as $\bm{P} = \bm{L}\bm{L}^T$ where $\bm{L} \in \mathbb{R}^{n \times k}, \bm{L}^T\bm{L} = \bm{I}_k$. Thus, for any $\bm{P}$ feasible to \eqref{opt_lrml4:P_prob_reduced} we express the objective value as:
    \begin{equation*}
            \langle \bm{\Bar{C}}, \bm{P} \rangle = \text{Tr}(\bm{M}\bm{\Sigma}\bm{M}^T\bm{L}\bm{L}^T) = \text{Tr}(\bm{\Sigma}(\bm{M}^T\bm{L}\bm{L}^T\bm{M})) = \sum_{i=1}^n \sigma_i \Vert (\bm{M}^T\bm{L})_{i, \star} \Vert_2^2.
    \end{equation*} Let $\bm{N} = \bm{M}^T\bm{L} \in \mathbb{R}^{n \times k}$. Note that we have $\bm{N}^T\bm{N} = \bm{L}^T\bm{M}\bm{M}^T\bm{L}=\bm{L}^T\bm{L}=\bm{I}_k$, which implies that the columns of $\bm{N}$ are orthonormal. This immediately implies that we have $N_{i,\star}^TN_{i, \star} = \Vert (\bm{M}^T\bm{L})_{i, \star} \Vert_2^2 \leq 1$. Moreover, we have \[\sum_{i=1}^n \Vert (\bm{M}^T\bm{L})_{i, \star} \Vert_2^2 = \sum_{i=1}^n N_{i,\star}^TN_{i, \star} = \sum_{i=1}^n\sum_{j=1}^kN_{ij}^2=\sum_{j=1}^k\sum_{i=1}^nN_{ij}^2 =\sum_{j=1}^k 1=k.\] We can therefore upper bound the optimal objective value of \eqref{opt_lrml4:P_prob_reduced} as \[\langle \bm{\Bar{C}}, \bm{P} \rangle = \sum_{i=1}^n \sigma_i \Vert (\bm{M}^T\bm{L})_{i, \star} \Vert_2^2 \leq \sum_{i=1}^k \sigma_i.\] To conclude the proof, notice that by taking $\bm{\Bar{P}} = \bm{\Bar{M}}\bm{\Bar{M}}^T$ where $\bm{\Bar{M}} \in \mathbb{R}^{n \times k}$ is the matrix that consists of the first $k$ columns of $\bm{M}$ we can achieve the upper bound on \eqref{opt_lrml4:P_prob_reduced}:
    \[\langle \bm{\Bar{C}}, \bm{\Bar{P}} \rangle = \text{Tr}(\bm{M}\bm{\Sigma}\bm{M}^T\bm{\Bar{M}}\bm{\Bar{M}}^T) = \text{Tr}(\bm{\Bar{M}}^T\bm{M}\bm{\Sigma}\bm{M}^T\bm{\Bar{M}}) = \sum_{i=1}^k\sigma_i.\] \end{proof}

To compute the optimal solution of \eqref{opt_lrml4:P_prob}, we need to compute a rank $k$ singular value decomposition of the matrix $\bm{\Bar{C}} = \big{(} \lambda \bm{Y}\bm{Y}^T+\frac{\rho_1}{2}\bm{Z}\bm{Z}^T+\frac{1}{2}(\bm{\Phi}\bm{Z}^T+\bm{Z}\bm{\Phi}^T) \big{)}$ which requires $O(kn^2)$ time since $\bm{\Bar{C}} \in \mathbb{R}^{n \times n}$. Moreover, explicitly forming the matrix $\bm{\Bar{C}}$ in memory from its constituent matrices $\bm{Y}, \bm{Z}, \bm{\Phi}$ requires $O(n^2(d+k))$ operations. Thus, naively computing the optimal solution to \eqref{opt_lrml4:P_prob} has complexity $O(n^2(d+k))$ where the bottleneck operation from a complexity standpoint is explicitly forming the matrix $\bm{\Bar{C}}$.

Fortunately, it is possible to compute the optimal solution to \eqref{opt_lrml4:P_prob} more efficiently. Observe that we can equivalently express the matrix $\bm{\Bar{C}}$ as $\bm{\Bar{C}} = \bm{F}_1\bm{F}_2^T$ where $\bm{F}_1, \bm{F}_2 \in \mathbb{R}^{n \times (d + 3k)}$ are defined as
\begin{equation*}
    \begin{aligned}
        \bm{F}_1 = \begin{pmatrix}\sqrt{\lambda} \bm{Y} & \vline & \sqrt{\frac{\rho_1}{2}}\bm{Z} & \vline & \sqrt{\frac{1}{2}}\bm{\Phi} & \vline & \sqrt{\frac{1}{2}}\bm{Z} \end{pmatrix}, \\ \bm{F}_2 = \begin{pmatrix}\sqrt{\lambda} \bm{Y} & \vline & \sqrt{\frac{\rho_1}{2}}\bm{Z} & \vline & \sqrt{\frac{1}{2}}\bm{Z} & \vline & \sqrt{\frac{1}{2}}\bm{\Phi} \end{pmatrix}.
    \end{aligned}
\end{equation*} Computing a truncated singular value decomposition requires only computing repeated matrix vector products. Therefore, rather than explicitly forming the matrix $\bm{\Bar{C}}$ in memory at a cost of $O(n^2(d+k))$ operations, in our implementation we design a custom matrix class where matrix vector products between $\bm{\Bar{C}}$ and arbitrary vectors $\bm{x} \in \mathbb{R}^n$ are computed by first evaluating the matrix vector product $\bm{\nu} = \bm{F}_2^T\bm{x}$ and subsequently evaluating the matrix vector product $\bm{\Bar{C}}\bm{x}=\bm{F}_1\bm{\nu}$. In so doing, we can evaluate matrix vector products $\bm{\Bar{C}}\bm{x}$ in $O(n(d+k))$ time rather than $O(n^2)$ time (in general, we will have $d+k \ll n$). Computing a truncated singular value decomposition of $\bm{\Bar{C}}$ with this methodology of evaluating matrix vector products requires only $O(k^2n+knd)$ operations. Thus, our custom matrix class implementation avoids needing to explicitly for $\bm{\Bar{C}}$ in memory and allows the optimal solution to \eqref{opt_lrml4:P_prob} to be computed in $O(k^2n+knd)$ time.

\subsection{Subproblem in \texorpdfstring{$\bm{Z}$}{Z}}

Now, suppose we fix variables $\bm{U}, \bm{V}, \bm{P}, \bm{\Phi}, \bm{\Psi}$ and seek to minimize $\mathcal{L}^A(\bm{U}, \bm{V}, \bm{P}, \bm{Z}, \bm{\Phi}, \bm{\Psi})$ over $\bm{Z}$. Eliminating terms that do not depend on $\bm{Z}$, the resulting subproblem is given by

\begin{equation}
    \begin{aligned}
        \min_{\bm{Z} \in \mathbb{R}^{n \times k}} \text{Tr}(\bm{\Phi}^T(\bm{I}_n-\bm{P})\bm{Z}) + \text{Tr}(\bm{\Psi}^T\bm{Z})
        +\frac{\rho_1}{2}\Vert (\bm{I}_n-\bm{P})\bm{Z} \Vert_F^2 +\frac{\rho_2}{2}\Vert \bm{Z} - \bm{U} \Vert_F^2.
    \end{aligned} \label{opt_lrml4:Z_prob}
\end{equation}
We now have the following result:

\begin{proposition} \label{prop_lrml4:Z_sol}
    The optimal solution $\bm{\Bar{Z}}$ for \eqref{opt_lrml4:Z_prob} is given by
    \begin{equation} \label{eq:opt_Z}
        \begin{aligned}
        \bm{\Bar{Z}} &= \frac{1}{\rho_1+\rho_2}\Big{(}\bm{I}_n+\frac{\rho_1}{\rho_2}\bm{P}\Big{)}\Big{(}\rho_2\bm{U}-(\bm{I}_n-\bm{P})\bm{\Phi} - \bm{\Psi}\Big{)} \\
        &= \frac{1}{\rho_1+\rho_2}\Big{(}\rho_2\bm{U}-\bm{\Phi}+\bm{P}\bm{\Phi}-\bm{\Psi}+\rho_1\bm{P}\bm{U}-\frac{\rho_1}{\rho_2}\bm{P}\bm{\Psi}\Big{)}.
        \end{aligned}
    \end{equation}
\end{proposition}

\begin{proof}
    Let $f(\bm{Z})$ denote the objective function of \eqref{opt_lrml4:Z_prob}. The function $f(\bm{Z})$ is a convex quadratic, thus it can be minimized by setting its gradient to $0$. Differentiating $f(\bm{Z})$, we obtain: \[ \nabla_{\bm{Z}}f(\bm{Z}) = (\bm{I}_n-\bm{P})^T\bm{\Phi}+\bm{\Psi}+\rho_1(\bm{I}_n-\bm{P})^T(\bm{I}_n-\bm{P})\bm{Z}+\rho_2(\bm{Z}-\bm{U}).\] Moreover, for any matrix $\bm{P}$ for which the augmented Lagrangian \eqref{eq:lagrangian} takes finite value, we will have $\bm{P} \in \mathcal{P}_k$ which implies that $\bm{P}^T=\bm{P}$ and $\bm{P}^2=\bm{P}$. We can therefore simplify $\nabla_{\bm{Z}}f(\bm{Z})$ as:
    \[ \nabla_{\bm{Z}}f(\bm{Z}) = (\bm{I}_n-\bm{P})\bm{\Phi}+\bm{\Psi}+\rho_1(\bm{I}_n-\bm{P})\bm{Z}+\rho_2(\bm{Z}-\bm{U}).\] By equating the gradient $\nabla_{\bm{Z}}f(\bm{Z})$ to $0$ and rearranging, we obtain that the optimal matrix $\bm{\Bar{Z}}$ is given by: \[\bm{\Bar{Z}}=\Big{(}\rho_1(\bm{I}_n-\bm{P})+\rho_2\bm{I}_n\Big{)}^{-1}\Big{(}\rho_2\bm{U}-(\bm{I}_n-\bm{P})\bm{\Phi} - \bm{\Psi}\Big{)}\] To conclude the proof, it remains to show that $\big{(}\rho_1(\bm{I}_n-\bm{P})+\rho_2\bm{I}_n\big{)}^{-1} = \frac{1}{\rho_1+\rho_2}\big{(}\bm{I}_n+\frac{\rho_1}{\rho_2}\bm{P}\big{)}$. Let $\bm{P} = \bm{M}\bm{M}^T$ where $\bm{M} \in \mathbb{R}^{n \times k}, \bm{M}^T\bm{M}=\bm{I}_k$. Such a matrix $\bm{M}$ is guaranteed to exist for any $\bm{P} \in \mathcal{P}_k$. We have 
    \begin{equation*}
        \begin{aligned}
            \rho_1(\bm{I}_n-\bm{P})+\rho_2\bm{I}_n &= \rho_1(\bm{I}_n-\bm{M}\bm{M}^T)+\rho_2\bm{I}_n \\
            &= (\rho_1 + \rho_2)\bm{I}_n+ \bm{M}(-\rho_1\bm{I}_n)\bm{M}^T \\
            &= \frac{1}{\rho_1+\rho_2}\bm{I}_n-\frac{1}{(\rho_1+\rho_2)^2}\bm{M}\bigg{(}\frac{1}{\rho_1+\rho_2}\bm{M}^T\bm{M}-\frac{1}{\rho_1}\bm{I}_k\bigg{)}^{-1}\bm{M}^T \\
            &= \frac{1}{\rho_1+\rho_2}\bm{I}_n-\frac{1}{(\rho_1+\rho_2)^2}\bm{M}\bigg{(}\frac{-\rho_1(\rho_1+\rho_2)}{\rho_2}\bm{I}_k\bigg{)}\bm{M}^T \\
            &= \frac{1}{\rho_1+\rho_2}\Big{(}\bm{I}_n+\frac{\rho_1}{\rho_2}\bm{P}\Big{)},
        \end{aligned}
    \end{equation*} where the third equality follows from the Woodbury matrix inversion lemma (see \cite{petersen2008matrix}, Section 3.2.2). As a sanity check, one can verify that the product of $\big{(}\rho_1(\bm{I}_n-\bm{P})+\rho_2\bm{I}_n\big{)}$ and $\frac{1}{\rho_1+\rho_2}\big{(}\bm{I}_n+\frac{\rho_1}{\rho_2}\bm{P}\big{)}$ is indeed the $n$ dimensional identity matrix.
\end{proof}

Evaluating the optimal solution to \eqref{opt_lrml4:Z_prob} requires only matrix-matrix multiplications. Computing the products of $\bm{P}\bm{\Phi}, \bm{P}\bm{U}, \bm{P}\bm{\Psi}$ in the definition of $\bm{\Bar{Z}}$ from \eqref{eq:opt_Z} requires $O(kn^2)$ operations. Thus, the naive cost of forming $\bm{\Bar{Z}}$ is $O(kn^2)$. However, notice that if we had a factored representation of the matrix $\bm{P}$ as $\bm{P} = \bm{M}\bm{M}^T$ with $\bm{M} \in \mathbb{R}^{n \times k}$, for any matrix $\bm{R} \in \mathbb{R}^{n \times k}$ we could compute matrix-matrix products $\bm{P}\bm{R}$ by first computing $\bm{S} = \bm{M}^T\bm{R}$ and thereafter computing $\bm{P}\bm{R} = \bm{M}\bm{S}$ for a total complexity of $O(k^2n)$. One might object that this ignores the time required to compute such a matrix $\bm{M}$. However, observe that in computing a matrix $\bm{P}$ that is optimal to \eqref{opt_lrml4:P_prob}, we in fact must already generate such a matrix $\bm{M}$ (see proposition \ref{prop_lrml4:P_sol}). In fact, in our implementation we never explicitly form a $n \times n$ matrix $\bm{P}$ as it suffices to only store a copy of its low rank factorization matrix $\bm{M}$. Thus, the optimal solution to \eqref{opt_lrml4:Z_prob} can be evaluated in $O(k^2n)$ time.

\subsection{An ADMM Algorithm}

Having illustrated that the partial minimization of the Lagrangian $\eqref{eq:lagrangian}$ across each of the primal variables (Problems \eqref{opt_lrml4:U_prob}, \eqref{opt_lrml4:V_prob}, \eqref{opt_lrml4:P_prob}, \eqref{opt_lrml4:Z_prob}) can be solved efficiently, we can now present the overall approach Algorithm \ref{alg:ADMM}.

\begin{algorithm}
\caption{Mixed-Projection ADMM}\label{alg:ADMM}
\begin{algorithmic}[1]
\Require $n, m, k \in \mathbb{Z}_+, \Omega \subset [n] \times [m], \{A_{ij}\}_{(i, j) \in \Omega}, \lambda, \gamma \in \mathbb{R}_+$. Tolerance parameter $\epsilon > 0$. Maximum iteration parameter $T \in \mathbb{Z}_+$.
\Ensure $(\bm{\Bar{U}}, \bm{\Bar{V}}, \bm{\Bar{P}})$ that is feasible to \eqref{opt_lrml4:MC_primal_UVPk}.
\State $(\bm{U}_0, \bm{P}_0, \bm{V}_0, \bm{Z}_0) \xleftarrow[]{} (\bm{L}\bm{\Sigma}^{\frac{1}{2}}, \bm{L}\bm{L}^T, \bm{R}\bm{\Sigma}^{\frac{1}{2}}, \bm{L}\bm{\Sigma}^{\frac{1}{2}})$ where $\bm{L}\bm{\Sigma}\bm{R}$ is a rank $k$ truncated SVD of $\bm{A}$ and missing entries are filled in with $0$;
\State $(\bm{\Phi}_0, \bm{\Psi}_0) \xleftarrow[]{} (\bm{1}_{n \times k}, \bm{1}_{n \times k})$;
\State $t \xleftarrow[]{} 0$;
\While{$t < T$ and $\max \{\Vert (\bm{I}_n-\bm{P}_t)\bm{Z}_t\Vert_F^2, \Vert\bm{Z}_t-\bm{U}_t \Vert_F^2\} > \epsilon$}
    \State $(\bm{U}_{t+1}, \bm{P}_{t+1}) \xleftarrow[]{} \argmin_{\bm{U}, \bm{P}}\mathcal{L}^A(\bm{U}, \bm{V}_t,\bm{P}, \bm{Z}_t, \bm{\Phi}_t, \bm{\Psi}_t) $;
    \State $(\bm{V}_{t+1}, \bm{Z}_{t+1}) \xleftarrow[]{} \argmin_{\bm{V}, \bm{Z}}\mathcal{L}^A(\bm{U}_{t+1}, \bm{V},\bm{P}_{t+1}, \bm{Z}, \bm{\Phi}_t, \bm{\Psi}_t) $;
    \State $\bm{\Phi}_{t+1} \xleftarrow[]{} \bm{\Phi}_{t} + \rho_1 (\bm{I} - \bm{P}_{t+1})\bm{Z}_{t+1}$;
    \State $\bm{\Psi}_{t+1} \xleftarrow[]{} \bm{\Psi}_{t} + \rho_2 (\bm{Z}_{t+1} - \bm{U}_{t+1})$;
    \State $t\xleftarrow[]{} t+1 $;
\EndWhile
\State Return $(\bm{U}_t, \bm{V}_t, \bm{P}_t)$.
\end{algorithmic}
\end{algorithm} We initialize primal iterates $\bm{U}_0 = \bm{Z}_0 = \bm{L}\bm{\Sigma}^{\frac{1}{2}}, \bm{P}_0 = \bm{L}\bm{L}^T, \bm{V}_0 = \bm{R}\bm{\Sigma}^{\frac{1}{2}}$ where $\bm{L}\bm{\Sigma}\bm{R}$ denotes a rank $k$ truncated singular value decomposition of $\bm{A}$ (the missing entries of $\bm{A}$ are filled in with $0$s) and we initialize dual iterates $\bm{\Phi}_0 = \bm{\Psi}_0 = \bm{1}_{n \times k}$. Observe that the subproblems \eqref{opt_lrml4:U_prob} and \eqref{opt_lrml4:P_prob} can be solved simultaneously. Similarly, the subprobems \eqref{opt_lrml4:V_prob} and \eqref{opt_lrml4:Z_prob} can be solved simultaneously. At each iteration of Algorithm \ref{alg:ADMM}, we first update the iterates $\bm{U}_{t+1}, \bm{P}_{t+1}$ by solving problems \eqref{opt_lrml4:U_prob} and \eqref{opt_lrml4:P_prob} with $(\bm{V}_t, \bm{Z}_t, \bm{\Phi}_t, \bm{\Psi}_t)$ fixed. Next, we update the iterates $\bm{V}_{t+1}, \bm{Z}_{t+1}$ by solving problems \eqref{opt_lrml4:V_prob} and \eqref{opt_lrml4:Z_prob} with $(\bm{U}_{t+1}, \bm{P}_{t+1}, \bm{\Phi}_t, \bm{\Psi}_t)$ fixed. Finally, we update the dual iterates $\bm{\Phi}, \bm{\Psi}$ by taking a gradient ascent step. The gradients of the augmented Lagrangian \eqref{eq:lagrangian} with respect to $\bm{\Phi}$ and $\bm{\Psi}$ are given by the primal residuals $(\bm{I}_n-\bm{P}_{t+1})\bm{Z}_{t+1}$ and $\bm{Z}_{t+1}-\bm{U}_{t+1}$ respectively. We use $\rho_1$ and $\rho_2$ respectively as the step size. We proceed until the squared norm of each primal residual is below a numerical tolerance parameter $\epsilon$ or until we reach an input maximum number of iterations $T$. We know have the following result:
\begin{proposition}
Assume that the number of compute threads $w$ is less than $\min \{n, m\}$. The per iteration complexity of Algorithm \ref{alg:ADMM} is $O\big{(}k^2n+knd+\frac{k^3(n+m)+k^2nm}{w}\big{)}$.
\end{proposition}

\begin{proof}
    The result follows from the complexity analysis of problems \eqref{opt_lrml4:U_prob}, \eqref{opt_lrml4:V_prob}, \eqref{opt_lrml4:P_prob} and \eqref{opt_lrml4:Z_prob}. 
\end{proof}

Having presented Algorithm \ref{alg:ADMM} in extensive detail, it is natural to consider what types of guarantees can be made on the final output solution $(\bm{U}_T, \bm{V}_T, \bm{P}_T)$. We explore this in the following theorem:

\begin{theorem} \label{thm_lrml:convergence}
    Let $\{(\bm{U}_t, \bm{V}_t, \bm{P}_t, \bm{Z}_t, \bm{\Phi}_t, \bm{\Psi}_t)\}$ denote a sequence generated by Algorithm \ref{alg:ADMM} (assuming we allow Algorithm \ref{alg:ADMM} to iterate indefinitely). Suppose the dual variable sequence $\{(\bm{\Phi}_t, \bm{\Psi}_t)\}$ is bounded and satisfies
    \begin{equation}
        \sum_{t=0}^\infty (\Vert \bm{\Phi}_{t+1} - \bm{\Phi}_t \Vert_F^2 + \Vert \bm{\Psi}_{t+1} - \bm{\Psi}_t \Vert_F^2) < \infty. \label{eq_lrml:thm_cond}
    \end{equation} Let $(\bm{\Bar{U}}, \bm{\Bar{V}}, \bm{\Bar{P}}, \bm{\Bar{Z}}, \bm{\Bar{\Phi}}, \bm{\Bar{\Psi}})$ denote any accumulation point of $\{(\bm{U}_t, \bm{V}_t, \bm{P}_t, \bm{Z}_t, \bm{\Phi}_t, \bm{\Psi}_t)\}$. If the set of $k$ leading eigenvectors of the matrix $[\lambda \bm{Y}\bm{Y}^T+\frac{1}{2}(\bm{\Bar{\Phi}}\bm{\Bar{Z}}^T+\bm{\Bar{Z}}\bm{\Bar{\Phi}}^T)]$ is the same as the set of $k$ leading eigenvectors of the matrix $\bm{\Bar{Z}}\bm{\Bar{Z}}^T$, then $(\bm{\Bar{U}}, \bm{\Bar{V}}, \bm{\Bar{P}}, \bm{\Bar{Z}}, \bm{\Bar{\Phi}}, \bm{\Bar{\Psi}})$ satisfies the first order optimality conditions for \eqref{opt_lrml4:MC_primal_dummy}.

\end{theorem}

\begin{proof}
    We leverage a proof technique similar to the technique used to establish Theorem 2.1 from \cite{xu2012alternating}. Note that from \eqref{eq_lrml:thm_cond}, we immediately have $\bm{\Phi}_{t+1}-\bm{\Phi}_t \rightarrow 0$ and $\bm{\Psi}_{t+1}-\bm{\Psi}_t \rightarrow 0$. We will first show that we also have $\bm{U}_{t+1}-\bm{U}_t \rightarrow 0, \bm{V}_{t+1}-\bm{V}_t \rightarrow 0$ and $\bm{Z}_{t+1}-\bm{Z}_t \rightarrow 0$.

    Let $\mathcal{L}^A(\bm{U}; \bm{V}, \bm{P}, \bm{Z}, \bm{\Phi}, \bm{\Psi})$ denote the augmented Lagrangian \eqref{eq:lagrangian} viewed as a function of $\bm{U}$. Notice that $\mathcal{L}^A(\bm{U}; \bm{V}, \bm{P}, \bm{Z}, \bm{\Phi}, \bm{\Psi}) - \frac{\gamma + \rho_2}{2} \Vert \bm{U} \Vert_F^2$ is a convex function, which implies that $\mathcal{L}^A(\bm{U}; \bm{V}, \bm{P}, \bm{Z}, \bm{\Phi}, \bm{\Psi})$ is a strongly convex function of $\bm{U}$ with parameter $\gamma + \rho_2$. Similarly, $\mathcal{L}^A(\bm{V}; \bm{U}, \bm{P}, \bm{Z}, \bm{\Phi}, \bm{\Psi})$ is a strongly convex function of $\bm{V}$ with parameter $\gamma$ and $\mathcal{L}^A(\bm{Z}; \bm{U}, \bm{V}, \bm{P}, \bm{\Phi}, \bm{\Psi})$ is a strongly convex function of $\bm{Z}$ with parameter $\rho_2$. By strong convexity, we know that for any matrices $\bm{U}, \bm{\Delta U}, \bm{V}, \bm{P}, \bm{Z}, \bm{\Phi}, \bm{\Psi}$ we have \begin{equation}
        \mathcal{L}^A(\bm{U} + \bm{\Delta U}) - \mathcal{L}^A(\bm{U}) \geq \langle \nabla_{\bm{U}} \mathcal{L}^A(\bm{U}), \bm{\Delta U} \rangle + (\gamma + \rho_2)\Vert \bm{\Delta U} \Vert_F^2, \label{eq:strong_convex_U}
    \end{equation} where the shorthand notation $\mathcal{L}^A(\bm{U})$ is understood to denote the augmented Lagrangian \eqref{eq:lagrangian} viewed as a function only of $\bm{U}$ with the other variables held fixed. Letting $\bm{U} = \bm{U}_{t+1}, \bm{\Delta U} = \bm{U}_t - \bm{U}_{t+1}$, and noting that $\langle \nabla_{\bm{U}} \mathcal{L}^A(\bm{U}_{t+1}), \bm{\Delta U} \rangle \geq 0$ since $\bm{U}_{t+1}$ minimizes $\mathcal{L}^A(\bm{U})$ at iteration $t$, from \eqref{eq:strong_convex_U} we have $\mathcal{L}^A(\bm{U}_t) - \mathcal{L}^A(\bm{U}_{t+1}) \geq (\gamma + \rho_2)\Vert \bm{U}_t - \bm{U}_{t+1}\Vert_F^2$. Similarly, we have $\mathcal{L}^A(\bm{V}_t) - \mathcal{L}^A(\bm{V}_{t+1}) \geq \gamma\Vert \bm{V}_t - \bm{V}_{t+1}\Vert_F^2$ and $\mathcal{L}^A(\bm{Z}_t) - \mathcal{L}^A(\bm{Z}_{t+1}) \geq \rho_2\Vert \bm{Z}_t - \bm{Z}_{t+1}\Vert_F^2$. Moreover, since $\bm{P}_{t+1}$ minimizes $\mathcal{L}^A(\bm{P})$ at iteration $t$, we have $\mathcal{L}^A(\bm{P}_t) - \mathcal{L}^A(\bm{P}_{t+1}) \geq 0$. Observe that we can express the difference in the value of the augmented Lagrangian between iteration $t$ and iteration $t+1$ as
    \begin{equation}
        \begin{aligned}
            \mathcal{L}^A(\bm{U}_t, \bm{V}_t, \bm{P}_t, \bm{Z}_t) - \mathcal{L}^A(\bm{U}_{t+1}, &\bm{V}_{t+1}, \bm{P}_{t+1}, \bm{Z}_{t+1}) = \\ &\mathcal{L}^A(\bm{U}_t) - \mathcal{L}^A(\bm{U}_{t+1}) + \mathcal{L}^A(\bm{V}_t) - \mathcal{L}^A(\bm{V}_{t+1}) \\
            &+ \mathcal{L}^A(\bm{P}_t) - \mathcal{L}^A(\bm{P}_{t+1}) +\mathcal{L}^A(\bm{Z}_t) - \mathcal{L}^A(\bm{Z}_{t+1}) \\
            &+ \mathcal{L}^A(\bm{\Phi}_t) - \mathcal{L}^A(\bm{\Phi}_{t+1}) +
            \mathcal{L}^A(\bm{\Psi}_t) - \mathcal{L}^A(\bm{\Psi}_{t+1}).
        \end{aligned} \label{eq:lagrange_decomp}
    \end{equation} Recognizing that we have $\mathcal{L}^A(\bm{\Phi}_t) - \mathcal{L}^A(\bm{\Phi}_{t+1}) = -\rho_1 \Vert \bm{\Phi}_t - \bm{\Phi}_{t+1} \Vert_F^2, \mathcal{L}^A(\bm{\Psi}_t) - \mathcal{L}^A(\bm{\Psi}_{t+1}) = -\rho_2 \Vert \bm{\Psi}_t - \bm{\Psi}_{t+1} \Vert_F^2$, \eqref{eq:lagrange_decomp} implies that
    \begin{equation}
    \begin{aligned}
        \mathcal{L}^A(\bm{U}_t, &\bm{V}_t, \bm{P}_t, \bm{Z}_t) - \mathcal{L}^A(\bm{U}_{t+1}, \bm{V}_{t+1}, \bm{P}_{t+1}, \bm{Z}_{t+1}) \geq \\ &(\gamma + \rho_2)\Vert \bm{U}_t - \bm{U}_{t+1}\Vert_F^2 + \gamma\Vert \bm{V}_t - \bm{V}_{t+1}\Vert_F^2 + \rho_2\Vert \bm{Z}_t - \bm{Z}_{t+1}\Vert_F^2 \\ &-\rho_1 \Vert \bm{\Phi}_t - \bm{\Phi}_{t+1} \Vert_F^2 -\rho_2 \Vert \bm{\Psi}_t - \bm{\Psi}_{t+1} \Vert_F^2.
    \end{aligned} \label{eq:lagrange_ineq}
    \end{equation} We claim that the augmented Lagrangian is bounded from below. To see this, note that $\mathcal{L}^A$ can equivalently be written as
    \begin{equation}
    \begin{aligned}
        \mathcal{L}^A(\bm{U}, \bm{V}, \bm{P}, \bm{Z}, \bm{\Phi}, \bm{\Psi}) &= \sum_{(i, j) \in \Omega}((\bm{U}\bm{V}^T)_{ij}-A_{ij})^2 + \lambda \Vert (\bm{I}_n - \bm{P})\bm{Y}\Vert_F^2 \\ &+ \frac{\gamma}{2}( \Vert \bm{U} \Vert_F^2 + \Vert \bm{V} \Vert_F^2)
        + \mathbb{I}_{\mathcal{P}_k}(\bm{P}) + \frac{\rho_1}{2} \Vert (\bm{I}_n-\bm{P})\bm{Z} + \frac{\bm{\Phi}}{\rho_1}\Vert_F^2 \\ & +\frac{\rho_2}{2} \Vert \bm{Z} - \bm{U} + \frac{\bm{\Psi}}{\rho_2}\Vert_F^2
        -\frac{1}{2\rho_1}\Vert \bm{\Phi} \Vert_F^2 +\frac{1}{2\rho_2}\Vert \bm{\Psi} \Vert_F^2,
    \end{aligned} \label{eq:lagrangian_scaled}
    \end{equation} and recall that by assumption the dual variables $\bm{\Phi}$ and $\bm{\Psi}$ are bounded. Thus, the bounded-ness of $\mathcal{L}^A$ coupled with summing \eqref{eq:lagrange_ineq} over $t$ implies that
    \begin{equation}
        \begin{aligned}
            &\sum_{t=0}^\infty c_1(\Vert \bm{U}_t - \bm{U}_{t+1}\Vert_F^2 + \Vert \bm{V}_t - \bm{V}_{t+1}\Vert_F^2 + \Vert \bm{Z}_t - \bm{Z}_{t+1}\Vert_F^2) \\
            &-\sum_{t=0}^\infty c_2(\Vert \bm{\Phi}_t - \bm{\Phi}_{t+1}\Vert_F^2 + \Vert \bm{\Psi}_t - \bm{\Psi}_{t+1}\Vert_F^2) < \infty,
        \end{aligned} \label{eq:lagrangian_telescope}
    \end{equation} where $c_1 = \min \{\gamma, \rho_2\}$ and $c_2 = \max \{\rho_1, \rho_2\}$. By assumption, the second term of \eqref{eq:lagrangian_telescope} is finite which implies that the first term must also be finite. This immediately implies that $\bm{U}_{t+1}-\bm{U}_t \rightarrow 0, \bm{V}_{t+1}-\bm{V}_t \rightarrow 0$ and $\bm{Z}_{t+1}-\bm{Z}_t \rightarrow 0$ as desired.

    We are now ready to prove the main result of the theorem. The (unaugmented) Lagrangian $\mathcal{L}$ for \eqref{opt_lrml4:MC_primal_dummy} is given by
    \begin{equation}
    \begin{aligned}
        \mathcal{L}(\bm{U}, \bm{V}, \bm{P}, &\bm{Z}, \bm{\Phi}, \bm{\Psi}) = \sum_{(i, j) \in \Omega}((\bm{U}\bm{V}^T)_{ij}-A_{ij})^2 + \lambda \text{Tr}\big{(}\bm{Y}^T(\bm{I}_n - \bm{P})\bm{Y}\big{)} \\ &+ \frac{\gamma}{2}( \Vert \bm{U} \Vert_F^2 + \Vert \bm{V} \Vert_F^2)
        + \mathbb{I}_{\mathcal{P}_k}(\bm{P}) + \text{Tr}(\bm{\Phi}^T(\bm{I}_n-\bm{P})\bm{Z}) + \text{Tr}(\bm{\Psi}^T(\bm{Z}-\bm{U})).
    \end{aligned} \label{eq:lagrangian_normal}
    \end{equation} The corresponding first order optimality conditions can be expressed as
    \begin{subequations}
    \begin{align}
        &[2\bm{V}^T\bm{W}_i\bm{V} + \gamma \bm{I}_k]U_{i, \star} = 2\bm{V}^T\bm{W}_iA_{i, \star} + \Psi_{i, \star} \quad  i \in [n], \label{eq:first_U}\\
        &[2\bm{U}^T\bm{W}_j\bm{U} + \gamma \bm{I}_k]V_{j, \star} = 2\bm{U}^T\bm{W}_jA_{\star, j} \quad  j \in [m], \label{eq:first_V}\\
        &\bm{P} = \bm{M}\bm{M}^T \text{where  } \bm{M}\bm{\Sigma}\bm{M}^T \text{ is a rank $k$ SVD of  } \lambda \bm{Y}\bm{Y}^T+\frac{1}{2}(\bm{\Phi}\bm{Z}^T+\bm{Z}\bm{\Phi}^T),  \label{eq:first_P}\\
        &\bm{\Phi} + \bm{\Psi} = \bm{P}\bm{\Phi}, \label{eq:first_Z}\\
        &\bm{Z} = \bm{P}\bm{Z}, \label{eq:first_Phi}\\
        &\bm{Z} = \bm{U}, \label{eq:first_Psi}
    \end{align}
    \end{subequations} where the diagonal matrices $\bm{W}_i, \bm{W}_j$ are defined as in Propositions \ref{prop_lrml4:U_sol} and \ref{prop_lrml4:V_sol} respectively.
    Let $(\bm{\Bar{U}}, \bm{\Bar{V}}, \bm{\Bar{P}}, \bm{\Bar{Z}}, \bm{\Bar{\Phi}}, \bm{\Bar{\Psi}})$ denote any limit point of $\{(\bm{U}_t, \bm{V}_t, \bm{P}_t, \bm{Z}_t, \bm{\Phi}_t, \bm{\Psi}_t)\}$. Recalling the Algorithm \ref{alg:ADMM} updates for $\bm{\Phi}$ and $\bm{\Psi}$, the conditions $\bm{\Phi}_{t+1}-\bm{\Phi}_t \rightarrow 0$ and $\bm{\Psi}_{t+1}-\bm{\Psi}_t \rightarrow 0$ imply that \eqref{eq:first_Phi} and \eqref{eq:first_Psi} hold at $(\bm{\Bar{U}}, \bm{\Bar{V}}, \bm{\Bar{P}}, \bm{\Bar{Z}}, \bm{\Bar{\Phi}}, \bm{\Bar{\Psi}})$. Moreover, when \eqref{eq:first_Psi} holds the Algorithm \ref{alg:ADMM} update for $\bm{U}$ given by Proposition \ref{prop_lrml4:U_sol} reduces to \eqref{eq:first_U} while the update for $\bm{V}$ given by Proposition \ref{prop_lrml4:V_sol} enforces \eqref{eq:first_V}. From the proof of Propsition \eqref{prop_lrml4:Z_sol}, we know that Algorithm \ref{alg:ADMM} updates $\bm{Z}$ to satisfy the following
    \begin{equation}
        \Big{(}\rho_1(\bm{I}_n-\bm{P})+\rho_2\bm{I}_n\Big{)}\bm{Z}=\Big{(}\rho_2\bm{U}-(\bm{I}_n-\bm{P})\bm{\Phi} - \bm{\Psi}\Big{)}. \label{eq:Z_update_proof}
    \end{equation} Since \eqref{eq:first_Phi} and \eqref{eq:first_Psi} hold, \eqref{eq:Z_update_proof} immediately implies \eqref{eq:first_Z} is satisfied by $(\bm{\Bar{U}}, \bm{\Bar{V}}, \bm{\Bar{P}}, \bm{\Bar{Z}}, \bm{\Bar{\Phi}}, \bm{\Bar{\Psi}})$. It remains to verify that $(\bm{\Bar{U}}, \bm{\Bar{V}}, \bm{\Bar{P}}, \bm{\Bar{Z}}, \bm{\Bar{\Phi}}, \bm{\Bar{\Psi}})$ satisfies \eqref{eq:first_P}. By Proposition \ref{prop_lrml4:P_sol}, we know that we have $\bm{\Bar{P}} = \bm{L} \bm{L}^T$ where $\bm{L}\bm{\Sigma}\bm{L}^T$ is a rank $k$ truncated SVD of the matrix $[\lambda \bm{Y}\bm{Y}^T+\frac{\rho_1}{2}\bm{\Bar{Z}}\bm{\Bar{Z}}^T +\frac{1}{2}(\bm{\Bar{\Phi}}\bm{\Bar{Z}}^T+\bm{\Bar{Z}}\bm{\Bar{\Phi}}^T)]$. If the set of $k$ leading eigenvectors of the matrix $\bm{\Bar{Z}}\bm{\Bar{Z}}^T$ is the same as the set of $k$ leading eigenvectors of $[\lambda \bm{Y}\bm{Y}^T +\frac{1}{2}(\bm{\Bar{\Phi}}\bm{\Bar{Z}}^T+\bm{\Bar{Z}}\bm{\Bar{\Phi}}^T)]$, it follows immediately that $\bm{L} \Sigma' \bm{L}^T$ will be a rank $k$ SVD of $[\lambda \bm{Y}\bm{Y}^T +\frac{1}{2}(\bm{\Bar{\Phi}}\bm{\Bar{Z}}^T+\bm{\Bar{Z}}\bm{\Bar{\Phi}}^T)]$ for some diagonal matrix $\Sigma'$. Thus, in this setting, $(\bm{\Bar{U}}, \bm{\Bar{V}}, \bm{\Bar{P}}, \bm{\Bar{Z}}, \bm{\Bar{\Phi}}, \bm{\Bar{\Psi}})$ satisfies \eqref{eq:first_P}. this completes the proof.
    
\end{proof}
In words, Theorem \ref{thm_lrml:convergence} states that if the sequence of dual variable iterates produced by Algorithm \ref{alg:ADMM} is bounded and the primal residuals converge to zero quickly enough (specifically, it is required that the norm of successive dual variable differences is summable), then any accumulation point $(\bm{\Bar{U}}, \bm{\Bar{V}}, \bm{\Bar{P}}, \bm{\Bar{Z}}, \bm{\Bar{\Phi}}, \bm{\Bar{\Psi}})$ of the sequence of iterates produced by Algorithm \ref{alg:ADMM} satisfies the first order optimality conditions of \eqref{opt_lrml4:MC_primal_dummy} if the rank $k$ approximation of the matrix $[\lambda \bm{Y}\bm{Y}^T+\frac{1}{2}(\bm{\Bar{\Phi}}\bm{\Bar{Z}}^T+\bm{\Bar{Z}}\bm{\Bar{\Phi}}^T)]$ shares the same column space as the matrix $\bm{\Bar{Z}}$. {\color{black}We note that the assumption made in Theorem \ref{thm_lrml:convergence} is consistent with common assumptions made in the analysis of nonconvex ADMM \citep{jiang2014alternating, shen2014augmented, xu2012alternating}.} We {\color{black}also }note that this condition can only be verified upon termination of Algorithm \ref{alg:ADMM} since it depends on the algorithm output in addition to the problem data. Accordingly, Theorem \ref{thm_lrml:convergence} provides an a posteriori convergence result.

We note that this condition was always satisfied by the output of Algorithm \ref{alg:ADMM} in our synthetic numerical experiments. Specifically, letting $\bm{P}_1 \in \mathbb{R}^{n \times n}$ denote the orthogonal projection onto the $k$-dimensional column space of $\bm{\Bar{Z}}$ and $\bm{P}_2 \in \mathbb{R}^{n \times n}$ denote the orthogonal projection onto the $k$ leading eigenvectors of $[\lambda \bm{Y}\bm{Y}^T+\frac{1}{2}(\bm{\Bar{\Phi}}\bm{\Bar{Z}}^T+\bm{\Bar{Z}}\bm{\Bar{\Phi}}^T)]$, the eigenvector condition stated in Theorem \ref{thm_lrml:convergence} is satisfied if and only if we have $(\bm{I}_n-\bm{P_1})\bm{P}_2 = \bm{0}_{n \times n}$. This suggests that the quantity $\bm{P}_2-\bm{P_1}\bm{P}_2$ can be viewed as a dual residual for Algorithm \ref{alg:ADMM} (recall that the primal residuals are given by $\bm{Z}-\bm{P}\bm{Z}$ and $\bm{Z}-\bm{U}$, which we term the $\bm{\Phi}$-residual and the $\bm{\Psi}$-residual respectively). In Figure \ref{fig_lrml4:synthetic_residual}, we plot the evolution of the primal and dual residuals over $500$ iterations of Algorithm \ref{alg:ADMM} for a single synthetic data run (see Section \ref{ssec_lrml4:synthetic_data} for a specification of the data generation procedure) where we have fixed $\rho_1 = \rho_2 = 10$. Observe that after only a small number of iterations, the norm of the $\bm{\Phi}$-residual and the norm of the $\bm{\Psi}$-residual quickly approach $0$, indicating that Algorithm \ref{alg:ADMM} has arrived at a feasible solution. Moreover, the norm of the dual residual similarly quickly approaches $0$ after a small number of iterations, indicating that the eigenvector condition from Theorem \ref{thm_lrml:convergence} is satisfied and that Algorithm \ref{alg:ADMM} has therefore arrived at a solution satisfying first order optimality conditions.

\begin{figure*}[h]\centering
  \includegraphics[width=0.9\textwidth]{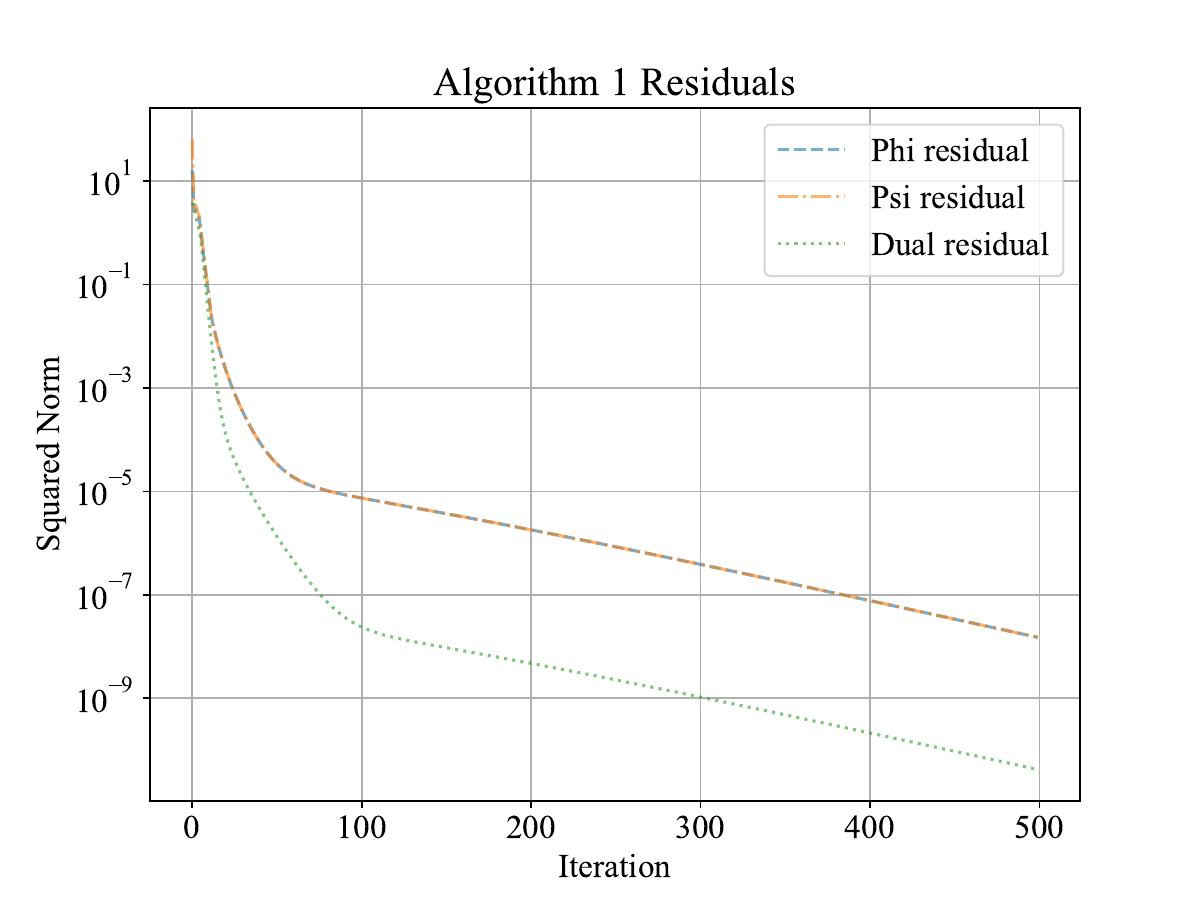}
  \caption{\color{black}Algorithm \ref{alg:ADMM} primal and dual residual evolution versus iteration number for a single synthetic data run with $n=1000, m=100, k=5$ and $d=150$. Note that due to the logarithmic scale, the phi residual and psi residual lines are overlapping.}
  \label{fig_lrml4:synthetic_residual}
\end{figure*}

\section{Computational Results} \label{sec_lrml4:experiments}

We evaluate the performance of Algorithm \ref{alg:ADMM} implemented in Julia 1.7.3. Throughout, we fix $\rho_1 = \rho_2 = 10$, set the maximum number of iterations $T=20$ and set the number of compute threads $w=24$. Note that given the novelty of Problem \eqref{opt_lrml4:MC_primal}, there are no pre-existing specialized methods to benchmark against. Accordingly, we compare the performance of Algorithm \ref{alg:ADMM} against well studied methods for the very closely related MC problem{\color{black}, a method designed for MC in the presence of side information and} a highly performant generic method for low rank matrix optimization problems. The MC methods we consider are Fast-Impute \citep{bertsimas2020fast}, Soft-Impute \citep{JMLR:v11:mazumder10a} and Iterative-SVD \citep{troyanskaya2001missing} which we introduced formally in Section \ref{ssec_lrml4:mc_methods}. We utilize the implementation of Fast-Impute made publicly available by \citep{bertsimas2020fast} while we use the implementation of Soft-Impute and Iterative-SVD from the python package fancyimpute $0.7.0$ \citep{fancyimpute}. {\color{black}We additionally consider the variant of Fast-Impute that incorporates side information (denoted as Fast-Impute-Side) which is also made publicly available by \cite{bertsimas2020fast}.} The matrix optimization method we consider is ScaledGD (scaled gradient descent) \citep{tong2021accelerating} which we introduced formally in Section \ref{sssec_lrml4:scaledGD} and implement ourselves. We perform experiments using both synthetic data and real world data on MIT’s Supercloud Cluster \citep{reuther2018interactive}, which hosts Intel Xeon Platinum 8260 processors. To bridge the gap between theory and practice, we have made our code freely available on \url{GitHub} at \url{https://github.com/NicholasJohnson2020/LearningLowRankMatrices}.

To evaluate the performance of Algorithm \ref{alg:ADMM}, Fast-Impute, {\color{black} Fast-Impute-Side,} Soft-Impute, Iterative-SVD and ScaledGD on synthetic data, we consider the objective value achieved by a returned solution in \eqref{opt_lrml4:MC_primal}, the $\ell_2$ reconstruction error between a returned solution and the ground truth, the coefficient of determination ($R^2$) when the returned solution is used as a predictor for the side information, the numerical rank of a returned solution and the execution time of each algorithm. Explicitly, let $\bm{\Hat{X}} \in \mathbb{R}^{n \times m}$ denote the solution returned by a given method (where we define $\bm{\Hat{X}} = \bm{\Hat{U}}\bm{\Hat{V}}^T$ if the method outputs low rank factors $\bm{\Hat{U}}, \bm{\Hat{V}}$) and let $\bm{A}^{true} \in \mathbb{R}^{n \times m}$ denote the ground truth matrix. We define the the $\ell_2$ reconstruction error of $\bm{\Hat{X}}$ as
\[ERR_{\ell_2}(\bm{\Hat{X}}) = \frac{\Vert \bm{\Hat{X}} - \bm{A}^{true} \Vert_F^2}{\Vert \bm{A}^{true} \Vert_F^2}.\] We compute the numerical rank of $\bm{\Hat{X}}$ by calling the default rank function from the Julia LinearAlgebra package. We aim to answer the following questions:

\begin{enumerate}
    \item How does the performance of Algorithm \ref{alg:ADMM} compare to existing methods such as Fast-Impute, {\color{black}Fast-Impute-Side,} Soft-Impute, Iterative-SVD and ScaledGD on synthetic and real world data?
    \item How is the performance of Algorithm \ref{alg:ADMM} affected by the number of rows $n$, the number of columns $m$, the dimension of the side information $d$ and the underlying rank $k$ of the ground truth?
    \item Empirically, which subproblem solution update is the computational bottleneck of Algorithm \ref{alg:ADMM}?
\end{enumerate}

\subsection{Synthetic Data Generation} \label{ssec_lrml4:synthetic_data}

To generate synthetic data, we specify a number of rows $n \in \mathbb{Z}_+$, a number of columns $m \in \mathbb{Z}_+$, a desired rank $k \in \mathbb{Z}_+$ with $k < \min \{n, m\}$, the dimension of the side information $d \in \mathbb{Z}_+$, a fraction of missing values $\alpha \in (0, 1)$ and a noise parameter $\sigma \in \mathbb{R}_+$ that controls the signal to noise ratio. We sample matrices $\bm{U} \in \mathbb{R}^{n \times k}, \bm{V} \in \mathbb{R}^{m \times k}, \bm{\beta} \in \mathbb{R}^{m \times d}$ by drawing each entry $U_{ij}, V_{ij}, \beta_{ij}$ independently from the uniform distribution on the interval $[0, 1]$. Furthermore, we sample a noise matrix $\bm{N} \in \mathbb{R}^{n \times d}$ by drawing each entry $N_{ij}$ independently from the univariate normal distribution with mean $0$ and variance $\sigma^2$. We let $\bm{A} = \bm{U}\bm{V}^T$ and we let $\bm{Y} = \bm{A}\bm{\beta}+\bm{N}$. Lastly, we sample $\lfloor \alpha \cdot n \cdot m \rfloor$ indices uniformly at random from the collection $\mathcal{I} = \{(i, j): 1 \leq i \leq n, 1 \leq j \leq m\}$ to be the set of missing indices, which we denote by $\Gamma$. The set of revealed entries can then be defined as $\Omega = \mathcal{I} \setminus \Gamma$. We fix $\alpha = 0.9, \sigma =2$ throughout our experiments and report numerical results for various different combinations of $(n, m, d, k)$.

\subsection{Sensitivity to Row Dimension} \label{sssec_lrml4:n_exp}

We present a comparison of Algorithm \ref{alg:ADMM} with ScaledGD, Fast-Impute, {\color{black}Fast-Impute-Side,} Soft-Impute and Iterative-SVD as we vary the number of rows $n$. In these experiments, we fixed $m=100, k=5$, and $d=150$ across all trials. We varied $n \in \{100, 200, 400, 800, 1000, 2000, 5000, 10000\}$ and we performed $20$ trials for each value of $n$. For ScaledGD, we set the step size to be $\eta = \frac{1}{10 \sigma_1(\bm{A})}$ where $\sigma_1(\bm{A})$ denotes the largest singular value of the input matrix $\bm{A}$ where we fill the unobserved entries with the value $0$. Letting $f(\bm{U}_t, \bm{V}_t)$ denote the objective value achieved after iteration $t$ of ScaledGD, we terminate ScaledGD when either $t > 1000$ or $\frac{f(\bm{U}_{t-1}, \bm{V}_{t-1}) - f(\bm{U}_t, \bm{V}_t)}{f(\bm{U}_{t-1}, \bm{V}_{t-1})} < 10^{-3}$. In words, we terminate ScaledGD after $1000$ iterations or after the relative objective value improvement between two iterations is less than $0.1\%$.

\begin{figure*}[h]\centering
  \includegraphics[width=0.9\textwidth]{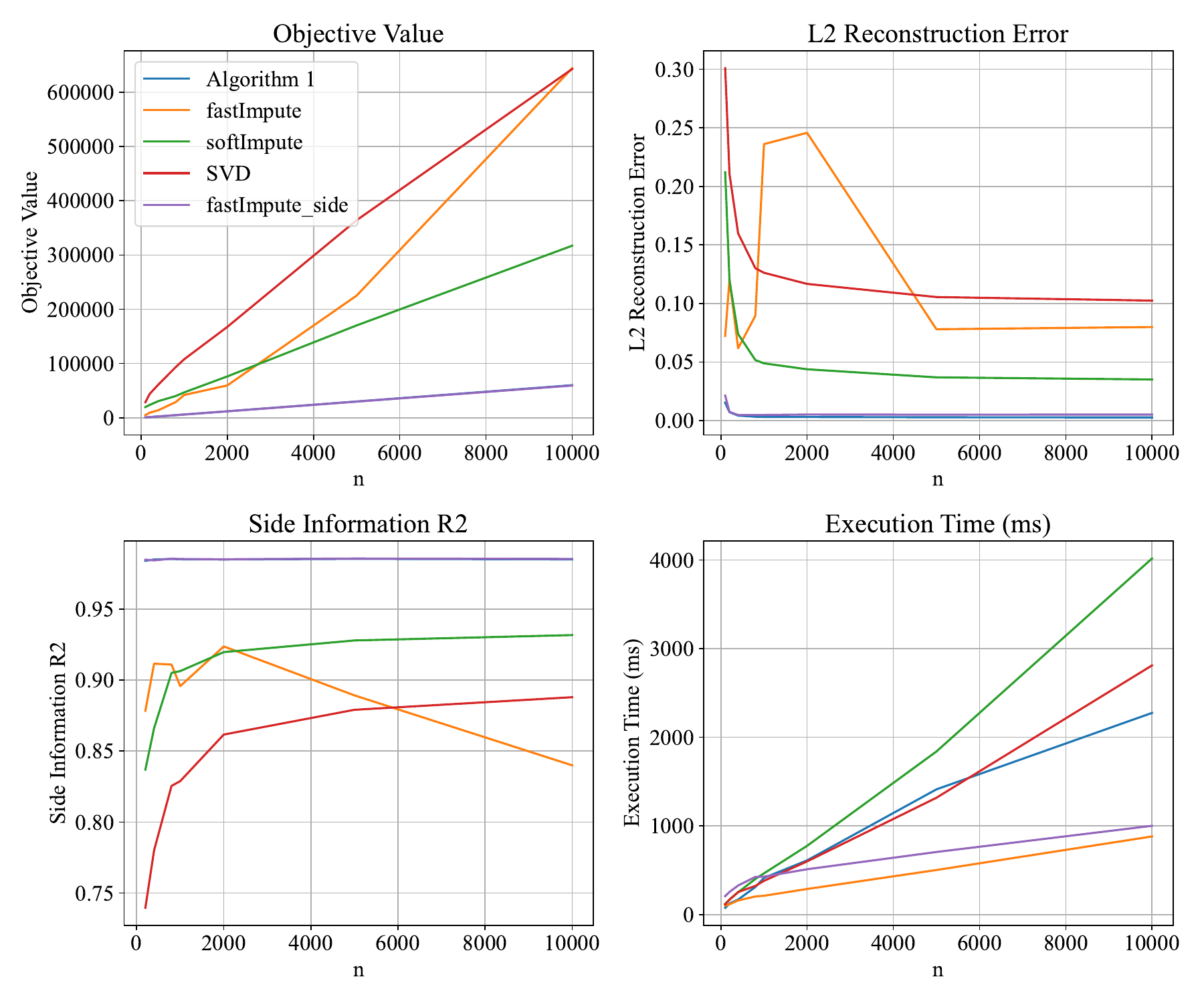}
  \caption{\color{black}Objective value (top left), $\ell_2$ reconstruction error (top right), side information $R^2$ (bottom left) and execution time (bottom right) versus $n$ with $m=100, k=5$ and $d=150$. Averaged over $20$ trials for each parameter configuration.}
  \label{fig_lrml4:synthetic_N}
\end{figure*}

\begin{figure*}[h]\centering
  \includegraphics[width=0.9\textwidth]{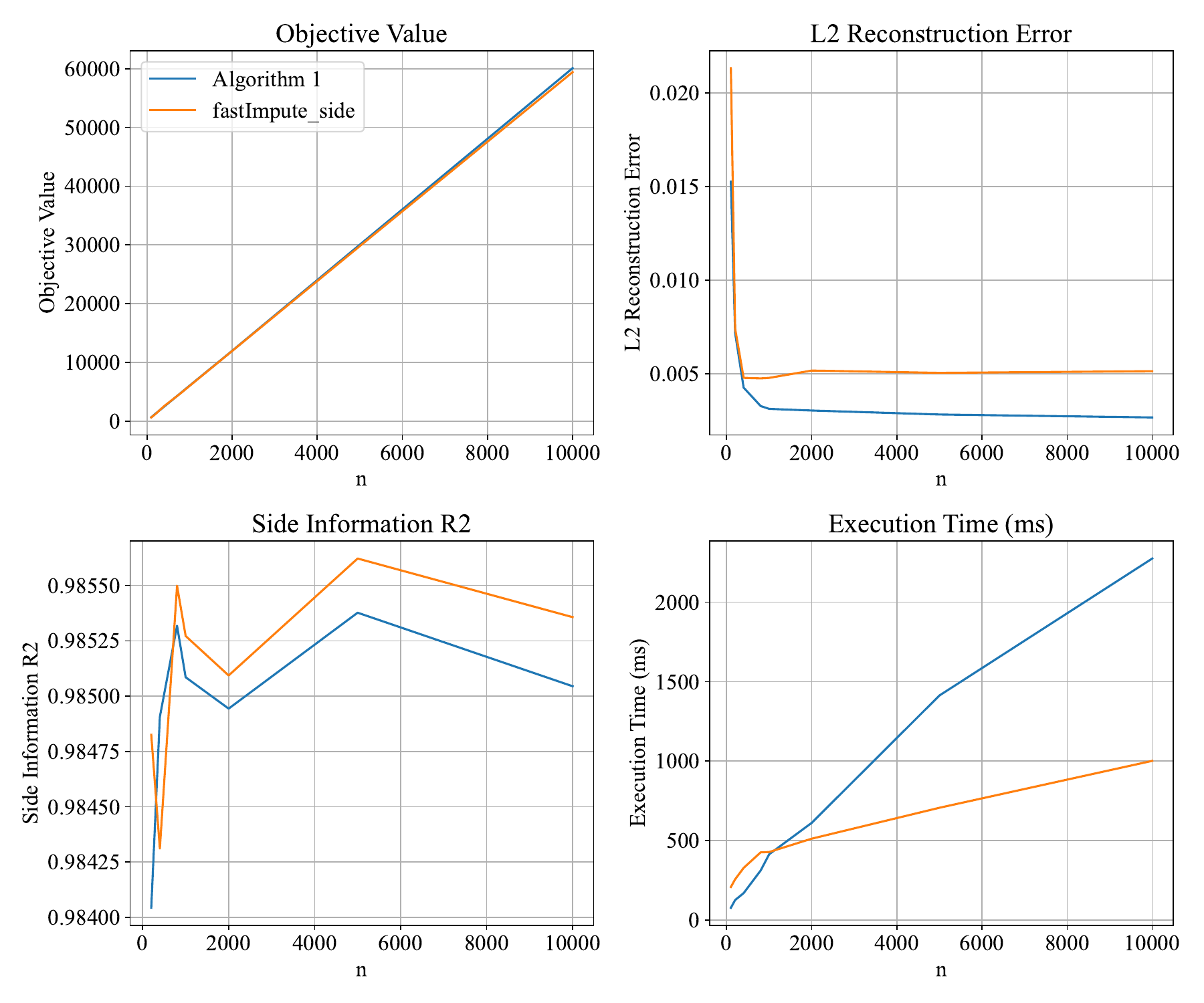}
  \caption{\color{black}Objective value (top left), $\ell_2$ reconstruction error (top right), side information $R^2$ (bottom left) and execution time (bottom right) versus $n$ with $m=100, k=5$ and $d=150$. Averaged over $20$ trials for each parameter configuration.}
  \label{fig_lrml4:synthetic_N_zoom}
\end{figure*}

We report the objective value, $\ell_2$ reconstruction error, side information $R^2$ and execution time for Algorithm \ref{alg:ADMM}, Fast-Impute, {\color{black}Fast-Impute-Side,} Soft-Impute and Iterative-SVD in Figure \ref{fig_lrml4:synthetic_N}. {\color{black}For ease of comparison between Algorithm \ref{alg:ADMM} and Fast-Impute-Side, we plot only the performance of these two methods in Figure \ref{fig_lrml4:synthetic_N_zoom}.} We additionally report the objective value, reconstruction error, side information $R^2$ and execution time for ScaledGD, Algorithm \ref{alg:ADMM}, Fast-Impute, {\color{black}Fast-Impute-Side,} Soft-Impute and Iterative-SVD in Tables \ref{tbl_lrml4:N_objective}, \ref{tbl_lrml4:N_error}, \ref{tbl_lrml4:N_r2} and \ref{tbl_lrml4:N_time} of Appendix \ref{sec_lrml4:app_syn_exp}. In Figure \ref{fig_lrml4:synthetic_N_timing}, we plot the average cumulative time spent solving subproblems \eqref{opt_lrml4:U_prob}, \eqref{opt_lrml4:V_prob}, \eqref{opt_lrml4:P_prob}, \eqref{opt_lrml4:Z_prob} during the execution of Algorithm \ref{alg:ADMM} versus $n$. Our main findings from this set of experiments are:

\begin{enumerate}
    {\color{black}\item Algorithm \ref{alg:ADMM} and Fast-Impute-Side systematically produce higher quality solutions than ScaledGD, Fast-Impute, Soft-Impute and Iterative-SVD (see Table \ref{tbl_lrml4:N_objective}), sometimes achieving an objective value that is an order of magnitude superior than the next best method. We remind the reader that Fast-Impute, Soft-Impute and Iterative-SVD are methods designed for the generic MC problem and are not custom built to solve \eqref{opt_lrml4:MC_primal} so it should not come as a surprise that Algorithm \ref{alg:ADMM} and Fast-Impute-Side (which leverages the side information $\bm{Y}$) significantly outperform these $3$ methods in terms of objective value. ScaledGD however has explicit knowledge of the objective function of \eqref{opt_lrml4:MC_primal} along with its gradient, yet surprisingly produces the weakest average objective value across these experiments. We note that we use the default hyperparameters for ScaledGD recommended by the authors of this method \citep{tong2021accelerating}. We observe that the objective value achieved by all methods increases linearly as the number of rows $n$ increases. There is no significant difference between the objective value achieved by the output of Algorithm $\ref{alg:ADMM}$ compared to that of Fast-Impute-Side (Algorithm $\ref{alg:ADMM}$ is on average $1\%$ higher).}

    \item {\color{black}In terms of $\ell_2$ reconstruction error, Algorithm \ref{alg:ADMM} systematically produces solutions that are of higher quality than ScaledGD, Fast-Impute, Fast-Impute-Side, Soft-Impute and Iterative-SVD (see Table \ref{tbl_lrml4:N_error}). On average, Algorithm \ref{alg:ADMM} outputs a solution whose $\ell_2$ reconstruction error is $30\%$ lesser than the reconstruction error achieved by the best performing alternative method (Fast-Impute-Side). This is especially noteworthy since Algorithm \ref{alg:ADMM} is not designed explicitly with reconstruction error minimization as the objective, unlike Fast-Impute, Fast-Impute-Side and Soft-Impute, and suggests that modeling the side information $\bm{Y}$ as a linear function of $\bm{X}$ is instrumental in recovering high quality low rank estimates of the partially observed data matrix.}
    
    \item {\color{black}With the exception of the experiments for which $n=100$, Algorithm \ref{alg:ADMM} and Fast-Impute-Side always produced solutions that achieved a superior $R^2$ value when used as a predictor for the side information compared to Fast-Impute, Soft-Impute, Iterative-SVD and ScaledGD. There is no significant difference between $R^2$ value achieved by the output of Algorithm $\ref{alg:ADMM}$ compared to that of Fast-Impute-Side.}
    
    \item {\color{black}The runtime of Algorithm \ref{alg:ADMM} is competitive with that of the other methods. The runtime of Algorithm \ref{alg:ADMM} is less than of Soft-Impute and Iterative-SVD but greater than that of Fast-Impute. The runtime for Algorithm $\ref{alg:ADMM}$ was less than that of Fast-Impute-Side for $n < 2000$ but greater than for $n \geq 2000$. For experiments with $n \leq 2000$, Table \ref{tbl_lrml4:N_time} illustrates that ScaledGD was the method with the fastest execution time (however as previously mentioned the returned solutions were of low quality). The runtime of Algorithm \ref{alg:ADMM}, Fast-Impute, Fast-Impute-Side, Soft-Impute and Iterative-SVD appear to grow linearly with $n$.}
    
    \item Figure \ref{fig_lrml4:synthetic_N_timing} illustrates that the computation of the solution for \eqref{opt_lrml4:U_prob} is the computational bottleneck in the execution of Algorithm \ref{alg:ADMM} in this set of experiments, followed next by the computation of the solution for \eqref{opt_lrml4:P_prob}. Empirically, we observe that the solution time of \eqref{opt_lrml4:U_prob}, \eqref{opt_lrml4:V_prob}, \eqref{opt_lrml4:P_prob} and \eqref{opt_lrml4:Z_prob} appear to scale linearly with the number of rows $n$. This observation is consistent with the computational complexities derived for each subproblem of Algorithm \ref{alg:ADMM} in Section \ref{sec_lrml4:admm}.
    
\end{enumerate}

\begin{figure*}[h]\centering
  \includegraphics[width=0.9\textwidth]{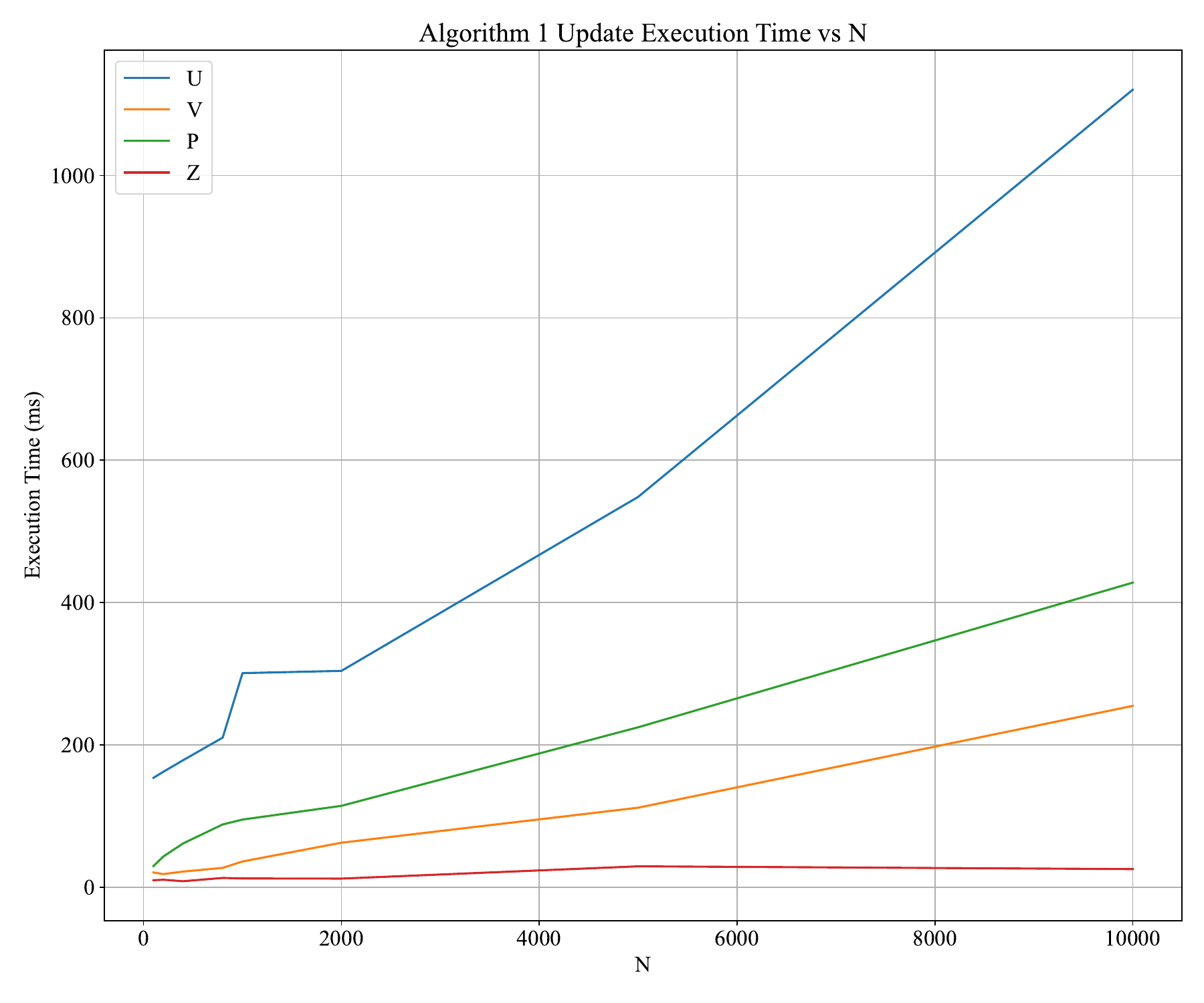}
  \caption{Cumulative time spent solving each subproblem of Algorithm \ref{alg:ADMM} versus $n$ with $m=100, k=5$ and $d=150$. Averaged over $20$ trials for each parameter configuration.}
  \label{fig_lrml4:synthetic_N_timing}
\end{figure*}

\subsection{Sensitivity to Column Dimension} \label{sssec_lrml4:m_exp}

Here, we present a comparison of Algorithm \ref{alg:ADMM} with ScaledGD, Fast-Impute, {\color{black}Fast-Impute-Side,} Soft-Impute and Iterative-SVD as we vary the number of columns $m$. We fixed $n=1000, k=5$, and $d=150$ across all trials. We varied $m \in \{100, 200, 400, 800, 1000, 2000, 5000, 10000\}$ and we performed $20$ trials for each value of $m$.

\begin{figure*}[h]\centering
  \includegraphics[width=0.9\textwidth]{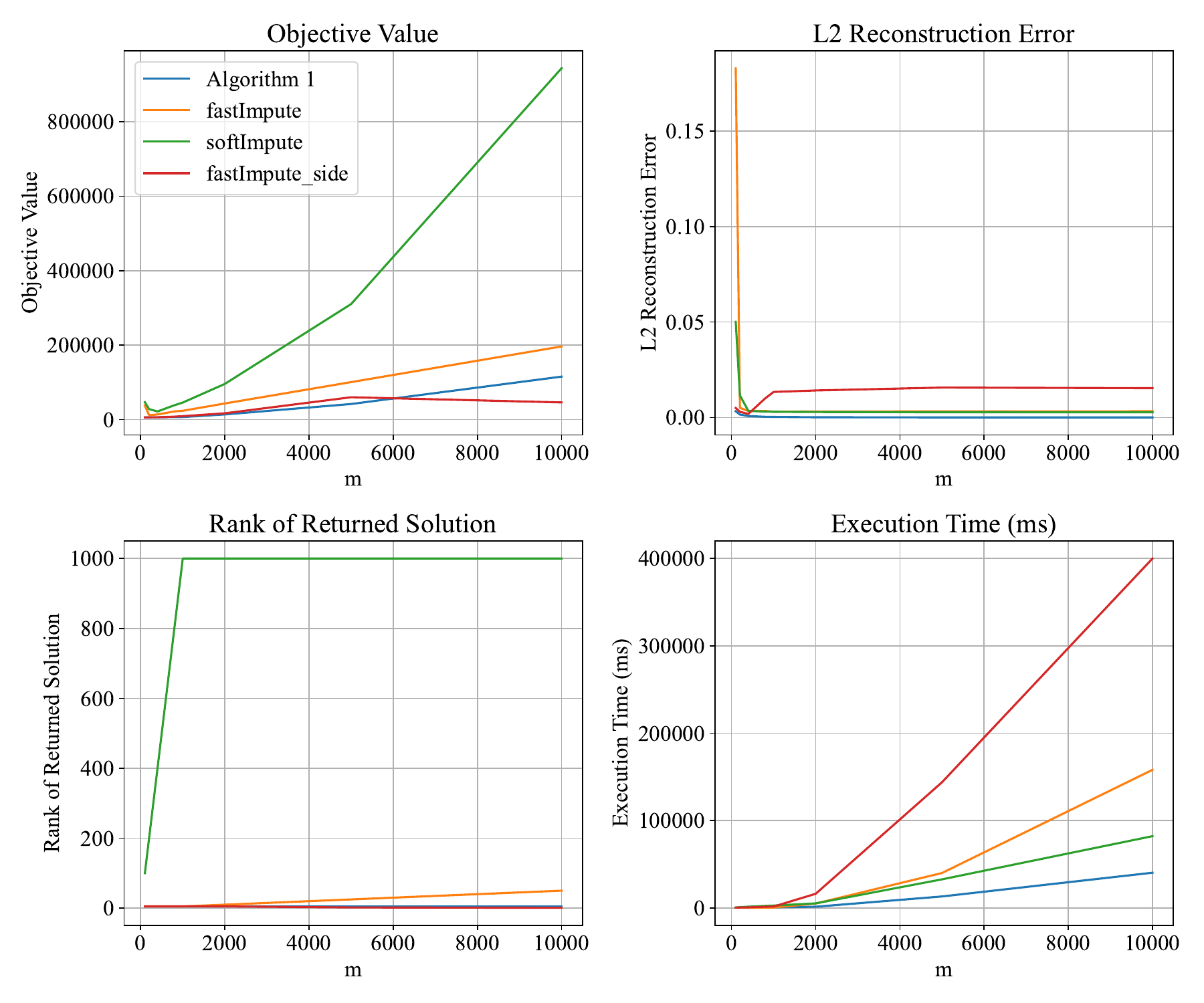}
  \caption{\color{black}Objective value (top left), $\ell_2$ reconstruction error (top right), fitted rank (bottom left) and execution time (bottom right) versus $m$ with $n=1000, k=5$ and $d=150$. Averaged over $20$ trials for each parameter configuration.}
  \label{fig_lrml4:synthetic_M}
\end{figure*}

\begin{figure*}[h]\centering
  \includegraphics[width=0.9\textwidth]{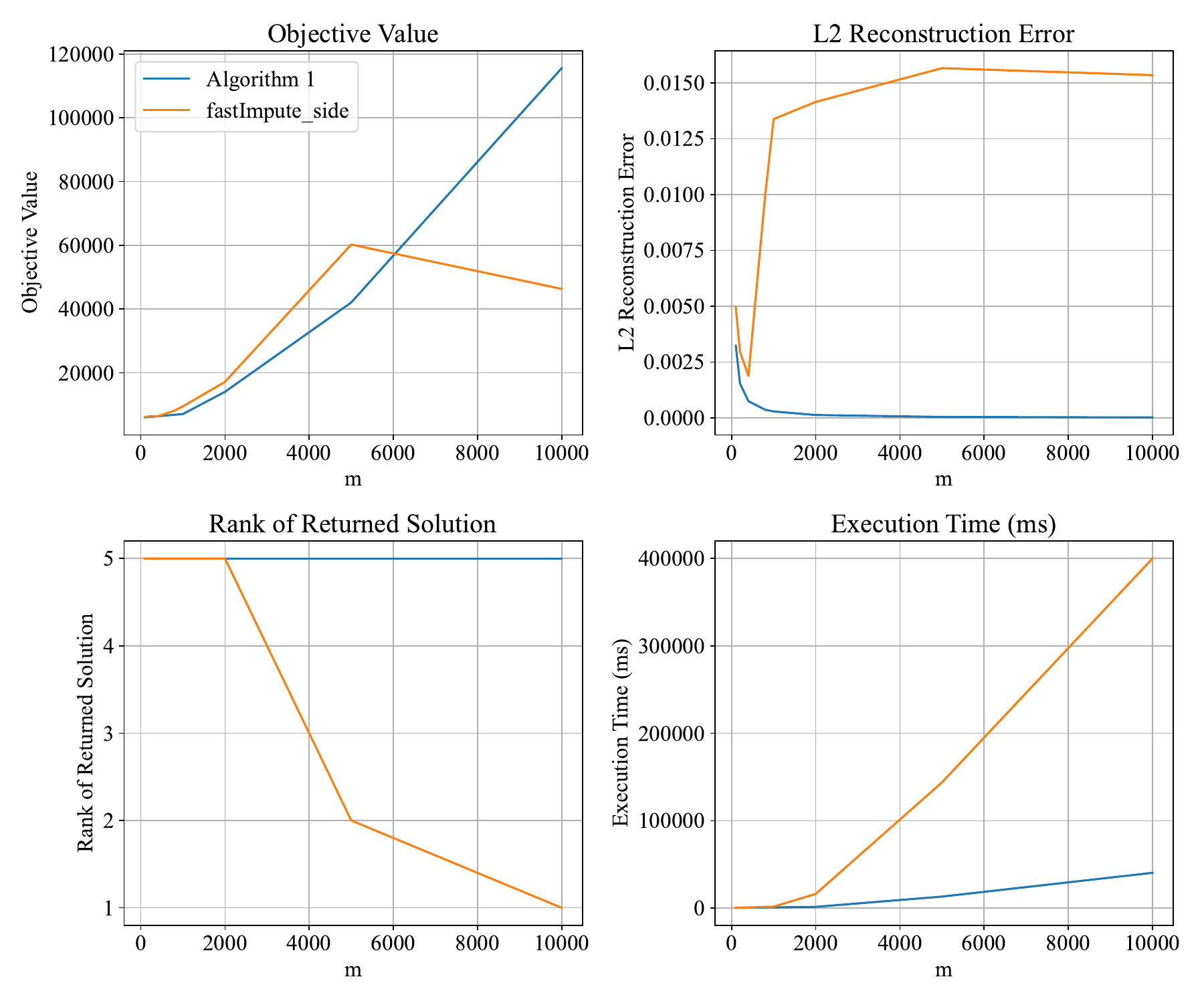}
  \caption{\color{black}Objective value (top left), $\ell_2$ reconstruction error (top right), fitted rank (bottom left) and execution time (bottom right) versus $m$ with $n=1000, k=5$ and $d=150$. Averaged over $20$ trials for each parameter configuration.}
  \label{fig_lrml4:synthetic_M_zoom}
\end{figure*}

We report the objective value, $\ell_2$ reconstruction error, fitted rank and execution time for Algorithm \ref{alg:ADMM}, Fast-Impute, {\color{black}Fast-Impute-Side} and Soft-Impute in Figure \ref{fig_lrml4:synthetic_M}. {\color{black}For ease of comparison between Algorithm \ref{alg:ADMM} and Fast-Impute-Side, we plot only the performance of these two methods in Figure \ref{fig_lrml4:synthetic_M_zoom}.} We additionally report the objective value, reconstruction error and execution time for ScaledGD, Algorithm \ref{alg:ADMM}, Fast-Impute, {\color{black}Fast-Impute-Side,} Soft-Impute and Iterative-SVD in Tables \ref{tbl_lrml4:M_objective}, \ref{tbl_lrml4:M_error} and \ref{tbl_lrml4:M_time} of Appendix \ref{sec_lrml4:app_syn_exp}. In Figure \ref{fig_lrml4:synthetic_M_timing}, we plot the average cumulative time spent solving subproblems \eqref{opt_lrml4:U_prob}, \eqref{opt_lrml4:V_prob}, \eqref{opt_lrml4:P_prob}, \eqref{opt_lrml4:Z_prob} during the execution of Algorithm \ref{alg:ADMM} versus $m$. Our main findings from this set of experiments are a follows:

\begin{enumerate}
     {\color{black}\item Here again, Algorithm \ref{alg:ADMM} and Fast-Impute-Side systematically produce higher quality solutions than ScaledGD, Fast-Impute, Soft-Impute and Iterative-SVD (see Table \ref{tbl_lrml4:M_objective}). For trials with $m \leq 5000$, Algorithm \ref{alg:ADMM} outputs a solution whose objective value is on average $12\%$ lesser than the objective value achieved by the best performing alternative method (Fast-Impute-Side). Fast-Impute-Side produced solutions with lower objective values for $m=10000$. Here again, ScaledGD produces the weakest average objective value across these experiments.}

    \item {\color{black}In terms of $\ell_2$ reconstruction error, Algorithm \ref{alg:ADMM} again systematically produces solutions that are of higher quality than ScaledGD, Fast-Impute, Fast-Impute-Side, Soft-Impute and Iterative-SVD (see Table \ref{tbl_lrml4:M_error}), often achieving an error that is an order of magnitude superior than the next best method. On average, Algorithm \ref{alg:ADMM} outputs a solution whose $\ell_2$ reconstruction error is $77\%$ lesser than the reconstruction error achieved by the best performing alternative method (either Fast-Impute-Side or Soft-Impute).}
    
    \item We observe that the fitted rank of the solutions returned by Algorithm \ref{alg:ADMM}, ScaledGD and Fast-Impute always matched the specified target rank as would be expected, but surprisingly the solutions returned by Soft-Impute and Iterative-SVD were always of full rank despite the fact that these methods were provided with the target rank explicitly. This is potentially due to a numerical issues in the computation of the rank due to presence of extremely small singular values.
    
    \item The runtime of Algorithm \ref{alg:ADMM} exhibits the most favorable scaling behavior among the methods tested in these experiments. For instances with $m \geq 2000$, Table \ref{tbl_lrml4:M_time} shows that Algorithm \ref{alg:ADMM} had the fastest runtime. For instances with $m < 2000$, ScaledGD had the fastest execution time but produced low quality solutions. The runtime of all methods tested grow super-linearly with $m$.
    
    \item Figure \ref{fig_lrml4:synthetic_M_timing} illustrates that the computation of the solution for \eqref{opt_lrml4:U_prob} and \eqref{opt_lrml4:V_prob} are the computational bottlenecks in the execution of Algorithm \ref{alg:ADMM} in this set of experiments while the computation of the solution for \eqref{opt_lrml4:Z_prob} and \eqref{opt_lrml4:P_prob} appear to be a constant function of $m$. This observation is consistent with the complexity analysis performed for each subproblem of Algorithm \ref{alg:ADMM} in Section \ref{sec_lrml4:admm}. Indeed, this analysis indicated that solve times for \eqref{opt_lrml4:Z_prob} and \eqref{opt_lrml4:P_prob} are independent of $m$ while the solve times for \eqref{opt_lrml4:U_prob} and \eqref{opt_lrml4:V_prob} scale linearly with $m$ when the number of threads $w$ satisfies $w < m$. 
\end{enumerate}

\begin{figure*}[h]\centering
  \includegraphics[width=0.9\textwidth]{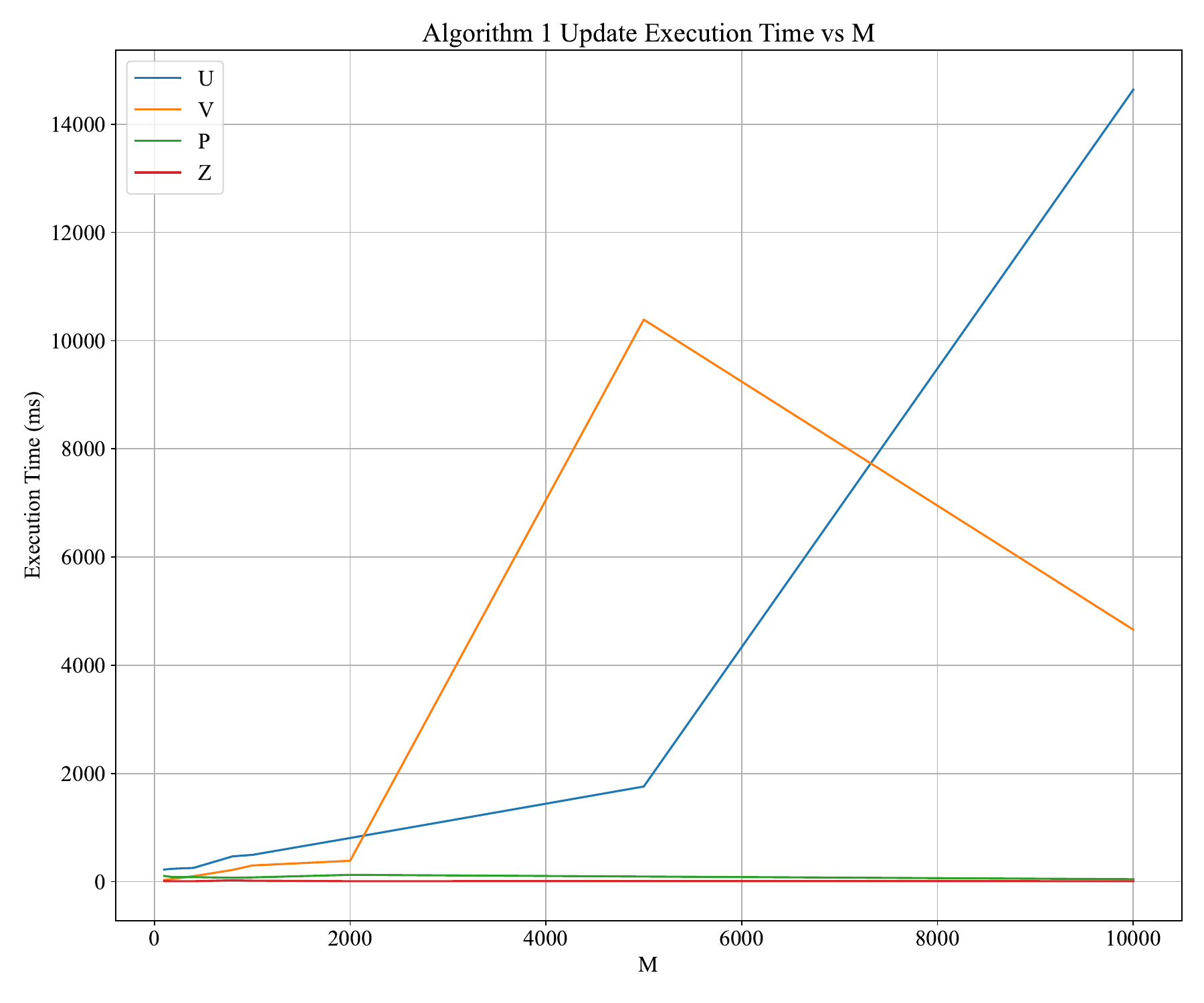}
  \caption{Cumulative time spent solving each subproblem of Algorithm \ref{alg:ADMM} versus $m$ with $n=1000, k=5$ and $d=150$. Averaged over $20$ trials for each parameter configuration.}
  \label{fig_lrml4:synthetic_M_timing}
\end{figure*}

\subsection{Sensitivity to Side Information Dimension} \label{sssec_lrml4:d_exp}

We present a comparison of Algorithm \ref{alg:ADMM} with ScaledGD, Fast-Impute, {\color{black}Fast-Impute-Side,} Soft-Impute and Iterative-SVD as we vary the dimension of the side information $d$. In these experiments, we fixed $n=1000, m = 100$ and $k=5$ across all trials. We considered values of $d$ in the collection $\{10, 50, 100, 150, 200, 250, 500, 1000\}$ and we performed $20$ trials for each value of $d$.

\begin{figure*}[h]\centering
  \includegraphics[width=0.9\textwidth]{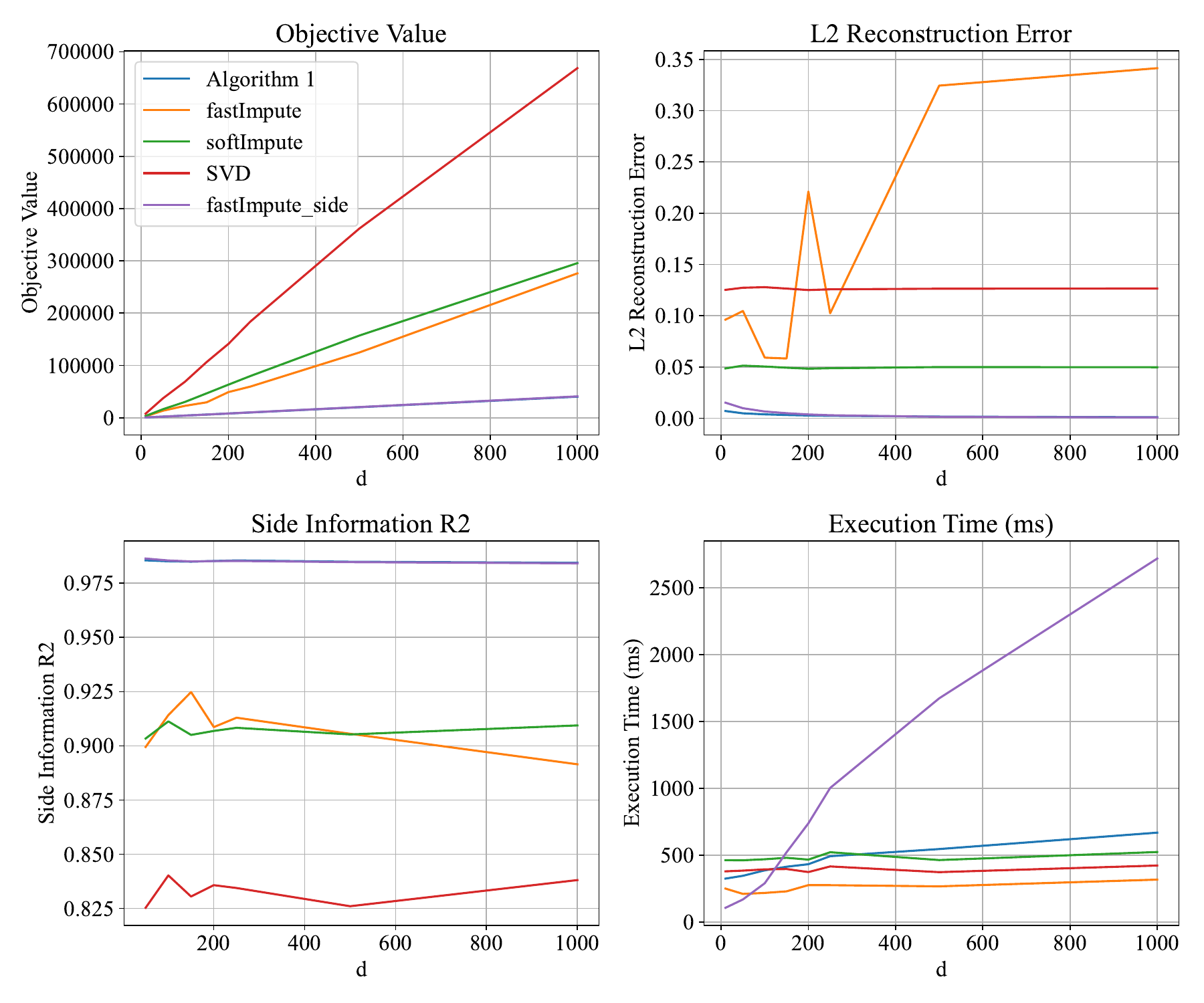}
  \caption{\color{black}Objective value (top left), $\ell_2$ reconstruction error (top right), side information $R^2$ (bottom left) and execution time (bottom right) versus $d$ with $n=1000, m=100$ and $k=5$. Averaged over $20$ trials for each parameter configuration.}
  \label{fig_lrml4:synthetic_D}
\end{figure*}

\begin{figure*}[h]\centering
  \includegraphics[width=0.9\textwidth]{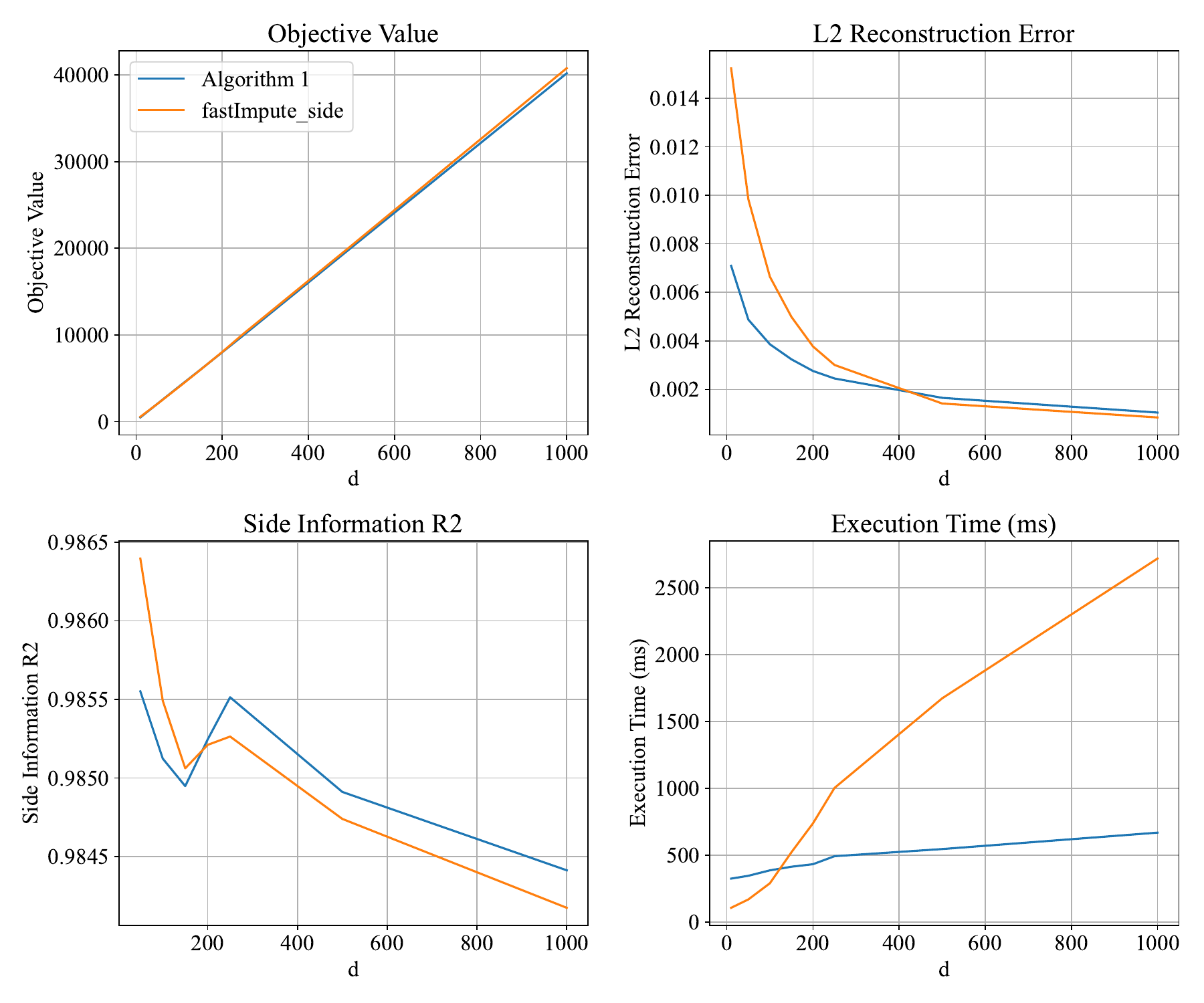}
  \caption{\color{black}Objective value (top left), $\ell_2$ reconstruction error (top right), side information $R^2$ (bottom left) and execution time (bottom right) versus $d$ with $n=1000, m=100$ and $k=5$. Averaged over $20$ trials for each parameter configuration.}
  \label{fig_lrml4:synthetic_D_zoom}
\end{figure*}

We report the objective value, $\ell_2$ reconstruction error, side information $R^2$ and execution time for Algorithm \ref{alg:ADMM}, Fast-Impute, {\color{black}Fast-Impute-Side,} Soft-Impute and Iterative-SVD in Figure \ref{fig_lrml4:synthetic_D}. {\color{black}For ease of comparison between Algorithm \ref{alg:ADMM} and Fast-Impute-Side, we plot only the performance of these two methods in Figure \ref{fig_lrml4:synthetic_D_zoom}.} We additionally report the objective value, reconstruction error, side information $R^2$ and execution time for ScaledGD, Algorithm \ref{alg:ADMM}, Fast-Impute, {\color{black}Fast-Impute-Side,} Soft-Impute and Iterative-SVD in Tables \ref{tbl_lrml4:D_objective}, \ref{tbl_lrml4:D_error}, \ref{tbl_lrml4:D_r2} and \ref{tbl_lrml4:D_time} of Appendix \ref{sec_lrml4:app_syn_exp}. In Figure \ref{fig_lrml4:synthetic_D_timing}, we plot the average cumulative time spent solving subproblems \eqref{opt_lrml4:U_prob}, \eqref{opt_lrml4:V_prob}, \eqref{opt_lrml4:P_prob}, \eqref{opt_lrml4:Z_prob} during the execution of Algorithm \ref{alg:ADMM} versus $d$. Our main findings from this set of experiments are:

\begin{enumerate}
    \item { \color{black}Just as in Sections \ref{sssec_lrml4:n_exp} and \ref{sssec_lrml4:m_exp}, Algorithm \ref{alg:ADMM} and Fast-Impute-Side systematically produce higher quality solutions than ScaledGD, Fast-Impute, Soft-Impute and Iterative-SVD (see Table \ref{tbl_lrml4:D_objective}). ScaledGD produces the weakest average objective value across these experiments. The objective value achieved by each method appears to increase linearly as the dimension $d$ of the side information increases. There is no significant difference between the objective value achieved by the output of Algorithm $\ref{alg:ADMM}$ compared to that of Fast-Impute-Side (Algorithm $\ref{alg:ADMM}$ is on average $2\%$ lower).}

    \item {\color{black}In terms of $\ell_2$ reconstruction error, Algorithm \ref{alg:ADMM} and Fast-Impute-Side produce solutions that are of higher quality than ScaledGD, Fast-Impute, Soft-Impute and Iterative-SVD (see Table \ref{tbl_lrml4:D_error}), often achieving an error that is an order of magnitude superior than the other methods. Relative to Fast-Impute-Side, Algorithm \ref{alg:ADMM} outputs a solution whose $\ell_2$ reconstruction error is on average $23\%$ lesser across these trials. The performance of Algorithm \ref{alg:ADMM} and Fast-Impute-Side improves as $d$ increases, consistent with the intuition that recovering the partially observed matrix $\bm{A}$ becomes easier as more side information becomes available.}

    \item {\color{black}Algorithm \ref{alg:ADMM} and Fast-Impute-Side always produced solutions that achieved a strictly greater $R^2$ value when used as a predictor for the side information compared to the other methods. The $R^2$ achieved by each method is roughly constant as the value of $d$ increases. There is no significant difference between the $R^2$ achieved by the output of Algorithm $\ref{alg:ADMM}$ compared to that of Fast-Impute-Side.}
    
    \item {\color{black}The runtime of Algorithm \ref{alg:ADMM} is competitive with that of the other methods. The runtime of Algorithm \ref{alg:ADMM} is less than of Soft-Impute and Iterative-SVD but greater than that of Fast-Impute. Table \ref{tbl_lrml4:D_time} illustrates that ScaledGD was the fastest performing method, however its solutions were of the lowest quality. The runtimes of Algorithm \ref{alg:ADMM}, Fast-Impute-Side and ScaledGD grow with $d$ while Fast-Impute, Soft-Impute and iterate SVD are constant with $d$ which should be expected as these methods do not act on the side information matrix $\bm{Y}$. The runtime of Fast-Impute-Side exhibits particularly poor scaling behavior relative to the other methods as $d$ increases.}
    
    \item Figure \ref{fig_lrml4:synthetic_D_timing} illustrates that the computation of the solution for \eqref{opt_lrml4:P_prob} is the computational bottleneck in the execution of Algorithm \ref{alg:ADMM} in this set of experiments, followed next by the computation of the solution to \eqref{opt_lrml4:U_prob}. The solution times for \eqref{opt_lrml4:U_prob}, \eqref{opt_lrml4:V_prob} and \eqref{opt_lrml4:Z_prob} appear constant as a function of $d$. This is consistent with the complexity analysis from Section \ref{sec_lrml4:admm} which found that the solve time for \eqref{opt_lrml4:P_prob} is linear in $d$ while the solve time for the $3$ other subproblems are independent of $d$.
\end{enumerate}

\begin{figure*}[h]\centering
  \includegraphics[width=0.9\textwidth]{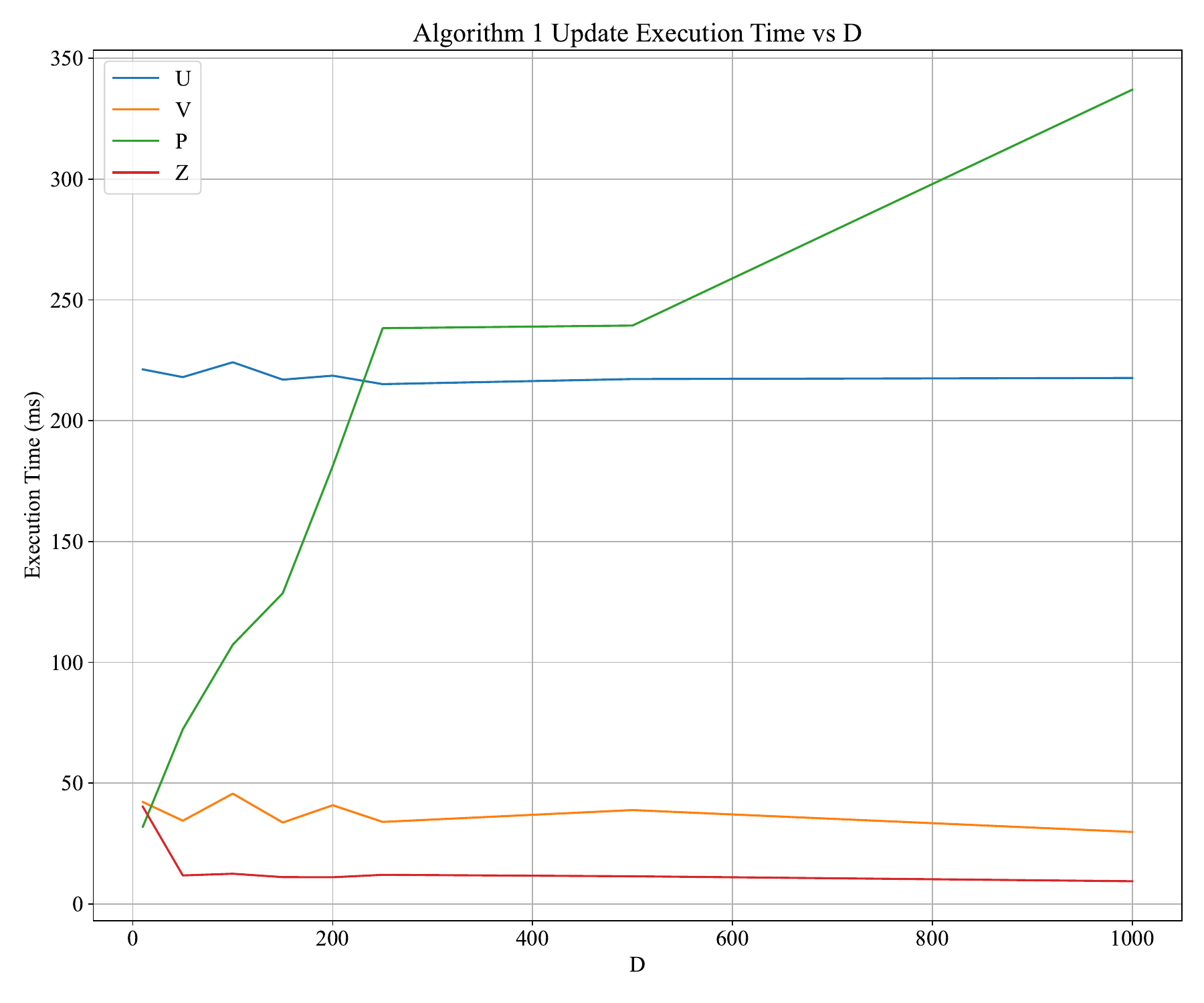}
  \caption{Cumulative time spent solving each subproblem of Algorithm \ref{alg:ADMM} versus $d$ with $n=1000, m=100$ and $k=5$. Averaged over $20$ trials for each parameter configuration.}
  \label{fig_lrml4:synthetic_D_timing}
\end{figure*}

\subsection{Sensitivity to Target Rank} \label{sssec_lrml4:k_exp}

We present a comparison of Algorithm \ref{alg:ADMM} with ScaledGD, Fast-Impute, {\color{black}Fast-Impute-Side,} Soft-Impute and Iterative-SVD as we vary the rank of the underlying matrix $k$. In these experiments, we fixed $n=1000, m = 100$ and $d=150$ across all trials. We varied $k \in \{5, 10, 15, 20, 25, 30, 35, 40\}$ and we performed $20$ trials for each value of $d$.

\begin{figure*}[h]\centering
  \includegraphics[width=0.9\textwidth]{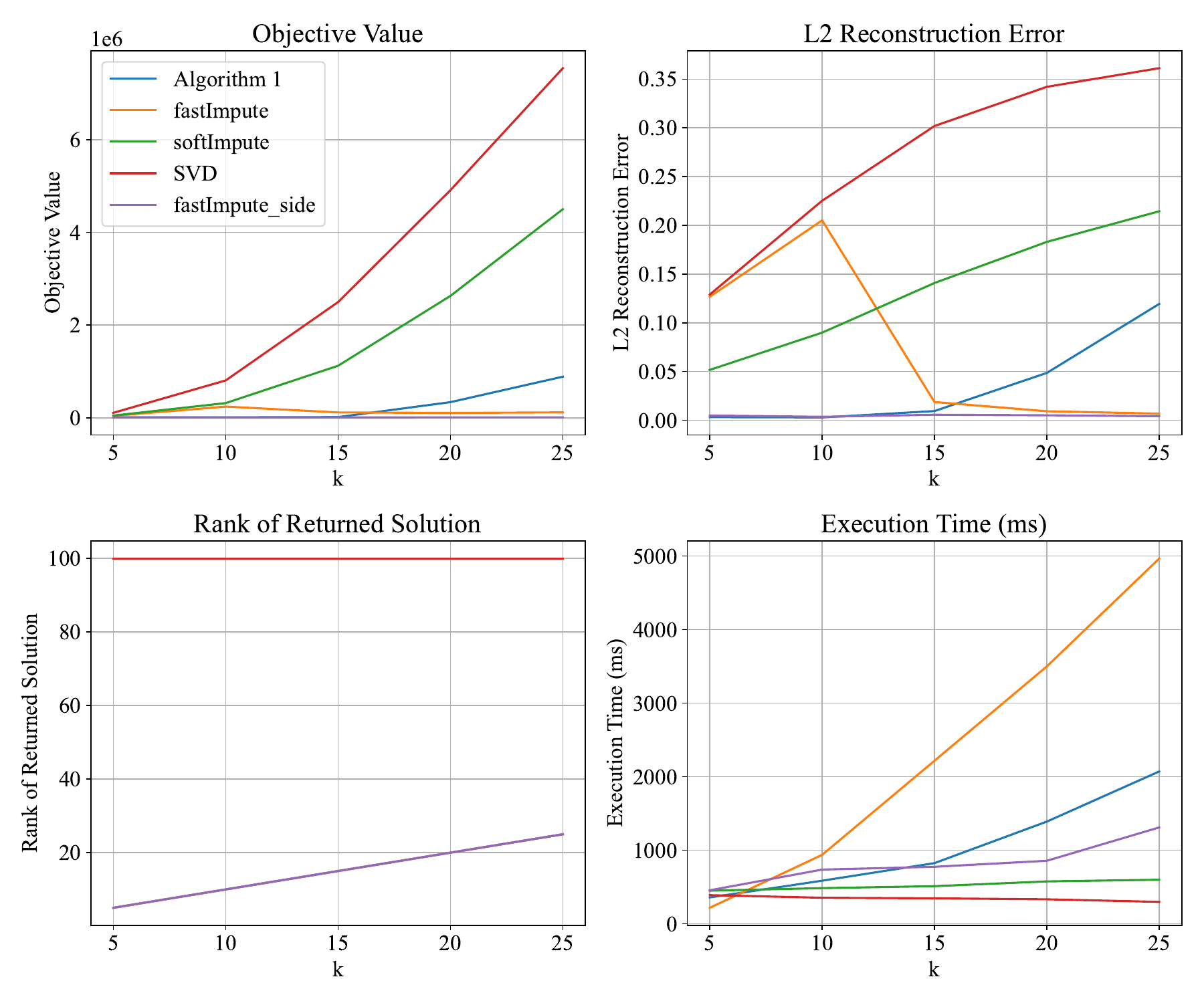}
  \caption{\color{black}Objective value (top left), $\ell_2$ reconstruction error (top right), fitted rank (bottom left) and execution time (bottom right) versus $k$ with $n=1000, m=100$ and $d=150$. Averaged over $20$ trials for each parameter configuration.}
  \label{fig_lrml4:synthetic_K}
\end{figure*}

\begin{figure*}[h]\centering
  \includegraphics[width=0.9\textwidth]{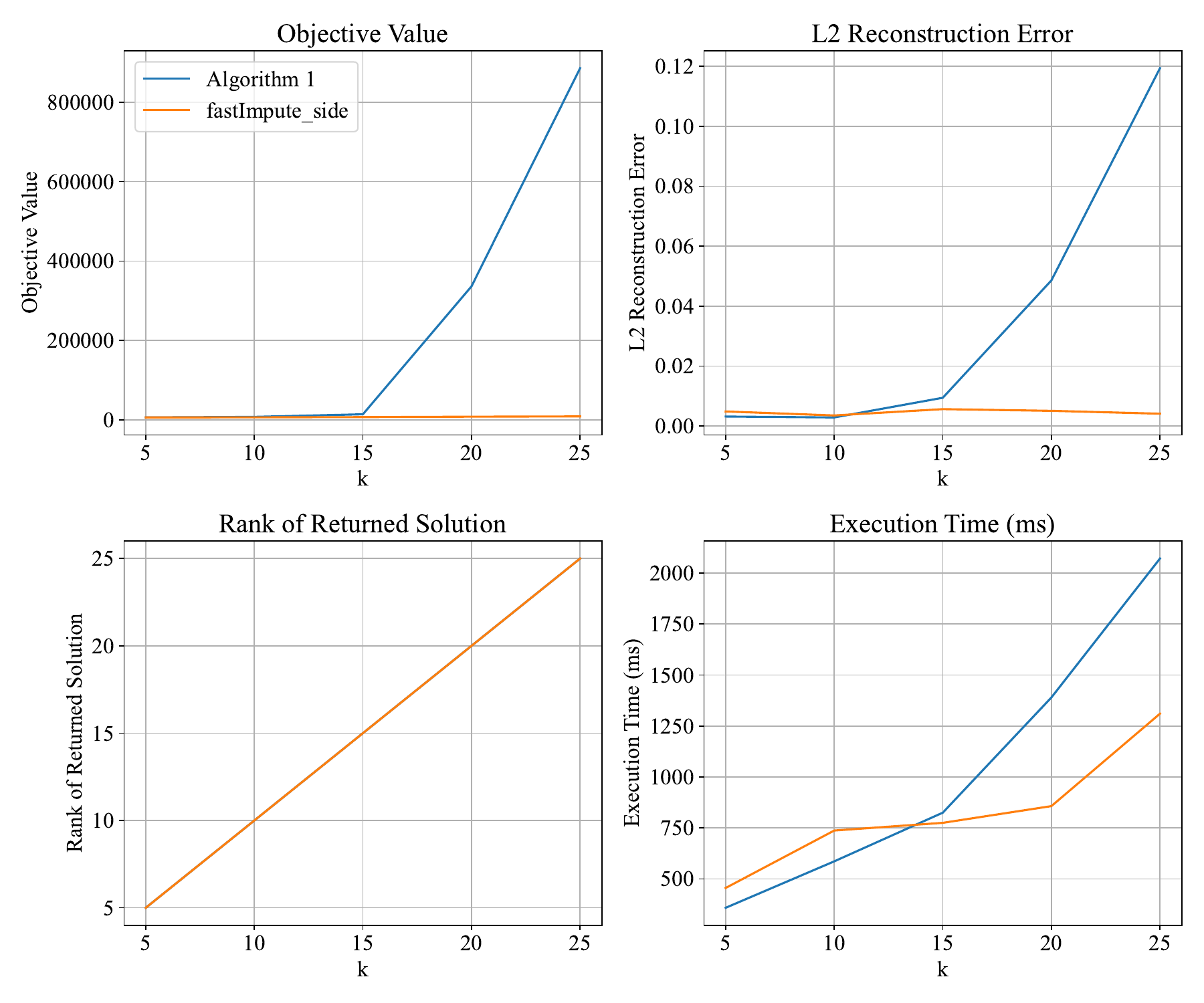}
  \caption{\color{black}Objective value (top left), $\ell_2$ reconstruction error (top right), fitted rank (bottom left) and execution time (bottom right) versus $k$ with $n=1000, m=100$ and $d=150$. Averaged over $20$ trials for each parameter configuration.}
  \label{fig_lrml4:synthetic_K_zoom}
\end{figure*}

We report the objective value, $\ell_2$ reconstruction error, fitted rank and execution time for Algorithm \ref{alg:ADMM}, Fast-Impute, {\color{black}Fast-Impute-Side,} Soft-Impute and Iterative-SVD in Figure \ref{fig_lrml4:synthetic_K}. {\color{black}For ease of comparison between Algorithm \ref{alg:ADMM} and Fast-Impute-Side, we plot only the performance of these two methods in Figure \ref{fig_lrml4:synthetic_K_zoom}.} We additionally report the objective value, reconstruction error and execution time for ScaledGD, Algorithm \ref{alg:ADMM}, Fast-Impute, {\color{black}Fast-Impute-Side} Soft-Impute and Iterative-SVD in Tables \ref{tbl_lrml4:K_objective}, \ref{tbl_lrml4:K_error} and \ref{tbl_lrml4:K_time} of Appendix \ref{sec_lrml4:app_syn_exp}. In Figure \ref{fig_lrml4:synthetic_K_timing}, we plot the average cumulative time spent solving subproblems \eqref{opt_lrml4:U_prob}, \eqref{opt_lrml4:V_prob}, \eqref{opt_lrml4:P_prob}, \eqref{opt_lrml4:Z_prob} during the execution of Algorithm \ref{alg:ADMM} versus $k$. Our main findings from this set of experiments are as follows:

\begin{enumerate}
    \item {\color{black}Unlike in Sections \ref{sssec_lrml4:n_exp}, \ref{sssec_lrml4:m_exp} and \ref{sssec_lrml4:d_exp}, Algorithm \ref{alg:ADMM} underperformed Fast-Impute-Side in terms of objective value (see Table \ref{tbl_lrml4:K_objective}). Fast-Impute-Side was the best performing method in $7$ configurations and Soft-Impute was best in the remaining configuration. ScaledGD produces the weakest average objective value across these experiments.}

    \item {\color{black}In terms of $\ell_2$ reconstruction error, Algorithm \ref{alg:ADMM} produced higher quality solutions than all benchmark methods in $2$ out of $8$ of the tested parameter configurations where $k \leq 10$ (see Table \ref{tbl_lrml4:K_error}). Fast-Impute-Side produced solutions achieving the lowest error in the other $6$ parameter configurations.}

    \item {\color{black}The fitted rank of the solutions returned by Algorithm \ref{alg:ADMM}, ScaledGD, Fast-Impute and Fast-Impute-Side always matched the specified target rank, but the solutions returned by Soft-Impute and Iterative-SVD were always of full rank despite the fact that these methods were provided with the target rank explicitly.}
    
    \item {\color{black}The runtime of Algorithm \ref{alg:ADMM} is competitive with that of the other methods. Table \ref{tbl_lrml4:K_time} illustrates that ScaledGD was the fastest performing method, however its solutions were of the lowest quality. The runtime of Algorithm \ref{alg:ADMM} is most competitive with Fast-Impute-Side, Soft-Impute and Iterative-SVD for small values of $k$.}
    
    \item Figure \ref{fig_lrml4:synthetic_K_timing} illustrates that the computation of the solution for \eqref{opt_lrml4:U_prob} is the computational bottleneck in the execution of Algorithm \ref{alg:ADMM} in this set of experiments, followed next by the computation of the solution to \eqref{opt_lrml4:V_prob} and \eqref{opt_lrml4:Z_prob}.
\end{enumerate}

\begin{figure*}[h]\centering
  \includegraphics[width=0.9\textwidth]{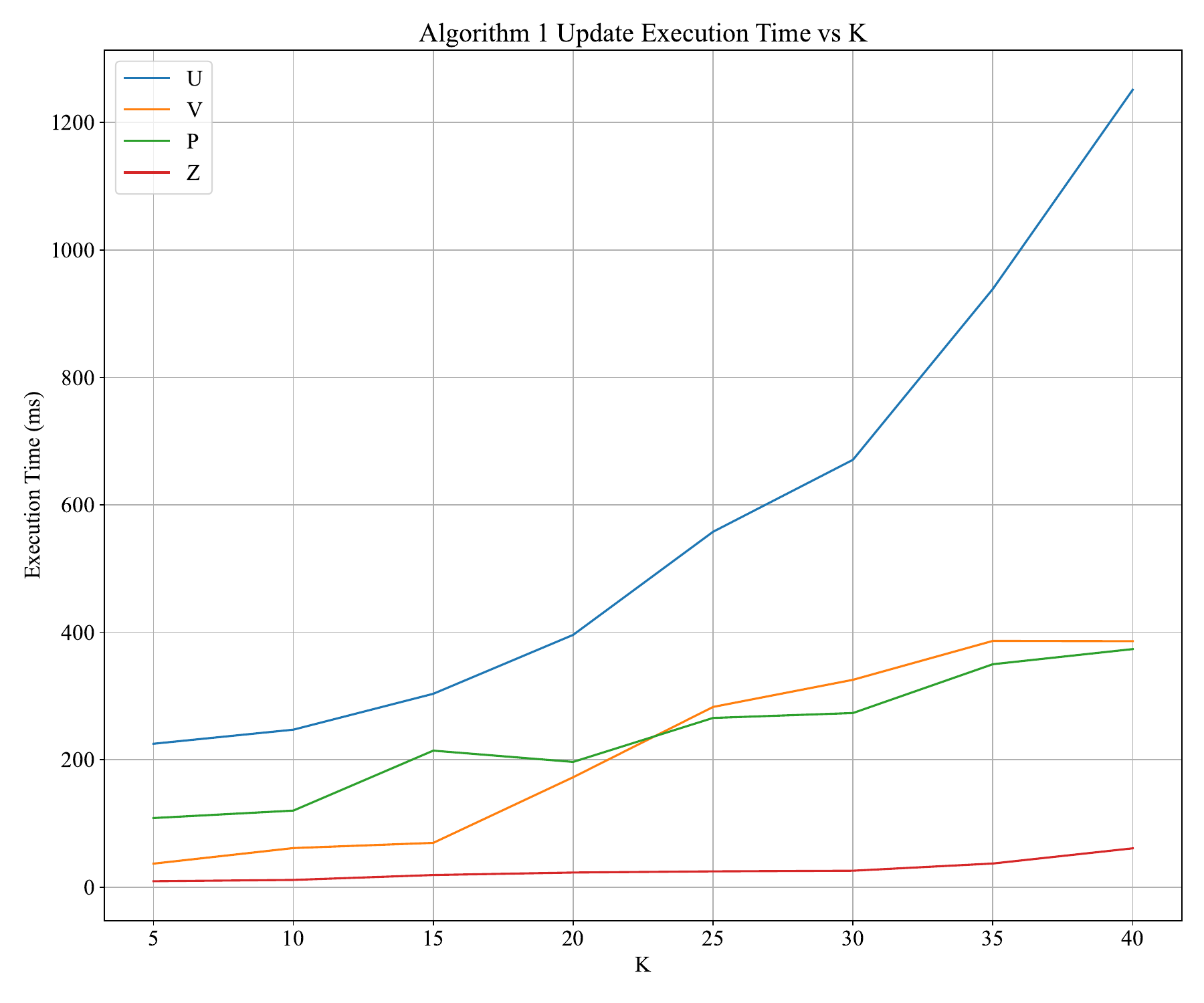}
  \caption{Cumulative time spent solving each subproblem of Algorithm \ref{alg:ADMM} versus $k$ with $n=1000, m=100$ and $d=150$. Averaged over $20$ trials for each parameter configuration.}
  \label{fig_lrml4:synthetic_K_timing}
\end{figure*}

\subsection{Real World Data Experiments} \label{ssec_lrml4:real_experiments}

We seek to answer the following question: how does the performance of Algorithm \ref{alg:ADMM} compare to Fast-Impute {\color{black}and Fast-Impute-Side} on real world data? We consider the Netflix Prize Dataset augmented with features from the TMDB Database {\color{black}as side information}.

The Netflix Prize Dataset consists of greater than $10$ million user ratings of movies spread across more than $450000$ users and $17000$ movies. To prepare data for our experiment, we first pull the following numerical features from the TMDB database:
\begin{enumerate}
    \item Total Budget;
    \item Revenue;
    \item Popularity;
    \item Average Vote;
    \item Vote Count;
    \item Total Runtime.
\end{enumerate} {\color{black}Here, the matrix $\bm{X} \in \mathbb{R}^{n \times m}$ from \eqref{opt_lrml4:MC_primal} corresponds to the matrix of movie-user ratings, $\Omega \subseteq [n] \times [m]$ corresponds to indices of $\bm{X}$ for which the user ratings of movies are observed and $\bm{Y} \in \mathbb{R}^{n \times d}$ corresponds to the TMDB Database features.} Note that many movies did not have all $6$ features available from TMDB. We constructed two datasets for our experimentation. In Dataset $1$, we restricted the dataset to movies that had all $6$ features present {\color{black}($d=6$)} and to users who had given at least $5$ ratings across those movies. After performing this filtering, we were left with $n = 3430$ movies and $m = 467364$ users. In Dataset $2$, we considered the $4$ most frequent features (popularity, average vote, vote count, total runtime) and restricted the dataset to movies that had all $4$ of these features present {\color{black}($d=4$)} and to users who had given at least $5$ ratings to any of these movies. After performing this filtering, we were left with $n = 10574$ movies and $m = 470706$ users. For each dataset, we conducted experiments for values of the target rank $k$ in the set $k \in \{3, 4, 5, 6, 7, 8, 9, 10\}$. For each value of $k$, we conducted $5$ trials where a given trial consisted of randomly withholding $20\%$ of the data as test data, estimating a low rank matrix on the $80\%$ training data and evaluating the out of sample $\ell_2$ reconstruction error on the withheld data.

We report the in sample $\ell_2$ reconstruction error, out of sample $\ell_2$ reconstruction error and execution time for Algorithm \ref{alg:ADMM}{\color{black}, Fast-Impute and Fast-Impute-Side} in addition to the average cumulative time spent solving subproblems \eqref{opt_lrml4:U_prob}, \eqref{opt_lrml4:V_prob}, \eqref{opt_lrml4:P_prob}, \eqref{opt_lrml4:Z_prob} during the execution of Algorithm \ref{alg:ADMM} versus $k$ on Dataset $1$ and Dataset $2$ in Figures \ref{fig_lrml4:netflix_6Y} and \ref{fig_lrml4:netflix_4Y} respectively. We additionally report the in sample $\ell_2$ reconstruction error, out of sample $\ell_2$ reconstruction error and execution time for Algorithm \ref{alg:ADMM}{\color{black}, Fast-Impute and Fast-Impute-Side} on Dataset $1$ and Dataset $2$ in Tables \ref{tbl_lrml4:netflix_6Y} and \ref{tbl_lrml4:netflix_4Y} of Appendix \ref{sec_lrml4:app_syn_exp} respectively. We report only results for Fast-Impute {\color{black}and Fast-Impute-Side} as benchmark{\color{black}s} because Soft-Impute, Iterative-SVD and ScaledGD failed to terminate after a $20$ hour time limit across all experiments involving Dataset $1$ and Dataset $2$. Fast-Impute failed to terminate after a $20$ hour time limit across all experiments involving Dataset $2$ and across experiments involving Dataset $1$ for which the target rank was greater than $6$. Note that Fast-Impute {\color{black}and Fast-Impute-Side were} the best performing benchmark method{\color{black}s} across the synthetic data experiments so {\color{black}they} consist of a reasonable method to compare against. In Figure \ref{fig_lrml4:netflix_r2}, we report the coefficient of determination ($R^2$) achieved by Algorithm \ref{alg:ADMM} on the side information both overall and on individual features in Dataset $1$ and Dataset $2$. Our main findings from this set of experiments are as follows:

\begin{figure*}[h]\centering
  \includegraphics[width=0.9\textwidth]{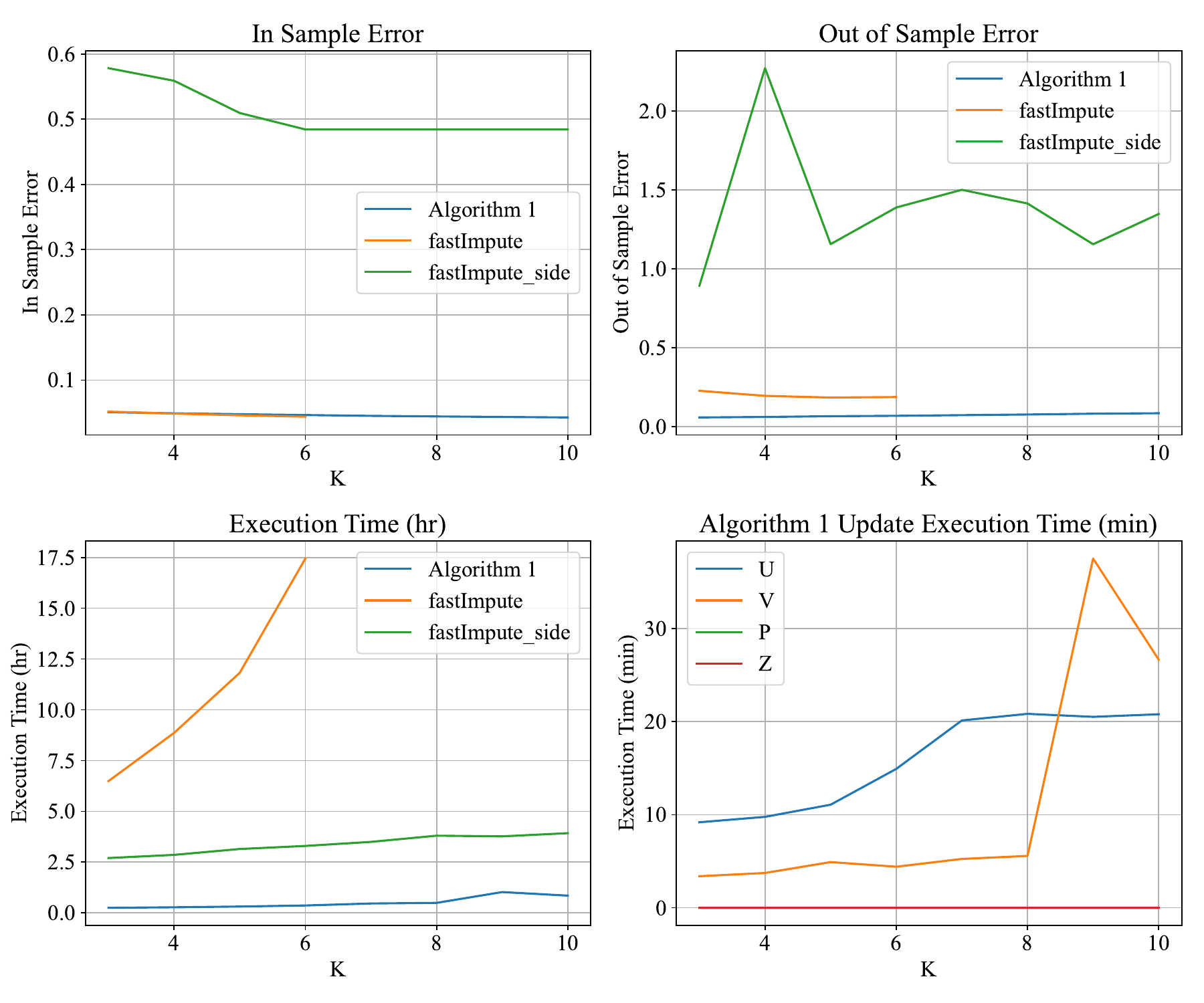}
  \caption{\color{black}In sample $\ell_2$ reconstruction error (top left), out of sample $\ell_2$ reconstruction error (top right), execution time (bottom left) and subproblem execution time (bottom right) versus $k$ on Netflix Prize Dataset $1$. Averaged over $5$ trials.}
  \label{fig_lrml4:netflix_6Y}
\end{figure*}

\begin{figure*}[h]\centering
  \includegraphics[width=0.9\textwidth]{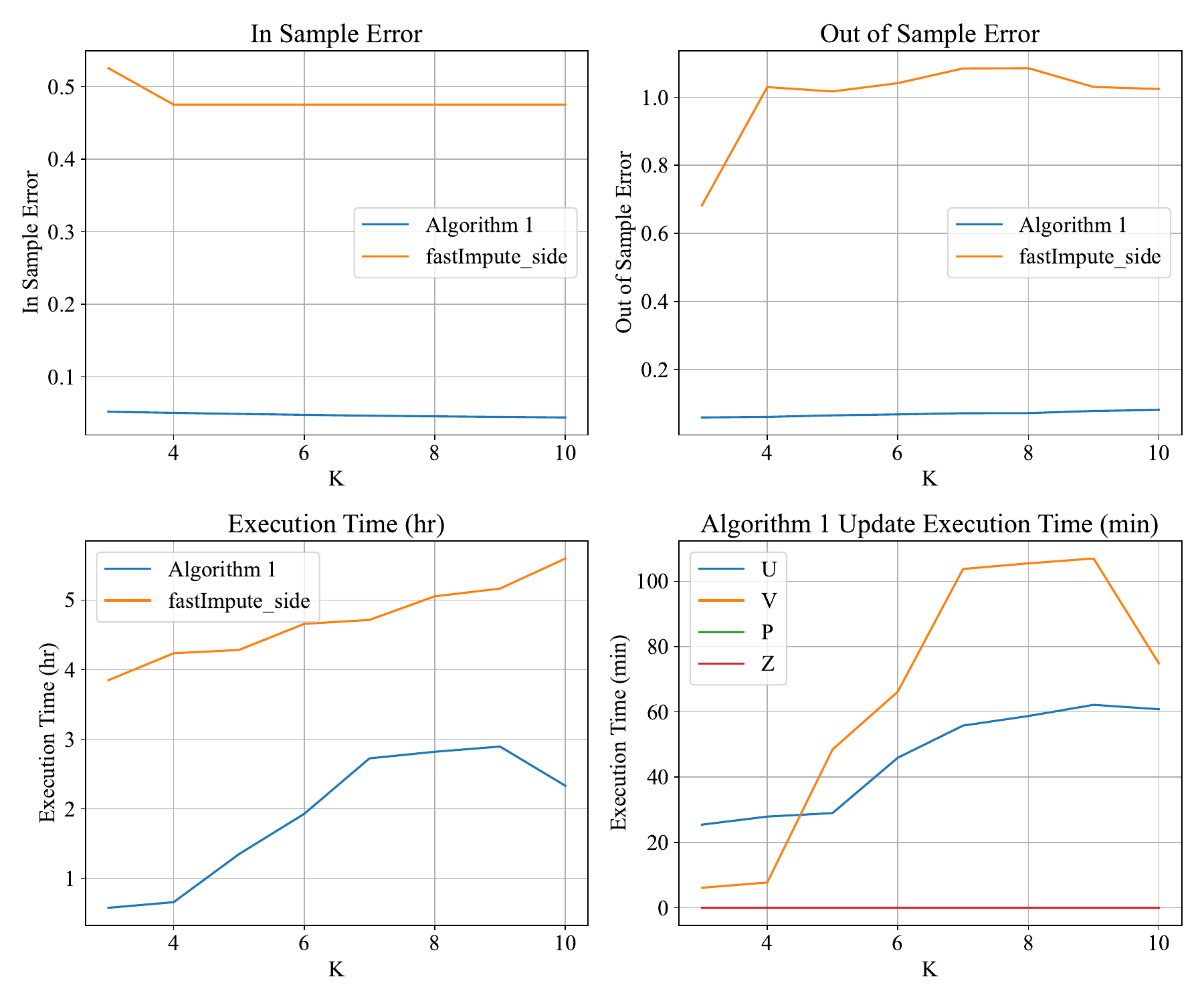}
  \caption{\color{black}In sample $\ell_2$ reconstruction error (top left), out of sample $\ell_2$ reconstruction error (top right), execution time (bottom left) and subproblem execution time (bottom right) versus $k$ on Netflix Prize Dataset $2$. Averaged over $5$ trials.}
  \label{fig_lrml4:netflix_4Y}
\end{figure*}

\begin{figure*}[h]\centering
  \includegraphics[width=0.9\textwidth]{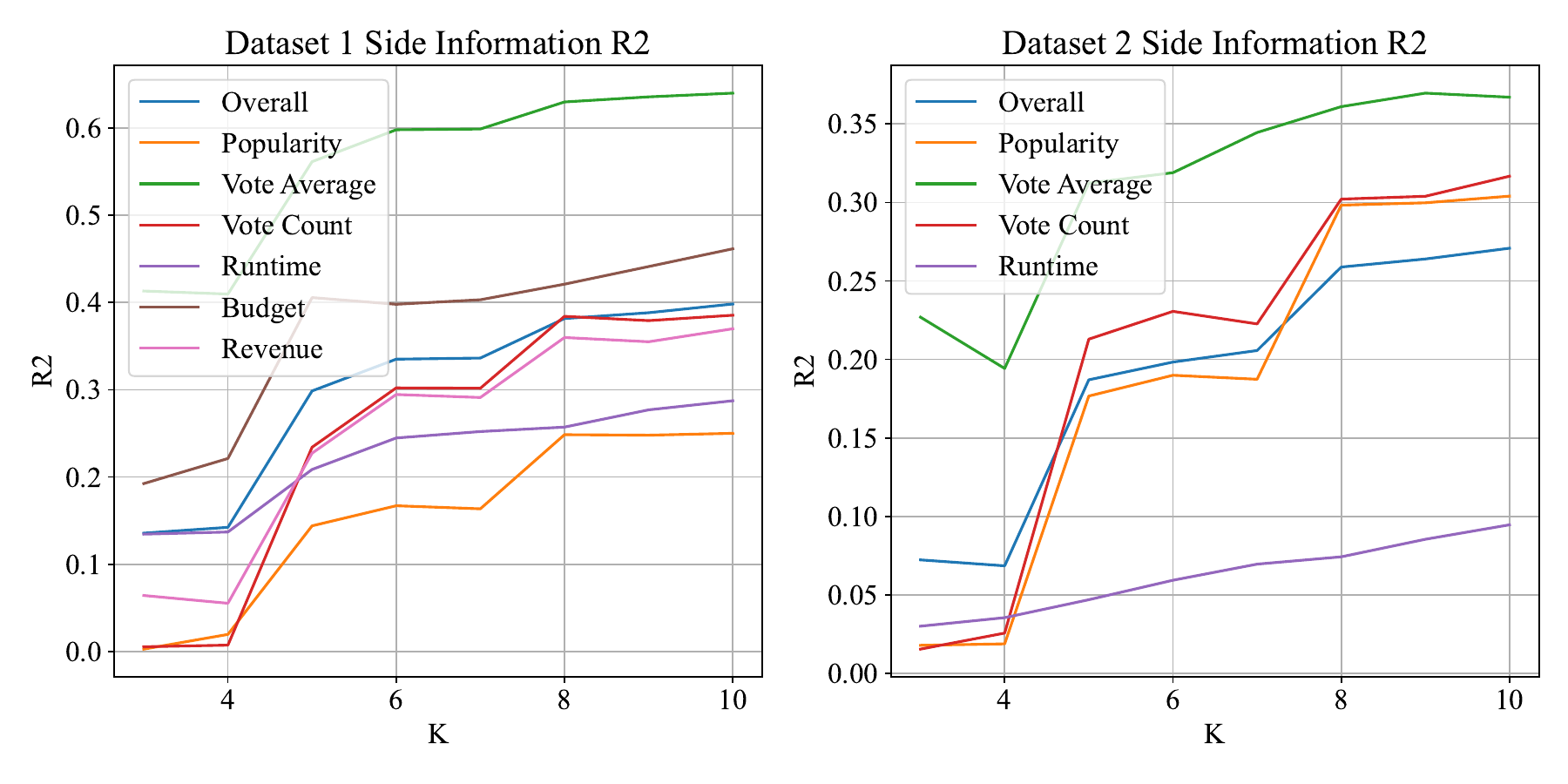}
  \caption{Algorithm \ref{alg:ADMM} side information $R^2$ on Netflix Prize Dataset 1 (left) and Dataset 2 (right) versus $k$. Averaged over $5$ trials.}
  \label{fig_lrml4:netflix_r2}
\end{figure*}

\begin{enumerate}
    \item {\color{black}Fast-Impute in general produced solutions that achieved slightly lower in sample error but significantly higher out of sample error than the solutions produced by Algorithm \ref{alg:ADMM}. Out of sample error is a much more important metric than in sample error as out of sample error captures the ability of a candidate solution to generalize to unseen data. Fast-Impute-Side produced solutions that had both higher in sample error and out of sample error than the solutions produced by Algorithm \ref{alg:ADMM}. Across the experiments in which Fast-Impute terminated within the specified time limit, Algorithm \ref{alg:ADMM} produced solutions that on average achieved $67\%$ lower out of sample error than Fast-Impute. The high out of sample error (relative to in sample error) of Fast-Impute and Fast-Impute-Side suggests that these methods are likely over-fitting the training data. As is expected, in sample error decreased as the rank of the reconstruction increased. In the case of Algorithm \ref{alg:ADMM}, out of sample error increased as the reconstruction rank increased, suggesting that Algorithm \ref{alg:ADMM} was over-fitting the data as rank increased. }

    \item {\color{black}Algorithm \ref{alg:ADMM} exhibited significantly superior scalability than both Fast-Impute and Fast-Impute-Side. Across the experiments in which Fast-Impute terminated within the specified time limit, Algorithm \ref{alg:ADMM} required on average $97\%$ less time to execute than Fast-Impute (the faster of the methods between Fast-Impute and Fast-Impute-Side). The execution time of Algorithm \ref{alg:ADMM} on the largest tested instance (Dataset $2, k=10$) was less than the execution time of Fast-Impute on the smallest tested instance (Dataset $1, k=3$).}
    
    \item The computational bottleneck of Algorithm \ref{alg:ADMM} is the solution time of subproblems \eqref{opt_lrml4:U_prob} and \eqref{opt_lrml4:V_prob}. Solving these two subproblems requires an order of magnitude more time than solving subproblems \eqref{opt_lrml4:P_prob} and \eqref{opt_lrml4:Z_prob}

    \item Figure \ref{fig_lrml4:netflix_r2} illustrates that the reconstructed matrix produced by Algorithm \ref{alg:ADMM} becomes a better predictor of the side information as the value of $k$ increases. This is to be expected as increasing $k$ increases model complexity. We see that popularity and runtime are the most difficult features to predict as a linear function of the reconstructed matrix while vote average and budget are the easiest features to predict.
\end{enumerate}

\subsection{Summary of Findings}

{\color{black} We now summarize our findings from our numerical experiments. In Sections \ref{sssec_lrml4:n_exp}-\ref{sssec_lrml4:k_exp}, we see that across all experiments using synthetic data and target rank $k \leq 10$, Algorithm \ref{alg:ADMM} produces solutions that achieve on average $2.3\%$ lower objective value and $41\%$ lower $\ell_2$ reconstruction error than the solutions returned by the best performing benchmark method (usually Fast-Impute-Side). In the regime where $k > 10$, we see in Section \ref{sssec_lrml4:k_exp} that Fast-Impute-Side outperforms Algorithm \ref{alg:ADMM}. We see that the execution time of Algorithm \ref{alg:ADMM} is competitive with and often notably faster than the benchmark methods on synthetic data. Our computational results are consistent with the complexity analysis performed in Section \ref{sec_lrml4:admm} for Problems \eqref{opt_lrml4:U_prob}, \eqref{opt_lrml4:V_prob}, \eqref{opt_lrml4:P_prob} and \eqref{opt_lrml4:Z_prob}. We observe that solution time for \eqref{opt_lrml4:P_prob} becomes the bottleneck as the target rank $k$ scales, otherwise the solution time for \eqref{opt_lrml4:U_prob} is the bottleneck. On real world data comprised of the Netflix Prize Dataset augmented with features from the TMDB Database, Algorithm \ref{alg:ADMM} produces solutions that achieve $67\%$ lower out of sample error than Fast-Impute, the best performing benchmark, in $97\%$ less execution time. }

\section{Conclusion} \label{sec_lrml4:conclusion}

In this paper, we introduced Problem \eqref{opt_lrml4:MC_primal} which seeks to reconstruct a partially observed matrix that is predictive of fully observed side information. We illustrate that \eqref{opt_lrml4:MC_primal} has a natural interpretation as a robust optimization problem and can be reformulated as a mixed-projection optimization problem. We derive a semidefinite cone relaxation \eqref{opt_lrml4:MC_SDP_relax} to \eqref{opt_lrml4:MC_primal} and we present Algorithm \ref{alg:ADMM}, a mixed-projection alternating direction method of multipliers algorithm that obtains scalable, high quality solutions to \eqref{opt_lrml4:MC_primal}. We rigorously benchmark the performance of Algorithm \ref{alg:ADMM} on synthetic and real world data against benchmark methods Fast-Impute, {\color{black} Fast-Impute-Side,} Soft-Impute, Iterative-SVD and ScaledGD. We find that across all synthetic data experiments with {\color{black}$k \leq 10$}, Algorithm \ref{alg:ADMM} outputs solutions that achieve on average {\color{black}$2.3\%$} lower objective value in \eqref{opt_lrml4:MC_primal} and {\color{black}$41\%$} lower $\ell_2$ reconstruction error than the solutions returned by the best performing benchmark method. For the $5$ synthetic data experiments with $k > 15$, Fast-Impute returns superior quality solutions than Algorithm \ref{alg:ADMM}, however the former takes on average $3$ times as long as Algorithm \ref{alg:ADMM} to execute. The runtime of Algorithm \ref{alg:ADMM} is competitive with and often superior to that of the benchmark methods. Algorithm \ref{alg:ADMM} is able to solve problems with $n = 10000$ rows and $m = 10000$ columns in less than a minute. On real world data from the Netflix Prize competition, Algorithm \ref{alg:ADMM} produces solutions that achieve $67\%$ lower out of sample error than benchmark methods in $97\%$ less execution time.Future work could expand the mixed-projection ADMM framework introduced in this work to incorporate positive semidefinite constraints and general linear constraints. Additionally, future work could empirically investigate the strength of the semidefinite relaxation \eqref{opt_lrml4:MC_SDP_relax} and could explore how to leverage this lower bound to certify globally optimal solutions.

\section*{Declarations}

\paragraph{Funding:} The authors did not receive support from any organization for the submitted work.
\paragraph{Conflict of interest/Competing interests:} The authors have no relevant financial or non-financial interests to disclose.
\paragraph{Ethics approval:} Not applicable.
\paragraph{Consent to participate:} Not applicable.
\paragraph{Consent for publication:} Not applicable.
\paragraph{Availability of data and materials:} The data employed in Section \ref{ssec_lrml4:real_experiments} was taken from the open source Netflix Prize competitions and the TMDB database.
\paragraph{Code availability:} To bridge the gap between theory and practice, we have made our code freely available on \url{GitHub} at \url{github.com/NicholasJohnson2020/LearningLowRankMatrices}
\paragraph{Authors' contributions:} Both authors contributed to the problem formulation and algorithmic ideation. Algorithm implementation, data collection, simulation and data analysis was performed by Nicholas Johnson. The first draft of the manuscript was written by Nicholas Johnson and both authors commented and edited subsequent versions of the manuscript. Both authors read and approved the final manuscript.

\clearpage

\begin{appendices}

\section{Supplemental Computational Results} \label{sec_lrml4:app_syn_exp}

\begin{table}[h]
  \centering
  \caption{\color{black}Comparison of the objective value of ScaledGD, Algorithm \ref{alg:ADMM}, Fast-Impute, Soft-Impute and SVD versus $n$ with $m=100, k=5$ and $d=150$. Averaged over $20$ trials for each parameter configuration.}\label{tbl_lrml4:N_objective}
  \begin{tabular}{c || cccccc}
\toprule
    \multicolumn{1}{c}{} & \multicolumn{6}{c}{Objective} \\
    \cmidrule(l){1-1} \cmidrule(l){2-7}
    N & ScaledGD & Algorithm \ref{alg:ADMM} & Fast-Impute & Fast-Impute-Side & Soft-Impute & SVD \\
\midrule
100 & 249263 & 660 & 4890 & \textbf{638} & 19790 & 28678 \\
200 & 306739 & 1283 & 9448 & \textbf{1217} & 23719 & 44055 \\
400 & 417643 & \textbf{2457} & 13910 & 2549 & 30685 & 61113 \\
800 & 421032 & 4844 & 28784 & \textbf{4805} & 40015 & 93119 \\
1000 & 522586 & 6046 & 42010 & \textbf{6002} & 46788 & 107851 \\
2000 & 563033 & 11975 & 59636 & \textbf{11944} & 76246 & 167459 \\
5000 & 1226490 & 30040 & 225143 & \textbf{29797} & 170248 & 364066 \\
10000 & 1973666 & 60083 & 644293 & \textbf{59411} & 317106 & 642760 \\
\bottomrule
\end{tabular}

\end{table}

\begin{table}[h]
  \centering
  \caption{\color{black}Comparison of the reconstruction error of ScaledGD, Algorithm \ref{alg:ADMM}, Fast-Impute, Soft-Impute and SVD versus $n$ with $m=100, k=5$ and $d=150$. Averaged over $20$ trials for each parameter configuration.}\label{tbl_lrml4:N_error}
  \begin{tabular}{c || cccccc}
\toprule
    \multicolumn{1}{c}{} & \multicolumn{6}{c}{$\ell_2$ Reconstruction Error} \\
    \cmidrule(l){1-1} \cmidrule(l){2-7}
    N & ScaledGD & Algorithm \ref{alg:ADMM} & Fast-Impute & Fast-Impute-Side & Soft-Impute & SVD \\
\midrule
100 & 100.225 & \textbf{0.015} & 0.072 & 0.021 & 0.212 & 0.301 \\
200 & 58.372 & \textbf{0.007} & 0.121 & 0.007 & 0.119 & 0.210 \\
400 & 33.925 & \textbf{0.004} & 0.062 & 0.005 & 0.074 & 0.160 \\
800 & 14.979 & \textbf{0.003} & 0.089 & 0.005 & 0.052 & 0.130 \\
1000 & 12.665 & \textbf{0.003} & 0.236 & 0.005 & 0.049 & 0.126 \\
2000 & 5.544 & \textbf{0.003} & 0.246 & 0.005 & 0.044 & 0.117 \\
5000 & 2.473 & \textbf{0.003} & 0.078 & 0.005 & 0.037 & 0.105 \\
10000 & 1.321 & \textbf{0.003} & 0.080 & 0.005 & 0.035 & 0.102 \\
\bottomrule
\end{tabular}

\end{table}

\begin{table}[h]
  \centering
  \caption{\color{black}Comparison of the side information $R^2$ of ScaledGD, Algorithm \ref{alg:ADMM}, Fast-Impute, Soft-Impute and SVD versus $n$ with $m=100, k=5$ and $d=150$. Averaged over $20$ trials for each parameter configuration.}\label{tbl_lrml4:N_r2}
  \begin{tabular}{c || cccccc}
\toprule
    \multicolumn{1}{c}{} & \multicolumn{6}{c}{Side Information $R^2$} \\
    \cmidrule(l){1-1} \cmidrule(l){2-7}
    N & ScaledGD & Algorithm \ref{alg:ADMM} & Fast-Impute & Fast-Impute-Side & Soft-Impute & SVD \\
\midrule
100 & 0.157 & 0.983 & 0.868 & - & \textbf{1.000} & 1.000 \\
200 & 0.193 & 0.984 & 0.878 & \textbf{0.985} & 0.837 & 0.740 \\
400 & 0.167 & \textbf{0.985} & 0.912 & 0.984 & 0.866 & 0.780 \\
800 & 0.441 & 0.985 & 0.911 & \textbf{0.985} & 0.905 & 0.826 \\
1000 & 0.328 & 0.985 & 0.896 & \textbf{0.985} & 0.906 & 0.829 \\
2000 & 0.557 & 0.985 & 0.924 & \textbf{0.985} & 0.920 & 0.862 \\
5000 & 0.525 & 0.985 & 0.889 & \textbf{0.986} & 0.928 & 0.879 \\
10000 & 0.582 & 0.985 & 0.840 & \textbf{0.985} & 0.932 & 0.888 \\
\bottomrule
\end{tabular}

\end{table}

\begin{table}[h]
  \centering
  \caption{\color{black}Comparison of the execution time of ScaledGD, Algorithm \ref{alg:ADMM}, Fast-Impute, Soft-Impute and SVD versus $n$ with $m=100, k=5$ and $d=150$. Averaged over $20$ trials for each parameter configuration.}\label{tbl_lrml4:N_time}
  \begin{tabular}{c || cccccc}
\toprule
    \multicolumn{1}{c}{} & \multicolumn{6}{c}{Execution Time (ms)} \\
    \cmidrule(l){1-1} \cmidrule(l){2-7}
    N & ScaledGD & Algorithm \ref{alg:ADMM} & Fast-Impute & Fast-Impute-Side & Soft-Impute & SVD \\
\midrule
100 & \textbf{11.05} & 76.74 & 100.68 & 207.84 & 120.84 & 115.21 \\
200 & \textbf{41.95} & 125.53 & 116.53 & 256.05 & 169.37 & 169.00 \\
400 & \textbf{56.63} & 169.16 & 159.47 & 327.89 & 251.79 & 254.37 \\
800 & \textbf{73.11} & 311.63 & 204.53 & 425.74 & 402.47 & 323.79 \\
1000 & \textbf{48.05} & 413.84 & 212.63 & 427.95 & 465.32 & 380.89 \\
2000 & \textbf{134.74} & 611.63 & 288.95 & 511.84 & 775.95 & 600.05 \\
5000 & 822.79 & 1413.00 & \textbf{503.00} & 706.37 & 1840.00 & 1318.32 \\
10000 & 19707.21 & 2275.00 & \textbf{881.58} & 1001.53 & 4016.89 & 2810.84 \\
\bottomrule
\end{tabular}

\end{table}

\begin{table}[h]
  \centering
  \caption{\color{black}Comparison of the objective value of ScaledGD, Algorithm \ref{alg:ADMM}, Fast-Impute, Soft-Impute and SVD versus $m$ with $n=1000, k=5$ and $d=150$. Averaged over $20$ trials for each parameter configuration.}\label{tbl_lrml4:M_objective}
  \begin{tabular}{c || cccccc}
\toprule
    \multicolumn{1}{c}{} & \multicolumn{6}{c}{Objective} \\
    \cmidrule(l){1-1} \cmidrule(l){2-7}
    M & ScaledGD & Algorithm \ref{alg:ADMM} & Fast-Impute & Fast-Impute-Side & Soft-Impute & SVD \\
\midrule
100 & 530097 & 6044 & 37924 & \textbf{5997} & 46917 & 103403 \\
200 & 2483914 & \textbf{6134} & 12019 & 6388 & 28371 & 114448 \\
400 & 14226535 & 6368 & 14060 & \textbf{6329} & 22009 & 90652 \\
800 & 99356630 & \textbf{6839} & 22086 & 8133 & 38745 & 87895 \\
1000 & 105452002 & \textbf{7060} & 23888 & 9470 & 45968 & 128500 \\
2000 & 591164397 & \textbf{14041} & 43504 & 17158 & 96007 & 815807 \\
5000 & 4002088274 & \textbf{42069} & 100878 & 60243 & 310417 & 11294105 \\
10000 & 9826250213 & 115589 & 196625 & \textbf{46335} & 943496 & 60913874 \\
\bottomrule
\end{tabular}

\end{table}

\begin{table}[h]
  \centering
  \caption{\color{black}Comparison of the reconstruction error of ScaledGD, Algorithm \ref{alg:ADMM}, Fast-Impute, Soft-Impute and SVD versus $m$ with $n=1000, k=5$ and $d=150$. Averaged over $20$ trials for each parameter configuration.}\label{tbl_lrml4:M_error}
  \begin{tabular}{c || cccccc}
\toprule
    \multicolumn{1}{c}{} & \multicolumn{6}{c}{$\ell_2$ Reconstruction Error} \\
    \cmidrule(l){1-1} \cmidrule(l){2-7}
    M & ScaledGD & Algorithm \ref{alg:ADMM} & Fast-Impute & Fast-Impute-Side & Soft-Impute & SVD \\
\midrule
100 & 13.68740 & \textbf{0.00323} & 0.18290 & 0.00500 & 0.05010 & 0.12560 \\
200 & 40.58900 & \textbf{0.00154} & 0.00500 & 0.00300 & 0.01150 & 0.06640 \\
400 & 127.71450 & \textbf{0.00075} & 0.00340 & 0.00190 & 0.00340 & 0.02240 \\
800 & 508.24550 & \textbf{0.00036} & 0.00340 & 0.01000 & 0.00310 & 0.00460 \\
1000 & 443.26630 & \textbf{0.00029} & 0.00300 & 0.01340 & 0.00300 & 0.00350 \\
2000 & 1292.61610 & \textbf{0.00013} & 0.00310 & 0.01410 & 0.00290 & 0.00300 \\
5000 & 3658.18840 & \textbf{0.00004} & 0.00320 & 0.01570 & 0.00270 & 0.00540 \\
10000 & 4559.27300 & \textbf{0.00002} & 0.00330 & 0.01530 & 0.00270 & 0.00650 \\
\bottomrule
\end{tabular}

\end{table}

\begin{table}[h]
  \centering
  \caption{\color{black}Comparison of the execution time of ScaledGD, Algorithm \ref{alg:ADMM}, Fast-Impute, Soft-Impute and SVD versus $m$ with $n=1000, k=5$ and $d=150$. Averaged over $20$ trials for each parameter configuration.}\label{tbl_lrml4:M_time}
  \begin{tabular}{c || cccccc}
\toprule
    \multicolumn{1}{c}{} & \multicolumn{6}{c}{Execution Time (ms)} \\
    \cmidrule(l){1-1} \cmidrule(l){2-7}
    M & ScaledGD & Algorithm \ref{alg:ADMM} & Fast-Impute & Fast-Impute-Side & Soft-Impute & SVD \\
\midrule
100 & \textbf{45.26} & 371.16 & 219.84 & 450.11 & 451.95 & 375.37 \\
200 & \textbf{56.79} & 398.68 & 217.47 & 563.11 & 971.16 & 764.37 \\
400 & \textbf{79.89} & 531.11 & 288.74 & 698.42 & 1503.68 & 1594.00 \\
800 & \textbf{131.16} & 662.68 & 385.95 & 1402.89 & 2386.58 & 2721.89 \\
1000 & \textbf{159.79} & 817.53 & 541.84 & 1651.00 & 2781.74 & 3140.95 \\
2000 & 4698.95 & \textbf{1459.00} & 5125.47 & 16279.21 & 5282.89 & 6297.79 \\
5000 & 29291.74 & \textbf{13195.95} & 40036.79 & 143865.75 & 32777.05 & 35080.95 \\
10000 & 108192.58 & \textbf{40368.16} & 158083.84 & 399863.00 & 82190.32 & 89836.58 \\
\bottomrule
\end{tabular}

\end{table}

\begin{table}[h]
  \centering
  \caption{\color{black}Comparison of the objective value of ScaledGD, Algorithm \ref{alg:ADMM}, Fast-Impute, Soft-Impute and SVD versus $d$ with $n=1000, m=100$ and  $k=5$. Averaged over $20$ trials for each parameter configuration.}\label{tbl_lrml4:D_objective}
  \begin{tabular}{c || cccccc}
\toprule
    \multicolumn{1}{c}{} & \multicolumn{6}{c}{Objective} \\
    \cmidrule(l){1-1} \cmidrule(l){2-7}
    D & ScaledGD & Algorithm \ref{alg:ADMM} & Fast-Impute & Fast-Impute-Side & Soft-Impute & SVD \\
\midrule
10 & 11692 & \textbf{478} & 2648 & 553 & 3368 & 7710 \\
50 & 96771 & 2070 & 13402 & \textbf{2058} & 16273 & 37204 \\
100 & 229740 & 4051 & 22644 & \textbf{4008} & 30116 & 68634 \\
150 & 532018 & 6010 & 29433 & \textbf{5996} & 46596 & 106504 \\
200 & 734648 & \textbf{7994} & 49004 & 8029 & 63243 & 141252 \\
250 & 1195066 & \textbf{9984} & 59470 & 10160 & 79625 & 183720 \\
500 & 4165783 & \textbf{20081} & 124737 & 20302 & 157096 & 361741 \\
1000 & 13578263 & \textbf{40188} & 276053 & 40781 & 295635 & 668550 \\
\bottomrule
\end{tabular}

\end{table}

\begin{table}[h]
  \centering
  \caption{\color{black}Comparison of the reconstruction error of ScaledGD, Algorithm \ref{alg:ADMM}, Fast-Impute, Soft-Impute and SVD versus $d$ with $n=1000, m=100$ and  $k=5$. Averaged over $20$ trials for each parameter configuration.}\label{tbl_lrml4:D_error}
  \begin{tabular}{c || cccccc}
\toprule
    \multicolumn{1}{c}{} & \multicolumn{6}{c}{$\ell_2$ Reconstruction Error} \\
    \cmidrule(l){1-1} \cmidrule(l){2-7}
    D & ScaledGD & Algorithm \ref{alg:ADMM} & Fast-Impute & Fast-Impute-Side & Soft-Impute & SVD \\
\midrule
10 & 0.60780 & \textbf{0.00709} & 0.09620 & 0.01520 & 0.04870 & 0.12530 \\
50 & 1.41600 & \textbf{0.00487} & 0.10460 & 0.00980 & 0.05140 & 0.12730 \\
100 & 5.03590 & \textbf{0.00386} & 0.05920 & 0.00660 & 0.05050 & 0.12790 \\
150 & 14.00330 & \textbf{0.00324} & 0.05830 & 0.00500 & 0.04940 & 0.12650 \\
200 & 22.26870 & \textbf{0.00276} & 0.22120 & 0.00380 & 0.04840 & 0.12510 \\
250 & 39.41630 & \textbf{0.00245} & 0.10240 & 0.00300 & 0.04890 & 0.12580 \\
500 & 177.68800 & 0.00170 & 0.32450 & \textbf{0.00142} & 0.05000 & 0.12640 \\
1000 & 679.11770 & 0.00100 & 0.34150 & \textbf{0.00084} & 0.04980 & 0.12660 \\
\bottomrule
\end{tabular}

\end{table}

\begin{table}[h]
  \centering
  \caption{\color{black}Comparison of the side information $R^2$ of ScaledGD, Algorithm \ref{alg:ADMM}, Fast-Impute, Soft-Impute and SVD versus $d$ with $n=1000, m=100$ and  $k=5$. Averaged over $20$ trials for each parameter configuration.}\label{tbl_lrml4:D_r2}
  \begin{tabular}{c || cccccc}
\toprule
    \multicolumn{1}{c}{} & \multicolumn{6}{c}{Side Information $R^2$} \\
    \cmidrule(l){1-1} \cmidrule(l){2-7}
    D & ScaledGD & Algorithm \ref{alg:ADMM} & Fast-Impute & Fast-Impute-Side & Soft-Impute & SVD \\
\midrule
10 & 0.91180 & 0.98800 & 0.90730 & \textbf{0.99199} & 0.90330 & 0.82120 \\
50 & 0.52800 & 0.98560 & 0.89950 & \textbf{0.98640} & 0.90340 & 0.82530 \\
100 & 0.53300 & 0.98510 & 0.91420 & \textbf{0.98549} & 0.91130 & 0.84020 \\
150 & 0.30270 & 0.98490 & 0.92480 & \textbf{0.98506} & 0.90500 & 0.83050 \\
200 & 0.41070 & \textbf{0.98525} & 0.90860 & 0.98520 & 0.90690 & 0.83570 \\
250 & 0.29080 & \textbf{0.98551} & 0.91300 & 0.98530 & 0.90830 & 0.83440 \\
500 & 0.15090 & \textbf{0.98491} & 0.90560 & 0.98470 & 0.90530 & 0.82600 \\
1000 & 0.29250 & \textbf{0.98441} & 0.89150 & 0.98420 & 0.90940 & 0.83810 \\
\bottomrule
\end{tabular}

\end{table}

\begin{table}[h]
  \centering
  \caption{\color{black}Comparison of the execution time of ScaledGD, Algorithm \ref{alg:ADMM}, Fast-Impute, Soft-Impute and SVD versus $d$ with $n=1000, m=100$ and  $k=5$. Averaged over $20$ trials for each parameter configuration.}\label{tbl_lrml4:D_time}
  \begin{tabular}{c || cccccc}
\toprule
    \multicolumn{1}{c}{} & \multicolumn{6}{c}{Execution Time (ms)} \\
    \cmidrule(l){1-1} \cmidrule(l){2-7}
    D & ScaledGD & Algorithm \ref{alg:ADMM} & Fast-Impute & Fast-Impute-Side & Soft-Impute & SVD \\
\midrule
10 & \textbf{81.73684} & 324.73680 & 250.10530 & 106.31580 & 462.84210 & 379.00000 \\
50 & \textbf{91.47368} & 345.84210 & 210.21050 & 168.89470 & 461.89470 & 384.42110 \\
100 & \textbf{91.89474} & 386.47370 & 217.57890 & 289.89470 & 469.42110 & 393.78950 \\
150 & \textbf{98.73684} & 413.57890 & 229.36840 & 521.73680 & 480.36840 & 396.36840 \\
200 & \textbf{99.00000} & 432.36840 & 275.94740 & 738.00000 & 466.73680 & 373.57890 \\
250 & \textbf{151.73684} & 492.31580 & 275.52630 & 1002.94740 & 522.05260 & 415.78950 \\
500 & \textbf{181.15789} & 545.78950 & 266.78950 & 1673.31580 & 463.47370 & 373.00000 \\
1000 & \textbf{148.26316} & 668.57890 & 317.10530 & 2719.10530 & 523.63160 & 422.42110 \\
\bottomrule
\end{tabular}

\end{table}

\begin{table}[h]
  \centering
  \caption{\color{black}Comparison of the objective value of ScaledGD, Algorithm \ref{alg:ADMM}, Fast-Impute, Soft-Impute and SVD versus $k$ with $n=1000, m=100$ and  $d=150$. Averaged over $20$ trials for each parameter configuration.}\label{tbl_lrml4:K_objective}
  \begin{tabular}{c || cccccc}
\toprule
    \multicolumn{1}{c}{} & \multicolumn{6}{c}{Objective} \\
    \cmidrule(l){1-1} \cmidrule(l){2-7}
    D & ScaledGD & Algorithm \ref{alg:ADMM} & Fast-Impute & Fast-Impute-Side & Soft-Impute & SVD \\
\midrule
5 & 514330 & 6040 & 40047 & \textbf{6002} & 46511 & 106391 \\
10 & 1892279 & 7502 & 240910 & \textbf{6315} & 316063 & 805397 \\
15 & 5213393 & 14022 & 113808 & \textbf{6870} & 1121725 & 2495972 \\
20 & 10196280 & 336294 & 103289 & \textbf{7952} & 2629190 & 4910387 \\
25 & 16816443 & 885813 & 115398 & \textbf{8601} & 4499282 & 7541301 \\
30 & 27397868 & 65551564 & 128971 & \textbf{8824} & 6863908 & 10634437 \\
35 & 39536651 & 160010082 & 144816 & \textbf{8832} & 9407573 & 14192828 \\
40 & - & 320532712 & 114058512 & 19625866 & \textbf{12538955} & 18290215 \\
\bottomrule
\end{tabular}

\end{table}

\begin{table}[h]
  \centering
  \caption{\color{black}Comparison of the reconstruction error of ScaledGD, Algorithm \ref{alg:ADMM}, Fast-Impute, Soft-Impute and SVD versus $k$ with $n=1000, m=100$ and  $d=150$. Averaged over $20$ trials for each parameter configuration.}\label{tbl_lrml4:K_error}
  \begin{tabular}{c || cccccc}
\toprule
    \multicolumn{1}{c}{} & \multicolumn{6}{c}{$\ell_2$ Reconstruction Error} \\
    \cmidrule(l){1-1} \cmidrule(l){2-7}
    D & ScaledGD & Algorithm \ref{alg:ADMM} & Fast-Impute & Fast-Impute-Side & Soft-Impute & SVD \\
\midrule
5 & 12.95500 & \textbf{0.00314} & 0.12670 & 0.00480 & 0.05160 & 0.12900 \\
10 & 5.41720 & \textbf{0.00284} & 0.20510 & 0.00350 & 0.08990 & 0.22520 \\
15 & 3.45210 & 0.00940 & 0.01870 & \textbf{0.00561} & 0.14080 & 0.30200 \\
20 & 2.34590 & 0.04860 & 0.00920 & \textbf{0.00505} & 0.18300 & 0.34220 \\
25 & 1.73610 & 0.11940 & 0.00680 & \textbf{0.00409} & 0.21440 & 0.36130 \\
30 & 1.40170 & 0.19840 & 0.00570 & \textbf{0.00350} & 0.23740 & 0.37240 \\
35 & 1.16780 & 0.17700 & 0.00480 & \textbf{0.00301} & 0.25270 & 0.37740 \\
40 & - & 0.21870 & 0.00440 & \textbf{0.00267} & 0.26520 & 0.38020 \\
\bottomrule
\end{tabular}

\end{table}

\begin{table}[h]
  \centering
  \caption{\color{black}Comparison of the execution time of ScaledGD, Algorithm \ref{alg:ADMM}, Fast-Impute, Soft-Impute and SVD versus $k$ with $n=1000, m=100$ and  $d=150$. Averaged over $20$ trials for each parameter configuration.}\label{tbl_lrml4:K_time}
  \begin{tabular}{c || cccccc}
\toprule
    \multicolumn{1}{c}{} & \multicolumn{6}{c}{Execution Time (ms)} \\
    \cmidrule(l){1-1} \cmidrule(l){2-7}
    D & ScaledGD & Algorithm \ref{alg:ADMM} & Fast-Impute & Fast-Impute-Side & Soft-Impute & SVD \\
\midrule
5 & \textbf{65} & 358 & 217 & 455 & 450 & 389 \\
10 & \textbf{85} & 586 & 938 & 737 & 485 & 354 \\
15 & \textbf{112} & 825 & 2217 & 775 & 512 & 346 \\
20 & \textbf{110} & 1390 & 3499 & 857 & 576 & 333 \\
25 & \textbf{118} & 2071 & 4966 & 1310 & 599 & 298 \\
30 & \textbf{108} & 2425 & 6381 & 1272 & 625 & 275 \\
35 & \textbf{120} & 2653 & 8635 & 1508 & 641 & 280 \\
40 & 4532 & 3092 & 11811 & 1797 & 689 & \textbf{270} \\
\bottomrule
\end{tabular}

\end{table}

\begin{table}[h]
  \centering
  \caption{\color{black}Comparison of the in sample $\ell_2$ reconstruction error, out of sample $\ell_2$ reconstruction error and execution time of Algorithm \ref{alg:ADMM} and Fast-Impute versus $k$ on Netflix Prize Dataset $1$. Averaged over $5$ trials.}\label{tbl_lrml4:netflix_6Y}
  \begin{tabular}{c || ccccccccc}
\toprule
    \multicolumn{1}{c}{} & \multicolumn{3}{c}{In Sample Error} & \multicolumn{3}{c}{Out of Sample Error} & \multicolumn{3}{c}{Execution Time (hr)}\\
    \cmidrule(l){1-1} \cmidrule(l){2-4} \cmidrule(l){5-7} \cmidrule(l){8-10}
    K & Alg \ref{alg:ADMM} & FI & FIS & Alg \ref{alg:ADMM} & FI & FIS & Alg \ref{alg:ADMM} & FI & FIS \\
\midrule
 3 &              \textbf{0.0507} &                     0.0516 &        0.5782&          \textbf{0.0573} &                         0.2264 &        0.8929 &         \textbf{0.2512} &                         6.4907 & 2.7021\\
 4 &              0.0490 &                     \textbf{0.0486} &        0.5588&          \textbf{0.0604} &                         0.1942 &        2.2728 &         \textbf{0.2760} &                         8.8641 & 2.8570\\
 5 &              0.0476 &                     \textbf{0.0460} &        0.5095&          \textbf{0.0651} &                         0.1835 &        1.1568 &         \textbf{0.3136} &                        11.8207 & 3.1463\\
 6 &              0.0463 &                     \textbf{0.0438} &        0.4842&          \textbf{0.0676} &                         0.1867 &        1.3890 &         \textbf{0.3654} &                        17.4471 & 3.2974\\
 7 &              \textbf{0.0451} &                    - &              0.4842&    \textbf{0.0718} &                        - &         1.5014 &        \textbf{0.4637} &                        - & 3.4976\\
 8 &              \textbf{0.0442} &                    - &              0.4842&    \textbf{0.0759} &                        - &         1.4145 &        \textbf{0.4941} &                        - & 3.8013\\
 9 &              \textbf{0.0434} &                    - &              0.4842&    \textbf{0.0811} &                        - &         1.1559 &        \textbf{1.0262} &                        - & 3.7692\\
10 &              \textbf{0.0427} &                    - &              0.4842&    \textbf{0.0839} &                        - &         1.3480 &        \textbf{0.8503} &                        - & 3.9254\\
\bottomrule
\end{tabular}

\end{table}

\begin{table}[h]
  \centering
  \caption{\color{black}In sample $\ell_2$ reconstruction error, out of sample $\ell_2$ reconstruction error and execution time of Algorithm \ref{alg:ADMM} versus $k$ on Netflix Prize Dataset $2$. Averaged over $5$ trials.}\label{tbl_lrml4:netflix_4Y}
  \begin{tabular}{c || cccccc}
\toprule
    \multicolumn{1}{c}{} & \multicolumn{2}{c}{In Sample Error} & \multicolumn{2}{c}{Out of Sample Error} & \multicolumn{2}{c}{Execution Time (hr)}\\
    \cmidrule(l){1-1} \cmidrule(l){2-3} \cmidrule(l){4-5} \cmidrule(l){6-7}
    K & Algorithm \ref{alg:ADMM} & FIS & Algorithm \ref{alg:ADMM} & FIS & Algorithm \ref{alg:ADMM} & FIS \\
\midrule
 3 &              \textbf{0.0518} &                     0.5255 &                  \textbf{0.0591} &                         0.6825 &                  \textbf{0.5773} &                         3.8481 \\
 4 &              \textbf{0.0502} &                     0.4753 &                  \textbf{0.0607} &                         1.0302 &                  \textbf{0.6581} &                         4.2354 \\
 5 &              \textbf{0.0488} &                     0.4753 &                  \textbf{0.0652} &                         1.0172 &                  \textbf{1.3482} &                         4.2825 \\
 6 &              \textbf{0.0475} &                     0.4753 &                  \textbf{0.0680} &                         1.0417 &                  \textbf{1.9273} &                         4.6575 \\
 7 &              \textbf{0.0463} &                     0.4753 &                  \textbf{0.0715} &                         1.0849 &                  \textbf{2.7246} &                         4.7145 \\
 8 &              \textbf{0.0454} &                     0.4753 &                  \textbf{0.0718} &                         1.0857 &                  \textbf{2.8198} &                         5.0538 \\
 9 &              \textbf{0.0446} &                     0.4753 &                  \textbf{0.0782} &                         1.0304 &                  \textbf{2.8945} &                         5.1645 \\
10 &              \textbf{ 0.0439} &                    0.4753 &                  \textbf{0.0812} &                         1.0246 &                  \textbf{2.3311} &                         5.5966 \\
\bottomrule
\end{tabular}

\end{table}

\begin{table}[h]
  \centering
  \caption{Algorithm \ref{alg:ADMM} side information $R^2$ on Netflix Prize Dataset 1. Averaged over $5$ trials.}\label{tbl_lrml4:netflix_6Y_r2}
  \begin{tabular}{c || ccccccc}
\toprule
    K & Overall & Popularity & Vote Average & Vote Count & Runtime & Budget & Revenue \\
\midrule
 3 & 0.136 & 0.003 & 0.413 & 0.005 & 0.134 & 0.192 & 0.064 \\
 4 & 0.142 &  0.02 &  0.41 & 0.007 & 0.137 & 0.221 & 0.055 \\
 5 & 0.299 & 0.144 & 0.561 & 0.234 & 0.209 & 0.406 & 0.227 \\
 6 & 0.335 & 0.167 & 0.598 & 0.302 & 0.245 & 0.398 & 0.295 \\
 7 & 0.336 & 0.163 & 0.599 & 0.302 & 0.252 & 0.403 & 0.291 \\
 8 & 0.382 & 0.248 &  0.63 & 0.384 & 0.257 & 0.421 &  0.36 \\
 9 & 0.388 & 0.248 & 0.636 & 0.379 & 0.277 & 0.441 & 0.355 \\
10 & 0.398 &  0.25 &  0.64 & 0.385 & 0.287 & 0.462 &  0.37 \\
\bottomrule
\end{tabular}

\end{table}

\begin{table}[h]
  \centering
  \caption{Algorithm \ref{alg:ADMM} side information $R^2$ on Netflix Prize Dataset 2. Averaged over $5$ trials.}\label{tbl_lrml4:netflix_4Y_r2}
  \begin{tabular}{c || ccccc}
\toprule
    K & Overall & Popularity & Vote Average & Vote Count & Runtime \\
\midrule
 3 & 0.072 & 0.018 & 0.227 & 0.016 &  0.03 \\
 4 & 0.069 & 0.019 & 0.194 & 0.026 & 0.036 \\
 5 & 0.187 & 0.177 & 0.312 & 0.213 & 0.047 \\
 6 & 0.198 &  0.19 & 0.319 & 0.231 & 0.059 \\
 7 & 0.206 & 0.187 & 0.344 & 0.223 &  0.07 \\
 8 & 0.259 & 0.298 & 0.361 & 0.302 & 0.074 \\
 9 & 0.264 &   0.3 &  0.37 & 0.304 & 0.086 \\
10 & 0.271 & 0.304 & 0.367 & 0.317 & 0.095 \\
\bottomrule
\end{tabular}

\end{table}
    
\end{appendices}

\clearpage

\section*{Declarations}

\paragraph{Funding:} The authors did not receive support from any organization for the submitted work.
\paragraph{Conflict of interest/Competing interests:} The authors have no relevant financial or non-financial interests to disclose.
\paragraph{Ethics approval:} Not applicable.
\paragraph{Consent to participate:} Not applicable.
\paragraph{Consent for publication:} Not applicable.
\paragraph{Availability of data and materials:} The data used in Section \ref{ssec_lrml4:real_experiments} was taken from the Netflix Prize Dataset and the TMDB Database, both of which are freely available online.
\paragraph{Code availability:} To bridge the gap between theory and practice, we have made our code freely available on \url{GitHub} at \url{https://github.com/NicholasJohnson2020/LearningLowRankMatrices}
\paragraph{Authors' contributions:} Both authors contributed to formulating the optimization problem. Algorithmic design and implementation along with data collection, simulation and data analysis was performed by Nicholas Johnson. The first draft of the manuscript was written by Nicholas Johnson and both authors commented and edited subsequent versions of the manuscript. Both authors read and approved the final manuscript.




\clearpage


\bibliography{sn-bibliography}

\end{document}